%% file: main.tex
\documentclass[lettersize,journal]{IEEEtran}

\newtheorem{theorem}{Theorem}

\newtheorem{lemma}{Lemma}   

\newtheorem{assumption}{Assumption}
\newtheorem{remark}{Remark}
\newtheorem{proof}{Proof}
\usepackage[linesnumbered,ruled,boxed]{algorithm2e}
\usepackage{amsmath,amsfonts}
\usepackage{bbm}
\usepackage[noend]{algorithmic}
\usepackage{array}
\usepackage[caption=false,font=normalsize,labelfont=sf,textfont=sf]{subfig}
\usepackage{textcomp}
\usepackage{stfloats}
\usepackage{adjustbox}
 \usepackage{multirow}
 \usepackage{booktabs}       
\usepackage{url}
\usepackage{verbatim}
\usepackage{graphicx}
\usepackage{caption} 
\usepackage{color}
\usepackage{tabularx}
\newcommand{\loc}{\mathrm{loc}}
\newcommand{\ser}{\mathrm{ser}}

\newcommand{\Sg}{\boldsymbol{\Sigma}}
\newcommand{\Bf}{\mathbf{b}}
\newcommand{\Tf}{\boldsymbol{\theta}}
\newcommand{\xb}{\mathbf{x}}

\newcommand{\alglinelabel}{%
  \addtocounter{ALC@line}{-1}
  \refstepcounter{ALC@line}
  \label
}

\def\BibTeX{{\rm B\kern-.05em{\sc i\kern-.025em b}\kern-.08em
    T\kern-.1667em\lower.7ex\hbox{E}\kern-.125emX}}
\usepackage{balance}
\begin{document}


\title{CANS: Accelerating Multiuser Collaborative Edge Inference via Cooperative Autodidactic NeuroSurgeon}
    \author{Zheshun Wu, Ziyang Zhang, Changyao Lin, Zenglin Xu, \IEEEmembership{Senior Member, IEEE}, and Jie Liu, \IEEEmembership{Fellow, IEEE}
\thanks{Copyright (c) 2015 IEEE. Personal use of this material is permitted. However, permission to use this material for any other purposes must be obtained from the IEEE by sending a request to pubs-permissions@ieee.org.}
\thanks{This work was supported in part by National Natural Science Foundation
of China (No. 62350710797). (\emph{Corresponding author: Jie Liu; Zenglin Xu.})}
\thanks{Zheshun Wu and Jie Liu are with the School of Computer Science and Technology, Harbin Institute of Technology Shenzhen, Shenzhen 518055, China (e-mail:
wuzhsh23@gmail.com; jieliu@hit.edu.cn).}
\thanks{Ziyang Zhang is with NESLab, Politecnico di Milano, 39-20133 Milano, Italy (e-mail:
ziyang.zhang@polimi.it).}
\thanks{Changyao Lin is with the School of Computer Science and Technology, Harbin Institute of Technology, Harbin, Heilongjiang 150001, China  (e-mail: lincy@stu.hit.edu.cn).}
\thanks{Zenglin Xu is with the Fudan University, and also with the Shanghai Academy of Artificial Intelligence for Science  (e-mail: zenglin@gmail.com).}
}


\maketitle

\begin{abstract}
Recently, mobile edge computing (MEC)–enabled collaborative deep neural network (DNN) inference has emerged as a promising approach for delivering intelligent services to resource-constrained mobile devices.  A representative scenario is multi-user collaborative edge inference, where distinct devices independently partition their DNN models and offload backend computation to a common edge server over wireless networks. However, determining the optimal DNN partition for each device is challenging due to unknown and time-varying system conditions, including fluctuating wireless links and diverse device capabilities. To address this problem, we propose Cooperative Autodidactic NeuroSurgeon (CANS), a collaborative edge inference framework that enables devices to adaptively learn optimal DNN partitions by sharing informative feedback during online inference. To handle the challenge of device heterogeneity and better leverage offline inference experience, we integrate a novel FedLinUCB-DW algorithm that groups devices of the same type and warm-starts online exploration using local offline early-exit inference experience. Furthermore, we provide theoretical guarantees for FedLinUCB-DW by deriving the regret upper bound. We also validate our method on both a simulated environment and a hardware prototype system. Empirical evaluations demonstrate that CANS achieves lower inference latency compared to state-of-the-art baselines. Especially, in prototype experiments on two edge devices, the proposed CANS reduced average inference latency by up to 50\% compared to the non-cooperative baseline.

\end{abstract}


\section{Introduction}
\input{content/intro}

\section{Problem Formulation}
\input{content/formulation}

\section{Algorithm}
\input{content/algorithm}

\section{Experiments}
\input{content/experiment}

\section{Conclusion}
This paper addresses the challenge of identifying optimal DNN partition points in multi-user collaborative edge inference under unknown system parameters. We formulate this task as a distributed online learning problem and introduce the CANS system, which accelerates parameter estimation by facilitating feedback sharing among UEs. Building on this system, we devise the FedLinUCB-DW algorithm to enhance learning efficiency by grouping homogeneous devices and leveraging offline inference experience for a warm start. Both theoretical analysis and extensive evaluations, including large-scale simulations and a hardware prototype, demonstrate that the proposed approach significantly outperforms existing baselines in reducing end-to-end inference latency.

\bibliographystyle{ieeetr}
\bibliography{IEEEtran}

\input{content/supplement}

\end{document}

%% file: content/intro.tex
\IEEEPARstart{D}{eep} neural networks (DNNs) have been widely adopted in various mobile and Internet of Things (IoT) applications such as face recognition~\cite{koubaa2021cloud}, video analytics~\cite{xu2023edge} and speech assistants~\cite{tulshan2018survey}. However, resource-constrained mobile devices, including handheld devices, wearables and wireless sensors, often struggle to execute computationally intensive DNN inference tasks  due to their limited energy budgets and processing capabilities~\cite{mao2016dynamic,shuvo2022efficient}. In recent years, Mobile Edge Computing (MEC) has  emerged as a promising paradigm to meet the high computational demands of DNN inference while alleviating the workload of mobile devices~\cite{10485479,8798727,8924682}. In MEC systems, edge servers equipped with powerful computing resources are deployed in close proximity to mobile devices, allowing each device to offload its inference tasks to the edge server~\cite{8771176,8998328}.

 A particularly effective offloading approach within MEC is collaborative edge inference, whose central idea is to partition the DNN into a front-end segment executed on the mobile device and a back-end segment processed on the edge server~\cite{kang2017neurosurgeon,ren2023survey}. Compared with purely on-device inference or complete task offloading, collaborative edge inference offers a flexible balance between communication and computation, enabling efficient resource utilization and reduced latency. The core challenge in collaborative edge inference lies in determining the optimal partition point that divides the DNN between the device and the server so as to minimize the end‑to‑end inference latency~\cite{zhang2021autodidactic}. Most existing approaches rely on performing offline profiling~\cite{kang2017neurosurgeon,8737614,9882293}, where the optimal partition is identified by solving an offline optimization problem based on prior knowledge of system parameters such as the wireless transmission rate and the processing capabilities of both the mobile device and the edge server. However, these system parameters are often difficult to obtain for mobile devices in real-world scenarios, and can be easily outdated due to highly dynamic and uncertain environments~\cite{zhang2021autodidactic,9612607}.

To address these challenges, some studies have utilized online learning methods to dynamically estimate unknown system parameters and thereby better identify  optimal partition points~\cite{zhang2021autodidactic,huang2023adversarial,11141772}. For instance, the authors in~\cite{zhang2021autodidactic} model the search for the optimal DNN partition as a contextual linear bandit problem and solve it using a Linear Upper Confidence Bound (LinUCB)-style algorithm. Based on it,~\cite{huang2023adversarial} and~\cite{11141772} further consider the offloading-server selection and the offloading-path selection problem respectively.   However,  these works focus solely on the single-user setting, where the edge server provides inference offloading services to only one user equipment (UE) in the wireless network.  In practical applications, an edge server typically provides inference offloading services to multiple users over wireless networks, supported by its powerful computing resources.

Several recent works have begun to target the multi-user scenario by jointly optimizing DNN partitioning and server-resource allocation~\cite{9542866,9144210}.  However, they still consider conducting offline profiling under the restrictive assumption that system parameters are known in advance, which is impractical in realistic deployments. The work most closely related to ours is~\cite{wang2024multi}, which iteratively solves optimization problems to identify optimal partitions for multiple UEs using deep reinforcement learning. However, this work still assumes a stationary environment and that all system parameters are known to the algorithm in advance. Hence, although supporting collaborative edge inference for multiple UEs is of greater practical importance, research on devising distributed online learning methods for searching optimal DNN partitioning under this setting still remains open.


   \begin{figure*}[ht]
\captionsetup[subfloat]{font=footnotesize}	
\centering
   	\subfloat[In CANS, each UE conducts its own collaborative inference with the edge server over the wireless network. ]{\includegraphics[width = 0.43\textwidth]{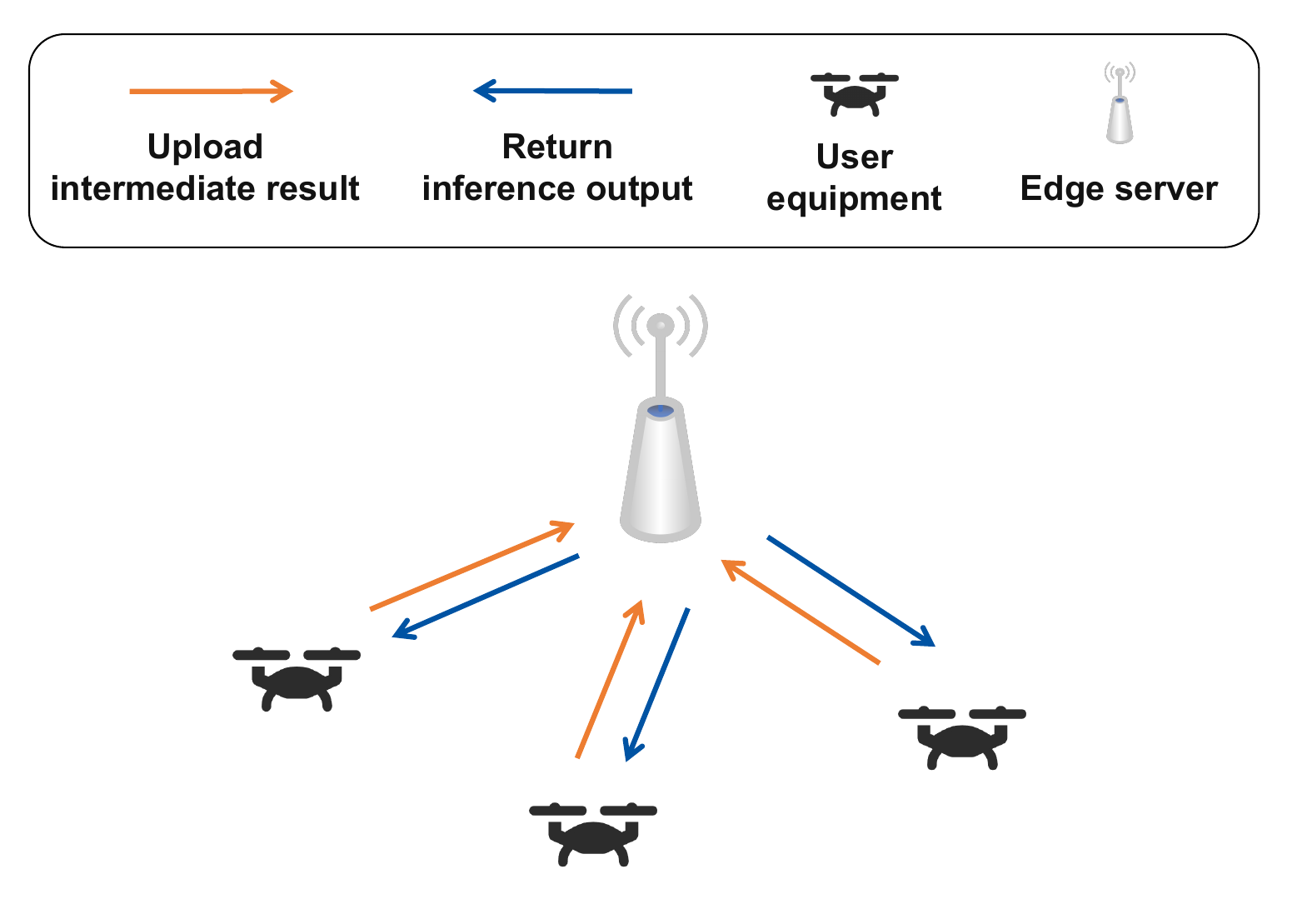}}
    \hspace{0.2in}
    \subfloat[ Each UE independently selects its partition points.]
    {\includegraphics[width = 0.44\textwidth]{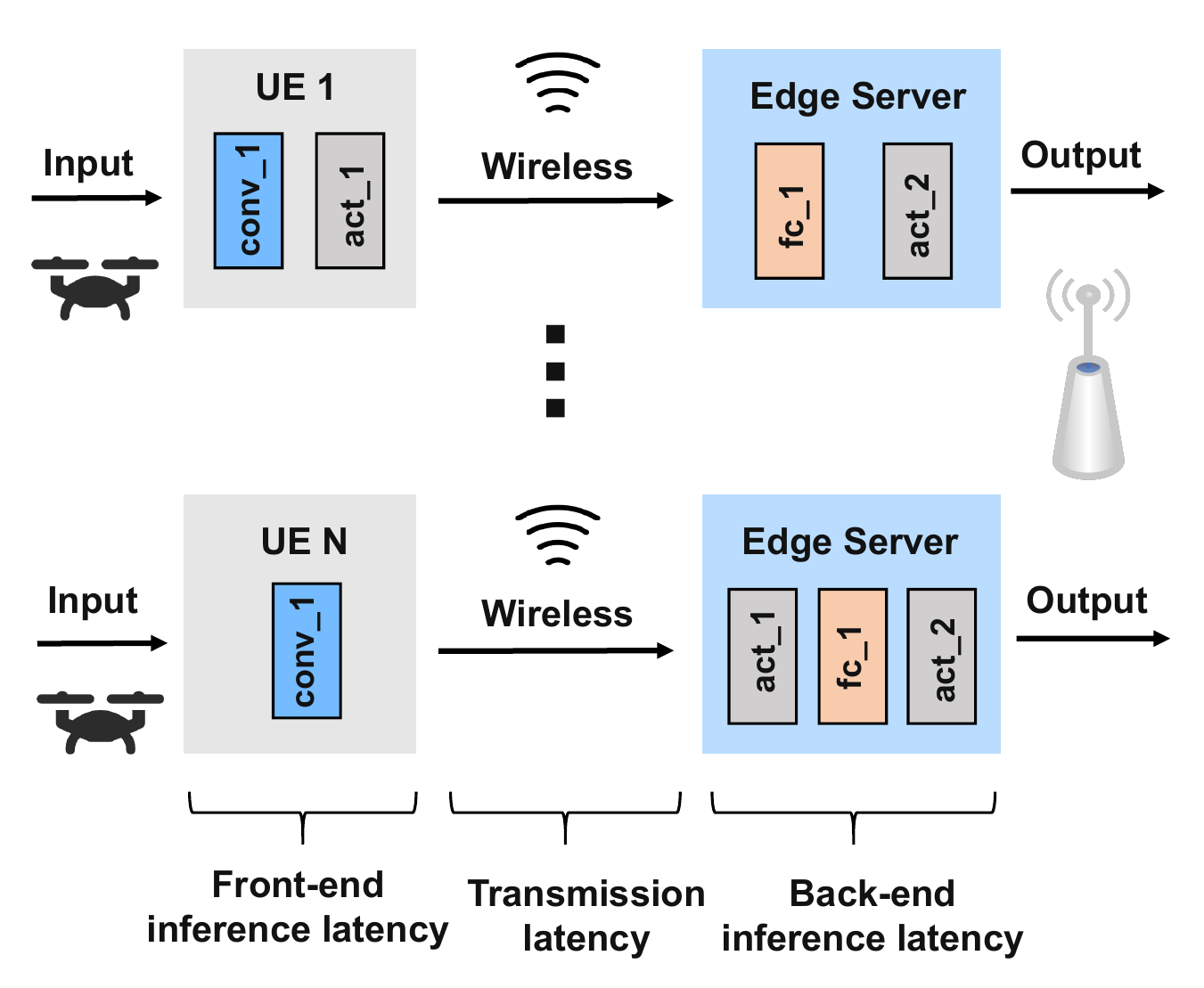}}
\caption{The system architecture of CANS. }
\label{fig:system}
\end{figure*}

Motivated by these observations, this paper presents a novel multiuser collaborative edge inference system called Cooperative Autodidactic NeuroSurgeon (CANS). CANS incorporates a well-designed distributed online learning module to efficiently identify optimal partition points for heterogeneous devices under limited latency feedback. Specifically,  CANS enables UEs to independently conduct collaborative edge inference with a common server, and to share local exploration experience with other devices for accelerating online learning processes. However, directly applying standard distributed online learning methods in CANS leads to significant learning inefficiencies due to two key challenges: (i) device heterogeneity and (ii) underutilization of offline inference experience. To address these issues, we propose a novel distributed linear bandit learning algorithm called FedLinUCB-DW, which mitigates these challenges by grouping devices of the same type and warm‑starting online learning stages using offline exploration experience collected by local early-exit inference. Our main contributions are summarized as follows.
\begin{enumerate}
    \item We take the first step to formulate the task of learning optimal DNN partition points in multiuser collaborative edge inference as a distributed online learning problem. Based on this formulation, we develop a novel CANS system, enabling multiple UEs to cooperate by sharing their exploration data, thereby accelerating  online learning and reducing the overall inference latency.
    \item  We identify two key limitations of directly applying traditional distributed online learning methods in CANS and further  propose a novel distributed linear bandit algorithm, FedLinUCB-DW, to mitigate these issues by (i) grouping UEs by device type and (ii) warm-starting the online learning process with local offline  inference experience.
    \item  We provide rigorous theoretical performance guarantees for the proposed algorithm by deriving an upper bound on the regret. We further conduct simulation experiments and implement our algorithm on prototype systems using NVIDIA Jetson Xavier NX and Orin Nano platforms. Extensive empirical results demonstrate that our approach outperforms baselines and aligns closely with our theoretical findings. Specifically, for prototype experiments, our proposed CANS reduced the average inference
latency by up to 50\% compared to the non-cooperative baseline.
\end{enumerate}

\textbf{Notations.} We use lower and upper case bold face letters to
denote vectors and matrices, respectively. For any positive integer $n$, we denote the set $\{1,2,\ldots,n\}$ by $[n]$. We use $\mathbf{I}$ to denote the identity matrix. For any vector $\mathbf{x}\in \mathbb{R}^d$ and positive semi-definite matrix $\boldsymbol{\Sigma}\in  \mathbb{R}^{d\times d}$, we denote $\Vert \mathbf{x}\Vert_{\boldsymbol{\Sigma}}=\sqrt{\mathbf{x}^\top \boldsymbol{\Sigma} \mathbf{x}}$, and by $\det(\boldsymbol{\Sigma})$ the determinant of $\boldsymbol{\Sigma}$. We use  $\lambda_{\max}(\boldsymbol{\Sigma})$, $\lambda_{\min}(\boldsymbol{\Sigma})$, $\lambda_{j}(\boldsymbol{\Sigma})$ denote the
largest, the smallest, and the $j$-th largest eigenvalue of matrix $\boldsymbol{\Sigma}$, respectively.

\section{Related Work}

\begin{table*}[ht]
  \centering  \normalsize
  \label{tab:offload_classification}
  \begin{tabular}{l|c|c|c|c|c}
    \hline
    & \cite{kang2017neurosurgeon,8737614,10413648} & \cite{zhang2021autodidactic,huang2023adversarial,chacun2024dronebandit} & \cite{9542866,9144210} &\cite{wang2024multi}& Ours\\
    \hline
 Offline  / Online & Offline& Online& Offline&Online& Online \\
    \hline
 Single-user / Multi-user & Single-user  &  Single-user& Multi-user&Multi-user& Multi-user\\
 \hline
 Known parameters  /  Unknown parameters &  Known&  Unknown & Known&Known& Unknown \\
    \hline
  \end{tabular}
   \caption{Comparison of Collaborative Edge Inference Approaches. We remark that by ``Online" we refer to the case the DNN partitioning is dynamically adjusted during the online inference stage.  By ``Multi-user” we refer to the case where the edge server provides inference offloading services for multiple UEs. The term ``Unknown parameters" indicates that the algorithm does not need to know the transmission rates and computing capacities of devices or servers in advance.}
\end{table*}

\textbf{Collaborative Edge Inference.} Prior research on collaborative edge inference has explored DNN‑partitioning optimization in both offline settings~\cite{kang2017neurosurgeon,8737614,10413648} and online settings \cite{zhang2021autodidactic,huang2023adversarial,11141772}. In the offline setting,~\cite{8737614} models the DNN partitioning problem as a min-cut problem and~\cite{10413648} proposes a fine-grained model partitioning method to jointly minimize inference latency and energy cost. In the online setting,~\cite{zhang2021autodidactic} first formulates the problem of identifying the optimal DNN partition as a linear bandit problem and solves it using a LinUCB-style algorithm. Based on it,~\cite{huang2023adversarial} further considers  the problem of selecting optimal offloading server. However, these works mainly focus on the case where the edge server provides offloading services for only one UE.

Recently, several works have targeted the multi‑user scenario in both offline~\cite{9542866,9144210} and online scenarios~\cite{wang2024multi}. However, these studies still assume that all  system parameters including transmission rates and computational capacities of devices or servers are stationary and known to the algorithm in advance, which is impractical in realistic applications.

A more recent work considers distributed inference among multiple devices in an online setting with unknown parameters~\cite{11141772}. Nevertheless, this work focuses on optimizing offloading path selection for a single inference request rather than providing offloading services for multiple UEs with distinct requests. In contrast, this paper considers the scenario where an edge server provides collaborative edge inference service for multiple UEs under uncertain environments, and proposes novel distributed online learning methods to efficiently identify optimal DNN partitions for UEs.

\textbf{Contextual Linear  Bandits.} The multi-armed bandit (MAB) problem is a well-established framework for modelling the fundamental trade-off between exploration and exploitation in sequential decision-making under uncertainty~\cite{lattimore2020bandit,bubeck2012regret}. Contextual linear bandit learning further extends the basic MAB model by incorporating contextual features of actions, and by assuming that the expected reward is a linear function of these features, tackling the challenge of large action spaces in MAB~\cite{dani2008stochastic}. A canonical algorithm in this setting is the Linear Upper Confidence Bound (LinUCB) method, which uses confidence-bound estimates within the linear model to guide action selection~\cite{abbasi2011improved}. 

Several studies have explored multi-agent linear bandit learning with decentralized data, where multiple agents collaborate under the coordination of a central server to solve contextual linear bandit problems. For example, \cite{DBLP:conf/iclr/WangHCW20} proposes several distributed linear bandit algorithms featuring event-triggered communication, and \cite{he2022simple} devises an asynchronous communication mechanism to tackle the challenge of stragglers in synchronous settings. Recently, \cite{blaser2024federated} develops a clustered bandit algorithm to overcome environment heterogeneity. However, all of these works overlook two important aspects: (i) more fine-grained environment heterogeneity, namely that the cluster structure among environments may apply only to a subset of the model parameters rather than the full parameter vector; and (ii) leveraging offline exploration data from agents to enhance online learning efficiency.


%% file: content/formulation.tex
In this section, we first describe how to formulate the task of identifying the optimal DNN partitioning in collaborative edge inference as a contextual linear bandit problem. We then extend this formulation to the multi-user setting by introducing a distributed linear bandit framework. 


\subsection{Contextual Linear Bandits for Optimal DNN Partitioning in  Collaborative Edge Inference}

In this subsection, we first introduce the procedure of collaborative edge inference and then formulate the task of identifying optimal DNN partitions as a contextual linear bandit problem.  

The basic idea of collaborative edge inference is to partition a DNN model into a front-end part running on the mobile device and a back-end part running on an edge server. Collaborative edge inference is thus more flexible in balancing the transmission and computation overhead compared with pure on-device processing or fully offloading.


\textbf{DNN Partitioning.}  Consider a resource-limited mobile device deployed with a pretrained DNN model over a wireless network that needs to execute DNN inference tasks. We denote $\mathcal{P}=\{0,1,2, \ldots, P\}$ as the set of all possible partition points (DNN layers). The mobile device can select a partition point $p\in\mathcal{P}$ to divide a DNN model into two parts: 1) the front-end part contains layers from the input to the partition point $p$ processed locally, and 2) the back-end part contains layers from $p$ to the output layer processed on the edge server. The two special
 cases $p = 0$ and $p=P$ correspond to the cases where the mobile device transmits raw input data to the edge server for executing the entire DNN model, and the entire inference task is run locally at the mobile device, respectively.



\textbf{DNN Inference Latency Model.} The end-to-end collaborative edge inference latency includes three main parts: 1) Front-end inference latency on the mobile device; 2) Transmission latency for transmitting the intermediate output from the device to the edge server; 3) Back-end inference latency on the edge server. Since the data size of the final inference result is generally small, the transmission latency for returning the final result can be ignored to simplify the problem formulation.

In this paper, we consider a limited feedback setting: each device observes only its own local front‑end inference latency and receives the back‑end/offloading latency from the edge server~\cite{zhang2021autodidactic}. The offloading latency reflects the total time elapsed between sending the intermediate output and receiving the inference result from the server. Consequently, for a given partition point $p$, the end‑to‑end inference latency is expressed as $\ell_p=\ell^f_p+\ell^b_p$, where $\ell^f_p$ and $\ell^b_p$ denote the local inference and edge offloading latencies, respectively.


 \textbf{Linear Prediction Model.} In this paper, we follow a similar idea proposed in~\cite{zhang2021autodidactic, huang2023adversarial} to learn a linear prediction model that maps contextual features of any partition point  $p\in\mathcal{P}$ to the end-to-end inference latency $\ell_p$. This linear prediction model is estimated using the past observed feedback information (inference latency of selected partition points), and is used for the subsequent identification of the optimal partition point. The advantage of using contextual features of partition points to learn a linear latency prediction model lies in tackling the challenge of searching a large decision space for the optimal choice~\cite{lattimore2020bandit}.  Besides, compared with learning more complex and non-linear models, the linear model is much more lightweight and needs fewer computational and storage  overhead on resource-constrained mobile devices~\cite{zhang2021autodidactic}.


In the following, we introduce two different constructions of contextual feature vectors that capture the characteristics of partition points at various levels,  proposed in~\cite{zhang2021autodidactic} and~\cite{huang2023adversarial}, respectively. Following the methodology of~\cite{zhang2021autodidactic}, we define the front-end feature vector for a specific partition point $p$ as: $\mathbf{x}^f_p = [m^f_{C}(p), m^f_{L}(p), m^f_{A}(p), n^f_{C}(p), n^f_{L}(p), n^f_{A}(p)]^T$. In this vector, the components $m^f_{C}(p)$, $m^f_{L}(p)$, and $m^f_{A}(p)$ represent the cumulative multiply-accumulate (MAC) operations for all convolutional, linear, and activation layers preceding and including the partition point $p$. Similarly, $n^f_{C}(p)$, $n^f_{L}(p)$, and $n^f_{A}(p)$ denote the total count of these respective layer types within the same front-end segment. The back-end feature vector is similarly defined as $\mathbf{x}^b_p=[m^b_{C}(p),m^b_{L}(p),m^b_{A}(p),n^b_{C}(p),n^b_{L}(p),n^b_{A}(p), \psi_p]^T$, further incorporating $\psi_p$ to denote the data size of the intermediate output at partition point $p$. 

Alternatively, following the simplified approach in~\cite{huang2023adversarial}, the contextual vectors can be streamlined to $\mathbf{x}^f_{p}=[\phi^f_p]$ and $\mathbf{x}^b_{p}=[\phi^b_p,\psi_p]^T$, where $\phi^f_p$ and $\phi^b_p$ aggregate the total computational workload (MAC operations) of the front-end and back-end layers, correspondingly.

We then introduce the linear latency prediction model considered in this paper. As mentioned earlier, the end-to-end inference latency $\ell_p:=\ell^f_p+\ell^b_p$ for a specific partition point $p\in \mathcal{P}$ consists of two parts: the front-end inference latency $\ell^f_p$ and the edge offloading latency $\ell^b_p$. We consider the prediction model of $\ell^f_p$ has the form of $\ell^f_p={\boldsymbol{\theta}}^f\cdot\mathbf{x}^f_p+\eta^f$, where  $\mathbf{x}^f_p\in \mathbb{R}^{d_f}$ denotes the contextual feature for the front-end inference part of partition point $p$, and ${\boldsymbol{\theta}}^f\in \mathbb{R}^{d_f}$ denotes the linear coefficients capturing the unknown effect of the computation capability of the mobile device on the front-end inference latency. $\eta^f$ models randomness in the front-end inference process. Similarly, the prediction model of $\ell^b_p$ has the form of $\ell^b_p={\boldsymbol{\theta}}^b\cdot\mathbf{x}^b_p+\eta^b$, where  $\mathbf{x}^b_p\in \mathbb{R}^{d_b}$ denotes the contextual feature for the edge offloading part of partition point $p$ and ${\boldsymbol{\theta}}^b\in \mathbb{R}^{d_b}$  denotes the linear coefficients capturing the unknown effects of the wireless
 uplink channel condition and the computation capability of the edge server on the edge offloading latency. $\eta^b$ characterizes the randomness in the transmission and back-end inference processes. 
 
Below we define $\Tf:=[\Tf^f,\Tf^b]\in \mathbb{R}^{d_f+d_b}$ as the overall system parameters and  $\xb:=[\xb^f_p,\xb^b_p]\in \mathbb{R}^{d_f+d_b}$ as the overall contextual features of partition point $p$. Consequently, the end-to-end latency $\ell_p$ of partition point $p$ can be modeled as: $\ell_p=\Tf\cdot\xb_p+\eta$, where $\eta:=\eta^f+\eta^b$ reflects the overall randomness. 


\textbf{Contextual Linear Bandits.} Now we introduce how to model the problem of identifying optimal  DNN partitioning as a contextual linear bandit problem. As introduced above, the mobile device deployed with a pretrained DNN model over a wireless network needs to execute inference tasks that arrive sequentially, denoted by the task index $t=1,\ldots,T$. For each task $t\in[T]$, the mobile device is required to select a partition point $p_t \in \mathcal{P}$ for performing collaborative edge inference with an edge server in the same wireless network. Once the overall inference process is finished, this mobile device can obtain the final inference result and calculate the end-to-end inference latency $\ell_{p_t}=\ell^{f}_{p_t}+\ell^{b}_{p_t}$ of $p_t$. Suppose that the linear coefficients ${\boldsymbol{\theta}}$ is a priori known,  the optimal partition point $p^*$ is apparently picked by minimizing the end-to-end inference latency by solving the following optimization problem: $p^*=\mathop{\arg\min}_{p \in \mathcal{P}} {\boldsymbol{\theta}}\cdot\mathbf{x}_p$.

 However, since both ${\boldsymbol{\theta}}^f$ and ${\boldsymbol{\theta}}^b$ are a priori unknown in practice,  the device must try different partition points and use the observed noisy feedback to obtain a good estimate of unknown linear coefficients. Hence, there is a typical \emph{exploitation vs. exploration trade-off} between picking the partition point by solving the above optimization problem based on the latest estimate,  and selecting other partition points to collect diverse feedback for constructing a more accurate estimate of linear coefficients. This is a classical contextual linear bandits problem and there are some well-developed algorithms like Linear Upper Confidence Bound (LinUCB)~\cite{abbasi2011improved} or Linear Thompson Sampling~\cite{DBLP:conf/icml/AgrawalG13} that can solve it.


%

\subsection{Asynchronous Distributed Linear Bandits for  Multiuser Inference-Offloading Systems}

\begin{table}[t]
\caption{ Notations used in online learning}
\label{tab:notations}
\centering
\small
\renewcommand{\arraystretch}{1.2}
\setlength{\tabcolsep}{6pt}
\begin{tabular}{|c|p{0.68\columnwidth}|}
\hline
\textbf{Notation} & \textbf{Description} \\ \hline
$\mathcal{P}$ & collection of all partition points \\ \hline
$T$ & number of rounds \\ \hline
$\ell_p^f,\ell_p^b$ & front-end/back-end latency for partition point $p$ \\ \hline
$\mathbf{x}_p^f,\mathbf{x}_p^b$ & front-end/back-end contextual feature vector for partition point $p$ \\ \hline
$\mathcal{M}$ & collection of all UEs \\ \hline
$\boldsymbol{\theta}^f_{m}$ & front-end parameters of device $m$ \\ \hline
$\boldsymbol{\theta}^b$ & back-end parameters \\ \hline
$\eta_t^f,\eta_t^b$ & front-end/back-end measurement noise \\ \hline
$m_t$ & active UE in round $t$ \\ \hline
$p_m^*$ & optimal partition point for UE $m$ \\ \hline
\end{tabular}
\vspace{-0.1in}
\end{table}

In this subsection, we consider a representative collaborative edge inference scenario where a set $\mathcal{M}$ of $|\mathcal{M}|=M$ UEs independently perform their respective DNN inference tasks with a resource-sufficient edge server over a wireless network.  

As discussed in the previous subsection, each UE in $\mathcal{M}$ is equipped with a pretrained DNN model and can select a partition point for conducting collaborative inference with the edge server. We assume that these mobile devices possess heterogeneous computational capabilities, while they all deploy the same DNN model. This scenario commonly arises in intelligent IoT applications, for example in video analytics or object-detection tasks where diverse devices such as sensors or UAVs run the same intelligent task with the assistance of an edge server~\cite{9144210,DBLP:conf/cogmi/WuLK21, zhu2026real}.

Given that all devices share identical backend computational resources on a common edge server and experience nearly identical wireless uplink conditions, the local latency feedback collected from one device can provide valuable information to accelerate the parameter estimation of other devices. \textbf{Accordingly, cooperative online learning across devices can be naturally leveraged to expedite the identification of individual optimal partition points.} Driven by this insight, we formulate the inter-device cooperative online learning framework as an \emph{asynchronous distributed linear bandit problem}, which is elaborated in the following sections.




In this paper, we consider the asynchronous communication setting. Specifically, in any given time slot, we assume that at most one UE is active in executing collaborative edge inference with the edge server, aligning with the standard assumption in asynchronous distributed linear bandits~\cite{he2022simple,li2022asynchronous}. This requirement can be realized using resource allocation or task scheduling strategies to coordinate device participation~\cite{9144210,dong2021joint}. In practice, concurrent inference requests can be serialized or scheduled by the edge server into a logical participation order. Under the assigned participation intervals for each device, there always exists a well-defined logical order indexed by $t \in [T]$. We refer to each index $t$ as a round, consistent with the terminology used in online learning literature.   In summary, only one device $m \in \mathcal{M}$ is active in round $t$ for conducting collaborative edge inference.


 
 Formally, at each round $t\in[T]$, an arbitrary device $m_t \in \mathcal{M}$ is active for participating in collaborative edge inference. This active device $m_t$ picks a partition point $p_t$ with the corresponding contextual feature $\mathbf{x}_{p_t}:=[\mathbf{x}^f_{p_t},\mathbf{x}^b_{p_t}] \in \mathbb{R}^{d_f+d_b}$ to conduct inference offloading with the edge server and obtains the corresponding end-to-end latency $\ell_t$ in this round. 
 
 
 As mentioned above, $\ell_t$ can be calculated as $\ell_t=\ell^f_t+\ell^b_t$, where $\ell^f_t$ denotes the front-end inference latency and $\ell^b_t$ denotes the edge offloading latency. According to our linear prediction model, the front-end inference latency $\ell^f_t$ is defined as $\ell^f_t:={\boldsymbol{\theta}}^f_{m_t}\cdot\mathbf{x}^f_{p_t}+\eta^f_t$, where  ${\boldsymbol{\theta}}^f_{m_t}\in\mathbb{R}^{d_f}$ is the linear coefficient reflecting the computation capacity of the active device $m_t$. Similarly, the edge offloading latency $\ell^b_t$ is defined as $\ell^b_t:={\boldsymbol{\theta}}^{b}\cdot\mathbf{x}^b_{p_t}+\eta^b_t$, where  ${\boldsymbol{\theta}}^{b}\in\mathbb{R}^{d_b}$ is the linear coefficient reflecting the joint effect of uplink channel conditions of the wireless network and the computation capacity of the edge server on edge offloading latency. 
 
 In this paper, since all mobile devices perform inference offloading to a common edge server over the same wireless network, we assume that the coefficient ${\boldsymbol{\theta}}^{b}$ is identical across all devices, while ${\boldsymbol{\theta}}^f_{m}$ may differ, thereby allowing heterogeneous computational capabilities across distinct devices $m\in\mathcal{M}$.  Below we define $\boldsymbol{\theta}_{m}:=[\boldsymbol{\theta}^f_{m}, \boldsymbol{\theta}^b]\in\mathbb{R}^{d_f+d_b}$ reflecting the overall parameters for $m\in\mathcal{M}$. The noises $\eta^f_t$ and $\eta^b_t$ capture the randomness inherent in the local inference and edge offloading processes.

The goal of all the participating mobile devices in the wireless network over $T$ rounds is to collaboratively minimize the total  cumulative regret, defined as
\begin{equation}
    \begin{aligned}
             \mathrm{Reg}(T):&=\sum_{t=1}^T \left(\Tf_{m_t}\cdot\xb_{p_t}-\min_{p\in\mathcal{P}}\boldsymbol{\theta}_{m_t}\cdot\mathbf{x}_{p} \right)\\
  &  =\sum_{t=1}^T \left\langle \boldsymbol{\theta}_{m_t},\mathbf{x}_{p_t}-\xb_{p^*_{m_t}}\right\rangle,
    \end{aligned}
\end{equation}
where $p^*_{m_t}:=\arg\min_{p\in\mathcal{P}}\Tf_{m_t}\cdot\xb_p$.

We then provide some further explanation of this learning objective. Notice that each individual term  $\Tf_{m_t}\cdot\xb_{p_t}-\min_{p\in\mathcal{P}}\boldsymbol{\theta}_{m_t}\cdot\mathbf{x}_{p}, \forall t\in [T]$   in the summation denotes the difference between the expected latency incurred by selecting the partition point $p_t$ selected by the active device $m_t$ in round $t$, and the expected latency that  $m_t$ would incur if it instead selected its optimal partition point. \textbf{Hence, a smaller cumulative regret  $\mathrm{Reg}(T)$ implies superior inference-acceleration performance among all mobile devices involved in this multiuser collaborative edge-inference system.}

To facilitate the theoretical analysis, we make the following assumptions on the noise terms, system parameters, and contextual features, which are standard in the contextual linear bandit literature~\cite{he2022simple,li2022asynchronous}. We first impose the following assumption on the random noise $\eta^f_t$ and $\eta^b_t$.
\begin{assumption} \label{assumption:noise}
For $i\in\{f,b\}$ and $t\in[T]$, $\eta^i_t$ is conditionally $R_i$-sub-Gaussian given $\mathbf{x}_{p_{1:t}}$, $m_{1:t}$ and $\ell_{1:t-1}$, i.e.,  $\forall \lambda \in \mathbb{R}$,
 \begin{align}
     \mathbb{E}\left[ \exp(\lambda \eta^i_t) \vert \mathbf{x}_{p_{1:t}}, m_{1:t}, \ell_{1:t-1}  \right]\leq \exp(R_i^2\lambda^2/2).
 \end{align}
\end{assumption}

We also assume that the system parameters and contextual features are bounded. Specifically, for the system parameters, we assume that $\Vert \boldsymbol{\theta}^f_{m} \Vert \leq S_f$ for all $m\in \mathcal{M}$ and $\Vert \boldsymbol{\theta}^b \Vert \leq S_b$. Similarly, for the contextual features, we assume that $\Vert \mathbf{x}^{f}_p \Vert \leq L_f$ and $\Vert \mathbf{x}^{b}_p \Vert \leq L_b$ for any partition point $p\in\mathcal{P}$. 






%% file: content/algorithm.tex

In this section, we introduce our proposed collaborative edge inference system, termed Cooperative Autodidactic NeuroSurgeon (CANS), tailored for the multi-user setting.    The key idea behind CANS is to enable cooperation among devices by sharing locally collected feedback to accelerate the identification of each device’s optimal DNN partition. This cooperative decision-making process naturally leads to a distributed linear bandits formulation.

Although there are existing well-established  algorithms such as DisLinUCB~\cite{DBLP:conf/iclr/WangHCW20} and FedLinUCB~\cite{he2022simple} designed for the distributed linear bandits task,  we observe two significant challenges when directly applying these  algorithms to our proposed CANS system.  Therefore, to address these challenges, we further devise a novel distributed online learning algorithm for our CANS system.

\subsection{FedLinUCB and its Limitations in Multiuser Inference-Offloading Systems.}

\begin{figure*}[ht]
    \centering
    \includegraphics[width = 0.8\textwidth]{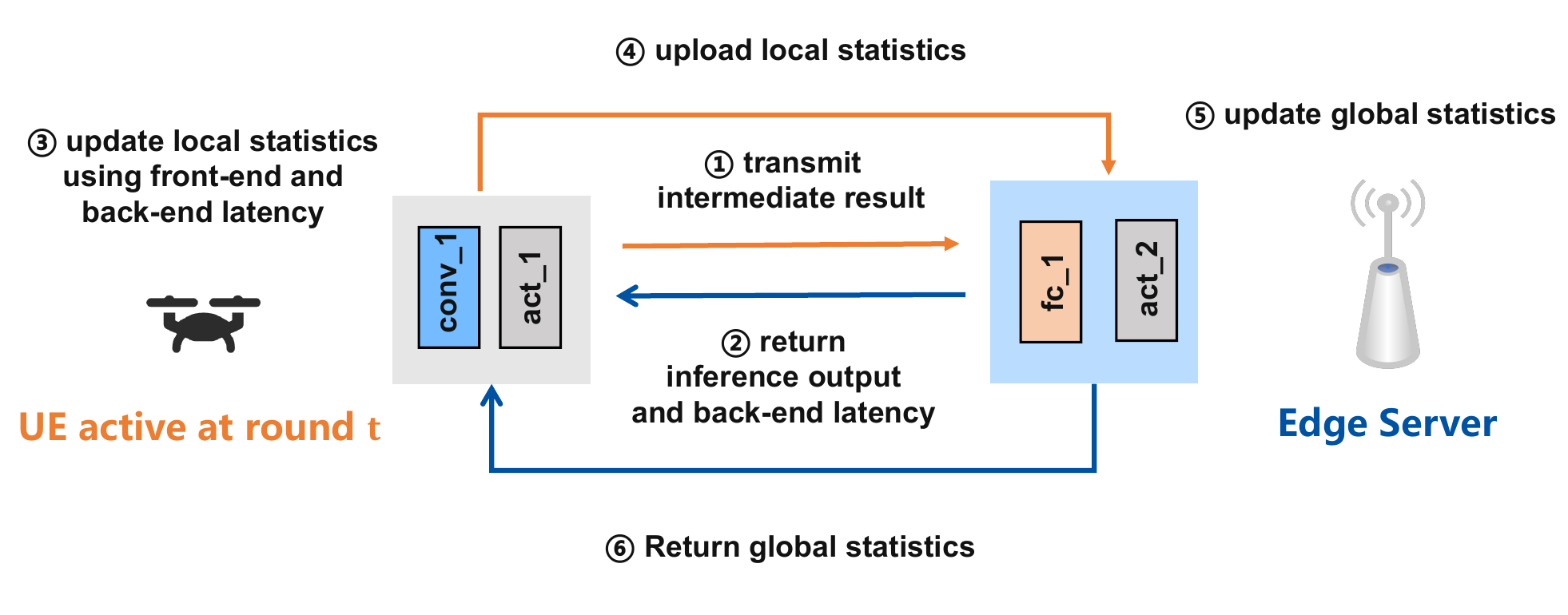}
    \caption{FedLinUCB for asynchronous multiuser collaborative edge inference.}
    \label{fig:coherence}
\end{figure*}

In this subsection, we first introduce a prevalent distributed online learning algorithm termed FedLinUCB~\cite{he2022simple} (Federated Linear UCB) and explain how it works in solving the asynchronous distributed linear bandits problem. We then discuss its limitations when applied in the considered multiuser collaborative edge inference system.

The core component of the FedLinUCB framework is the well-known LinUCB (Linear Upper Confidence Bound) algorithm, a canonical online learning method tailored for linear bandits based on  online linear regression. Specifically, LinUCB updates the linear coefficient of the latency-prediction model using historical feedback via ridge regression. For each round  $t\in[T]$ (task that sequentially arrives), the learner (mobile device) selects an action (partition point) by jointly considering its latency estimated by the current model and an exploration bonus reflecting the uncertainty of prediction, thus balancing the exploration-exploitation trade-off. In specific, the mobile device maintains sufficient statistics $\boldsymbol{\Sigma}_{t}\in\mathbb{R}^{(d_f+d_b)  \times (d_f+d_b) }$ and $\mathbf{b}_{t}\in \mathbb{R}^{(d_f+d_b) \times 1}$ to estimate the linear coefficient $\boldsymbol{\theta}\in\mathbb{R}^{(d_f+d_b)  \times 1}$. For each round $t\in[T]$, the unknown coefficient is estimated by $\widehat{\boldsymbol{\theta}}(t)= \boldsymbol{\Sigma}^{-1}_{t} \mathbf{b}_{t} $, and  the partition point  is selected to be
\begin{align}
    p_t=\mathop{\arg\min}_{p \in \mathcal{P}} \langle\widehat{\boldsymbol{\theta}}(t), \mathbf{x}_p\rangle - \beta\Vert \mathbf{x}_p\Vert_{\boldsymbol{\Sigma}^{-1}_{t}},
\end{align}
where $\langle\widehat{\boldsymbol{\theta}}(t), \mathbf{x}_p\rangle $ denotes the predicted end-to-end latency of partition point $p$ using the current estimate $\widehat{\boldsymbol{\theta}}(t)$.  $\beta\Vert \mathbf{x}_p\Vert_{\boldsymbol{\Sigma}^{-1}_{t}}$ represents the confidence interval of latency prediction.  This term reflects the exploration bonus tailored for linear bandits to balance the exploration-exploitation trade-off in online learning. After the inference offloading finishes, i.e., the final inference result is returned to the device, the practical inference latency $\ell_{t}=\ell^f_t+\ell^b_t$ of selected partition point $p_t$ can be obtained by the device. Then the device updates sufficient statistics as $\boldsymbol{\Sigma}_{t+1}\gets \boldsymbol{\Sigma}_{t} +\mathbf{x}_{p_t} \mathbf{x}^{\top}_{p_t} $ and $\mathbf{b}_{t+1}\gets \mathbf{b}_{t} +  \mathbf{x}_{p_t} \ell_t$.

Then we introduce the basic idea of  applying  FedLinUCB~\cite{he2022simple} in  asynchronous multiuser inference-offloading systems. In round $t\in[T]$, device $m_t$ participates in inference offloading with the edge server. $m_t$ first uses the local prediction model $\widehat{\Tf}_{m_t}(t)$ to select the partition point $p_t\in\mathcal{P}$ based on its current estimate of the inference latency with the corresponding exploration bonus. After receiving the practical latency $\ell_t$ of $p_t$, $m_t$ updates its local sufficient statistics.  Following the elaborate event-triggered communication protocol, $m_t$ will upload its local sufficient statistics to the edge server so that it can help other devices efficiently estimate latency prediction models.   The server then incorporates these uploaded statistics into its global sufficient statistics and returns the latest global statistics to  $m_t$, allowing $m_t$ to use the downloaded global statistics to update $\widehat{\Tf}_{m_t}(t)$. Meanwhile, the local sufficient statistics for all other devices $m \neq m_t, m \in \mathcal{M}$ remain unchanged for this round.

In addition, we observe that in FedLinUCB, the communication of exchanging sufficient statistics involves only  device $m_t$ active in round $t $ and the  server, remaining completely independent of other inactive agents in that round. This asynchronous communication mechanism  is well-suited for multiuser collaborative edge inference systems, as it avoids the straggler issue commonly encountered in synchronous settings. 

However, we notice that there are two major limitations in the original FedLinUCB algorithm for it to work effectively in our proposed CANS system:
\begin{itemize}
    \item \textbf{Cooperation fails among heterogeneous devices.} The original FedLinUCB naively aggregates all devices' local sufficient statistics to estimate the overall coefficients $\Tf^{m}$ including $\Tf^f_m$ and $\Tf^b$  for each device $m\in\mathcal{M}$. However, when devices possess heterogeneous computational capacities, the  parameters $\{\Tf_m\}_{m\in\mathcal{M}}$ vary significantly. In particular, the distance $\Vert \Tf_m-\Tf_n \Vert=\Vert \Tf^f_m-\Tf^f_{n} \Vert$ can be large for any distinct devices $m\neq n$. In such cases, directly sharing exploration experiences will even \emph{worsen} the online learning performance of each device. 
    \item \textbf{Inefficient utilization of offline inference experiences.} In FedLinUCB, all devices initialize their local sufficient statistics with $\Sg_{m,1}=\lambda\mathbf{I}$ and $\Bf_{m,1}=\mathbf{0}$. This design overlooks the valuable local  offline inference experiences (e.g., from early-exit/multi-exit DNN inference), which could serve as a warm start for the online learning stage and thereby enhance overall learning efficiency.
\end{itemize}




\subsection{FedLinUCB with Device Grouping and Local Warm-Start}

Given the limitations of the original FedLinUCB outlined above, we devise a novel distributed online learning algorithm, called FedLinUCB with  Device Grouping and Local Warm-Start (FedLinUCB-DW) to support our proposed CANS framework, which addresses these shortcomings based on theoretically grounded mechanisms.

\begin{figure*}[ht]
    \centering
    \includegraphics[width = 0.7\textwidth]{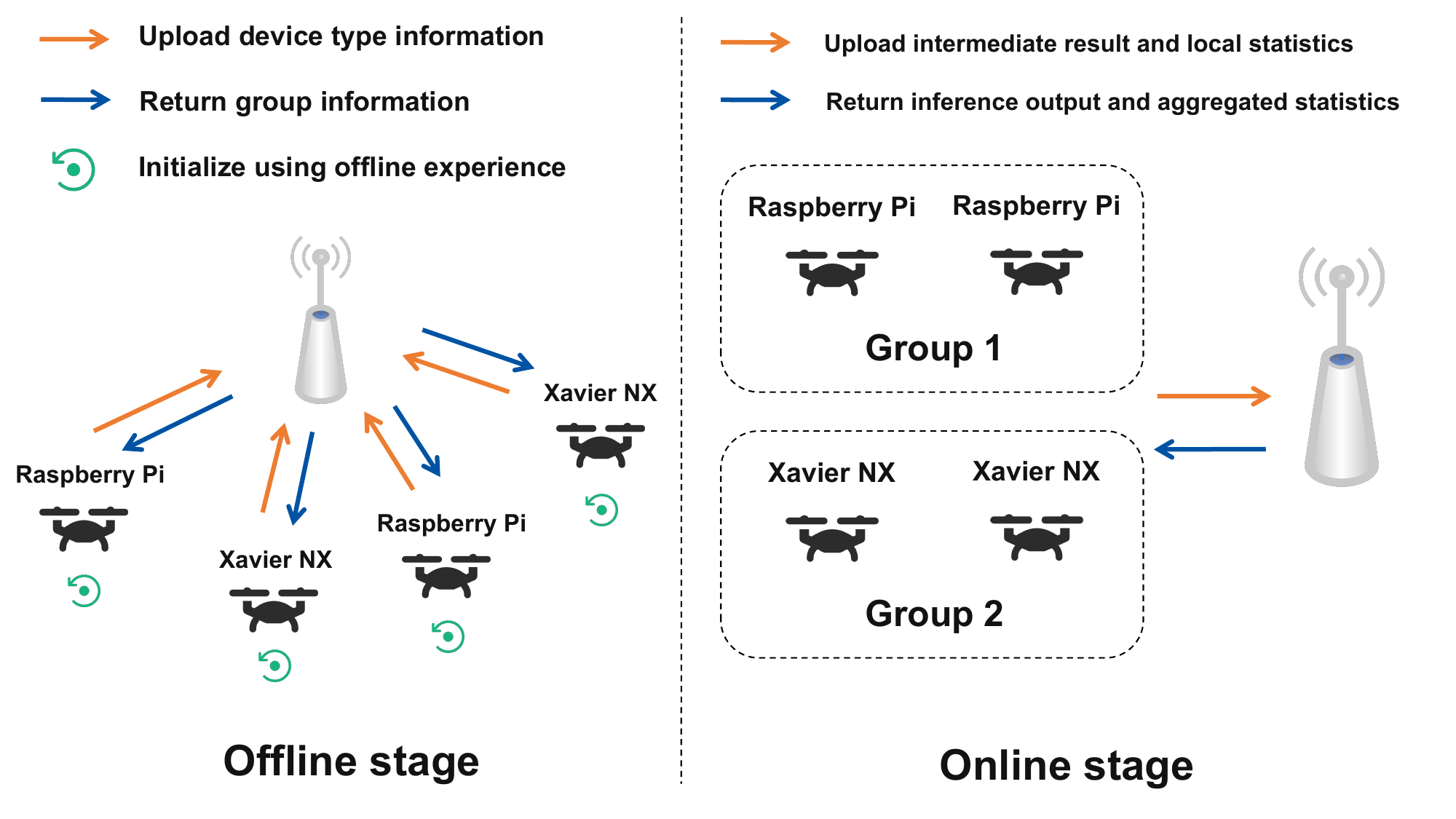}
    \caption{The workflow of FedLinUCB-DW.}
    \label{fig:fedlinucb-dw}
\end{figure*}

\textbf{Heterogeneity-aware device grouping.} As described above, naive aggregation of sufficient statistics for the overall parameter vector $ \Tf_{m} = [\Tf^f_m, \Tf^b] $ in  FedLinUCB deteriorates cooperative online learning under device heterogeneity. In contrast, our proposed FedLinUCB-DW method does not share the sufficient statistics for the overall parameters across all devices.  Instead, the edge server first divides the devices into different groups according to the prior knowledge of device types, which is sent to the server before the cooperative online learning begins.  Since the computational capability of each device strongly depends on its type, we can assume that devices within the same group share identical or closely similar front-end parameters $ \Tf^f_m $. Based on this heterogeneity-aware devices grouping, \textbf{the server can only share information related to the front-end part among devices within the same group, while sharing information related to the back-end part among all the devices.} Below, we provide the mathematical formulation of  devices grouping in FedLinUCB-DW.

\begin{figure}[ht]
    \centering
    \includegraphics[width = 0.5\textwidth]{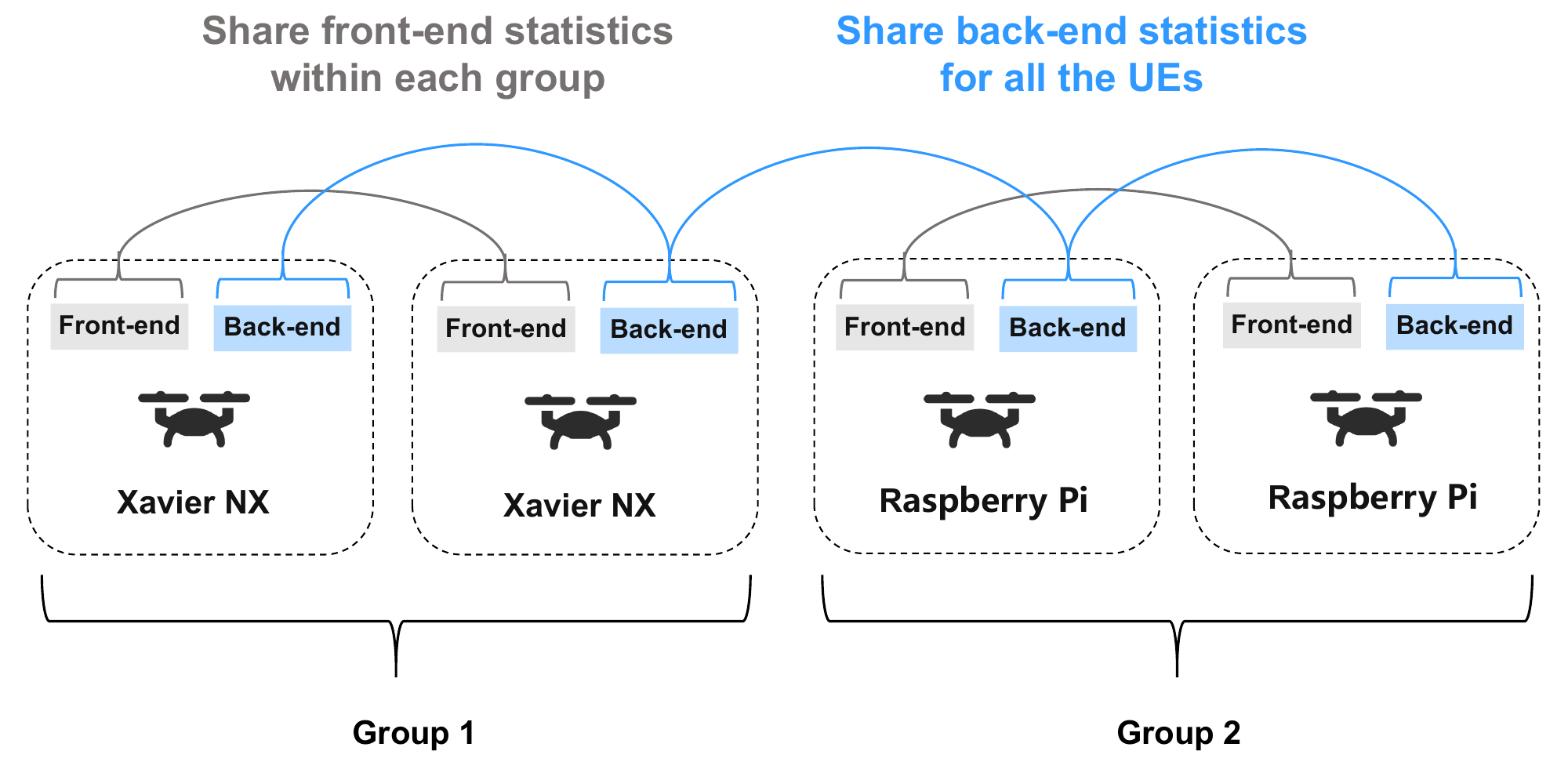}
    \caption{Illustration of device grouping. The server uses device-type information to assign UEs running on Xavier NX devices to one group and UEs running on Raspberry Pi devices to another. During the online stage, the server aggregates front-end statistics related to device capacity within each group. In contrast, back-end statistics are aggregated across all devices.}
    \label{fig:heterogrouping}
    \vspace{-0.1in}
\end{figure}

Formally, assume that a total of $\vert \mathcal{M}\vert=M$ devices participate in CANS,  and these devices can be divided into groups represented as $\mathcal{G}=\{\mathcal{G}_1,\mathcal{G}_2,\ldots,\mathcal{G}_K\}$, where $\mathcal{G}_k, \forall k \in [K]$ denotes the collection of devices with the same type and $K$ denotes the number of device types in CANS. For example, if there are three types of devices (Xavier NX, Orin Nano and Raspberry Pi) participating in CANS, then $K=3$ and  $\mathcal{G}=\{\mathcal{G}_1,\mathcal{G}_2,\mathcal{G}_3\}$. We use the below assumption about the Intra-group similarity for theoretical analysis.
\begin{assumption}\label{assumption:hetero}
    For any two different devices $m,n$ within a particular group $\mathcal{G}_k\in\mathcal{G}$, we have $\Vert \Tf^f_m-\Tf^f_n \Vert\leq \epsilon$.
    \begin{remark}
This assumption is more flexible than those commonly used in prior clustered bandit studies,  which either assume that all devices within a group share identical parameters~\cite{liu2022federated}, or presume the grouping structure exists w.r.t. the overall parameter $\Tf$~\cite{blaser2024federated}. In contrast, our assumption allows devices of the same type to have closely similar—but not exactly identical—computational capacities. We believe this assumption is more realistic in practical applications. 
    \end{remark}
\end{assumption}


\begin{table}[t]
\caption{Notations used in FedLinUCB-DW}
\label{tab:notations}
\centering
\small
\renewcommand{\arraystretch}{1.2}
\setlength{\tabcolsep}{6pt}
\begin{tabular}{|c|p{0.68\columnwidth}|}
\hline
\textbf{Notation} & \textbf{Meaning} \\ \hline
$\widehat{\Tf}^f_m(t), \widehat{\Tf}^b_m(t)$ & estimate of $\Tf^f_m$ and $\Tf^b$ \\ \hline
$\Sg^{f}_{m,t},\Bf^{f}_{m,t}$ & statistics used to estimate $\widehat{\Tf}^f_m(t)$ \\ \hline
$\Sg^{b}_{m,t},\Bf^{b}_{m,t}$ & statistics used to estimate $\widehat{\Tf}^b_m(t)$ \\ \hline
$\Sg^{f,\mathrm{loc}}_{m,t},\Bf^{f,\mathrm{loc}}_{m,t}$ & local  front-end statistics for device $m$ \\ \hline
$\Sg^{b,\mathrm{loc}}_{m,t},\Bf^{b,\mathrm{loc}}_{m,t}$ & local back-end statistics for device $m$ \\ \hline
$\Sg^{f,\mathrm{ser}}_{\mathcal{G}_k,t},\Bf^{f,\mathrm{ser}}_{\mathcal{G}_k,t}$ & server-side front-end statistics for group $\mathcal{G}_k$ \\ \hline
$\Sg^{b,\mathrm{ser}}_{t},\Bf^{b,\mathrm{ser}}_{t}$ & server-side back-end statistics \\ \hline
$\Sg^{f,\mathrm{off}}_{m},\Bf^{f,\mathrm{off}}_{m}$ & offline front-end statistics for device $m$ \\ \hline
\end{tabular}
\end{table}

During the cooperative online learning stage, devices in CANS  should maintain two kinds of estimated coefficients  $\widehat{\Tf}^f_m(t)$ and $ \widehat{\Tf}^b_m(t)$ used to estimate the front-end and back-end latency, respectively. Each device $m\in\mathcal{M}$ learns $\Tf^b$ with all the other devices $n\neq m\in\mathcal{M}$ and learns $\Tf^f_m$ with devices in the same group. This is because we assume that the back-end parameters $\Tf^b$ are the same across all the devices, while the front-end parameters $\Tf^f_m$ are closely similar within a group (devices with the same type). We introduce the complete workflow of FedLinUCB-DW in Algorithm~\ref{alg:main}. For clarity, we summarize the related notations in Table~\ref{tab:notations}. 

\begin{algorithm}[t]
	\caption{FedLinUCB-DW}
	\label{alg:main}
	\begin{algorithmic}[1]\label{algorithm:1}
     \STATE Edge server uses the uploaded device type information to divide $M$ devices into  $K$ groups: $\mathcal{G}=\{\mathcal{G}_k\}_{k\in[K]}$ 
	\STATE Initialize $\Sg^{f}_{m,1}=\Sg^{f}_{\mathcal{G}_k,1}=\lambda_f \mathbf{I}+\sum_{n\in\mathcal{G}_k}\Sg^{f,\mathrm{off}}_{n}$, $\Sg^{b}_{m,1}=\Sg^{b}_{1}=\lambda_b\mathbf{I}$,   $\Bf^{f}_{m,1}=\sum_{n\in\mathcal{G}_k}\Bf^{f,\mathrm{off}}_{n}$, $\Bf^{b}_{m,1}=\mathbf{0}$, $\Sg^{f,\mathrm{loc}}_{m,t}=\mathbf{0}$, $\Sg^{b,\mathrm{loc}}_{m,t}=\mathbf{0}$, $\Bf^{f,\mathrm{loc}}_{m,t}=\mathbf{0}$,  $\Bf^{b,\mathrm{loc}}_{m,t}=\mathbf{0}$, $\widehat{\Tf}^f_m(1)=(\Sg^{f}_{m,1})^{-1}\Bf^{f}_{m,1}$, $ \widehat{\Tf}^b_m(1)=\mathbf{0}$  for each device $m \in [M]$
	\FOR{round $t = 1, 2, \dots, T$}
	\STATE Device $m_t$ is active and selects $p_t$ based on Eq.~\eqref{eq:optimistic_selection} \alglinelabel{algmain:agent}
	\STATE Device $m_t$ executes DNN inference offloading with the edge server based on $p_t$, then obtains $\ell^f_t$ and $\ell^b_t$ 
	\alglinelabel{algmain:reward}
\FOR{$i\in\{f,b\}$}
	\STATE $\Sg^{i,\loc}_{m_t, t}\gets \Sg^{i,\loc}_{m_t, t-1} 
    + \xb^i_{p_t} (\xb^i_{p_t})^{\top}$, \quad $\Bf^{i,\loc}_{m_t, t} \gets 
   \Bf^{i,\loc}_{m_t, t-1} + \ell^i_t \xb^i_{p_t}$
    \alglinelabel{algmain:local_update} {\color{blue}\texttt{//Local update }}
	\IF{$\frac{\det(\Sg^i_{m_t, t} + \Sg^{i,\loc}_{m_t, t})}{\det(\Sg^i_{m_t, t})} > 1 + \alpha_i
    $} \alglinelabel{algmain:criterion} 
    \STATE Device $m_t$ sends $\Sg^{i,\loc}_{m_t, t}$ and $\Bf^{i,\loc}_{m_t, t}$ to 
    server \alglinelabel{algmain:upload} {\color{blue}\texttt{//Upload }}
        \STATE $\textbf{Global\_Update}(\Sg^{i,\loc}_{m_t, t},\Bf^{i,\loc}_{m_t, t})$
        \STATE $\Sg^{i,\loc}_{m_t, t} \gets 0$,\quad $\Bf^{i,\loc}_{m_t, t} \gets 
    0$ \alglinelabel{algmain:reset_local} {\color{blue}\texttt{//Reset local }}
    \STATE $\widehat{\Tf}^i_{m_t}(t+1) \gets {(\Sg^i_{m_t, t+1}})^{-1} \Bf^i_{m_t, t+1}$ \alglinelabel{algmain:estimate2} \\
    {\color{blue}\texttt{//Compute estimate }}
    \ELSE
    \STATE $\Sg^i_{m_t, t+1} \gets \Sg^i_{m_t, t}$,\quad $\Bf^i_{m_t, t+1} 
    \gets \Bf^i_{m_t, t}$, \quad $\widehat{\Tf}^i_{m_t}(t+1) \gets \widehat{\Tf}^i_{m_t}(t)$\alglinelabel{algmain:no_update}
	\ENDIF
\ENDFOR
	\FOR{other inactive agent $m\in [M] \setminus \{m_t\}$}
	\STATE $\Sg^f_{m, t+1} \gets \Sg^f_{m, t}$,\quad $\Bf^f_{m, t+1} 
    \gets \Bf^f_{m, t}$, $\Sg^b_{m, t+1} \gets \Sg^b_{m, t}$,\quad $\Bf^b_{m, t+1} 
    \gets \Bf^b_{m, t}$, \quad$\widehat{\Tf}^f_m(t+1) \gets \widehat{\Tf}^f_m(t)$,  $\widehat{\Tf}^b_m(t+1) \gets \widehat{\Tf}^b_m(t)$
    \alglinelabel{algmain:inactive_agents}
	\ENDFOR
	\ENDFOR
	\end{algorithmic}
\end{algorithm}

In each round $t\in[T]$, device $m_t\in \mathcal{G}_k$ is active and selects a partition point $p_t\in\mathcal{P}$ to conduct collaborative edge inference based on the optimistic criterion defined as
\begin{equation}\label{eq:optimistic_selection}
    \begin{aligned}
        p_t=\arg&\min_{p\in\mathcal{P}}
\Bigg\{
\left\langle \widehat{\boldsymbol{\theta}}^f_{m_t}(t),\mathbf{x}^f_p\right\rangle
-
\mathrm{CB}^f_{m_t,p}(t)\\
&+
\left\langle \widehat{\boldsymbol{\theta}}^b_{m_t}(t),\mathbf{x}^b_p\right\rangle
-
\mathrm{CB}^b_{m_t,p}(t)
\Bigg\},
    \end{aligned}
\end{equation}
where $\mathrm{CB}^f_{m_t,p}(t):=\beta^f_{k,t}\|\mathbf{x}^f_p\|_{(\mathbf{\Sigma}^f_{m_t,t})^{-1}}$ and $\mathrm{CB}^b_{m_t,p}(t):=\beta^b_t\|\mathbf{x}^b_p\|_{(\mathbf{\Sigma}^b_{m_t,t})^{-1}}$. These two terms reflect the uncertainty of the estimated front-end and back-end latency, respectively, which encourages exploration. The hyper-parameters $\beta^f_{k,t}$ and $\beta^b_t$ will be introduced below.

Once the inference offloading finishes,  $m_t$ obtains the local inference latency $\ell_t^f$ locally and the back-end inference latency $\ell_t^b$ returned by the server. Then, $m_t$ updates its local front-end and back-end statistics using $\ell^f_t$ and $\ell^b_t$, respectively (Line~\ref{algmain:local_update} in Alg.~\ref{alg:main}). Following FedLinUCB~\cite{he2022simple}, we also adopt the matrix-determinant-based criterion (Line~\ref{algmain:criterion} in Alg.~\ref{alg:main}) to decide when the information accumulated in the current local statistics is sufficient to upload to the server for updating the global model. Due to computational heterogeneity across devices, we apply this criterion independently to the front-end and back-end parts. If this criterion is satisfied, it implies that the local statistics would help significantly reduce the uncertainty of estimating the front-end (back-end) coefficient $\Tf^f_m$ ($\Tf^b$). Consequently, $m_t$ will share its progress by uploading the local statistics to the edge server (Line~\ref{algmain:upload} in Alg.~\ref{alg:main}).

\begin{algorithm}[t]
	\caption{Global updating sub-routine}
	\label{alg:global_update}
	\begin{algorithmic}[1]\label{algorithm:2}
	\STATE Input $\Sg^{i,\loc}_{m_t, t},\Bf^{i,\loc}_{m_t, t}$
\IF{$i=f$}
    \STATE $\Sg^{f,\ser}_{\mathcal{G}_k,t} \gets \Sg^{f,\ser}_{\mathcal{G}_k,t} + \Sg^{f,\loc}_{m_t, t}$, \quad
    $\Bf^{f,\ser}_{\mathcal{G}_k,t} \gets \Bf^{f,\ser}_{\mathcal{G}_k,t}  + \Bf^{f,\loc}_{m_t, t}$\alglinelabel{algmain:update_server_f}
    {\color{blue}\texttt{//Global update for front-end statistics of group $\mathcal{G}_k$ }}
    \STATE Server sends $\Sg^{f,\ser}_{\mathcal{G}_k,t}$ and $\Bf^{f,\ser}_{\mathcal{G}_k,t}$ back to agent $m_t$ \alglinelabel{algmain:download_f} {\color{blue}\texttt{//Download }}
    \STATE $\Sg^f_{m_t, t+1} \gets \Sg^{f,\ser}_{\mathcal{G}_k,t}$,\quad $\Bf^f_{m_t, t+1} 
    \gets \Bf^{f,\ser}_{\mathcal{G}_k,t}$ \alglinelabel{algmain:estimatef}
    \ELSE
    \STATE $\Sg^{b,\ser}_t \gets \Sg^{b,\ser}_t + \Sg^{b,\loc}_{m_t, t}$, \quad
    $\Bf^{b,\ser}_t \gets \Bf^{b,\ser}_t  + \Bf^{b,\loc}_{m_t, t}$\alglinelabel{algmain:update_server_b}
    {\color{blue}\texttt{//Global update for back-end statistics of all devices}}
    \STATE Server sends $\Sg^{b,\ser}_t$ and $\Bf^{b,\ser}_t$ back to agent $m_t$ \alglinelabel{algmain:download_b} {\color{blue}\texttt{//Download }}
    \STATE $\Sg^b_{m_t, t+1} \gets \Sg^{b,\ser}_t$,\quad $\Bf^b_{m_t, t+1} 
    \gets \Bf^{b,\ser}_t$ \alglinelabel{algmain:estimateb}
    \ENDIF
	\end{algorithmic}
\end{algorithm}

Then the server updates the global statistics according to the global updating sub-routine (Algorithm~\ref{alg:global_update}). If the uploaded local statistics pertains to the front-end ($\Sg^{f,\loc}_{m_t,t}, \Bf^{f,\loc}_{m_t,t}$), then the server updates the front-end global statistics $\Sg^{f,\ser}_{\mathcal{G}_k,t}$ and $\Bf^{f,\ser}_{\mathcal{G}_k,t}$, for the group $\mathcal{G}_k$ to which device $m_t$ belongs (i.e., $m_t\in \mathcal{G}_k$) (Line~\ref{algmain:update_server_f} in Alg.~\ref{alg:global_update}). Otherwise, the uploaded statistics is ${\Sg^{b,\loc}_{m_t,t}, \Bf^{b,\loc}_{m_t,t}}$, and the server updates the back-end global statistics ($\Sg^{b,\ser}_{t},\Bf^{b,\ser}_{t}$), which can be utilized by all devices (Line~\ref{algmain:update_server_b} in Alg.~\ref{alg:global_update}).  Afterwards, $m_t$ downloads the latest global statistics from the server, and  updates its corresponding
 local statistics and estimated parameters.  
 
If the criterion is not met, the communication between device $m_t$ and the server is not triggered. In this case, the newly collected information remains in the local buffer of $m_t$, while the server-side statistics and the estimator used by $m_t$ remain unchanged until the next synchronization (Line~\ref{algmain:no_update} in Alg.~\ref{alg:main}). At last, all the other inactive devices remain unchanged (Line~\ref{algmain:inactive_agents} in Alg.~\ref{alg:main}).



\textbf{Warm-start with local offline inference experiences.} As is known, devices deployed with  DNN models can also execute efficient inference locally based on early exit/multi-exit DNN inference without the help of inference offloading. Early-exit inference allows samples to terminate inference at intermediate layers when sufficient prediction confidence is achieved, reducing computational cost~\cite{DBLP:journals/csur/PSCPC25,DBLP:conf/icc/PachecoSCMHC23}. Devices perform early-exit inference offline and record exit points and corresponding inference latency, which can be used as offline exploration experience to warm-start online collaborative inference.


In this paper, we assume that before cooperative online learning starts, each device $m\in\mathcal{M}$ is equipped with offline local inference experiences, denoted by $D^{\mathrm{off}}_m=\{\Tilde{\xb}^f_{m,\tau}, \Tilde{\ell}^f_{m,\tau}\}_{\tau=1}^{n_{\mathrm{off}}}$. Specifically, the pair $(\Tilde{\xb}^f_{m,\tau}, \Tilde{\ell}^f_{m,\tau})$ represents the $\tau$-th offline inference experience of device $m$, which is collected when device $m$ selects the partition point $p_\tau\in\mathcal{P}$ to perform early-exit inference. Here, $\Tilde{\xb}^f_{m,\tau}$ denotes the contextual feature associated with partition point $p_\tau$, while $\Tilde{\ell}^f_{m,\tau}$ denotes the observed inference latency. We further assume that the computing capacity of each device remains unchanged between the offline and online stages. Therefore, the offline latency follows the same linear model as the online latency, i.e., $\Tilde{\ell}^f_{m,\tau}=\langle\boldsymbol{\theta}^f_m,\Tilde{\xb}^f_{m,\tau}\rangle+\Tilde{\eta}^f_{m,\tau}$, where $\Tilde{\eta}^f_{m,\tau}$ denotes the corresponding observation noise.

Based on the offline experiences $D^{\mathrm{off}}_m$, device $m$ constructs the offline front-end sufficient statistics as $\Sg^{f,\mathrm{off}}_{m}=\sum_{\tau=1}^{n_{\mathrm{off}}}\Tilde{\xb}^f_{m,\tau}(\Tilde{\xb}^f_{m,\tau})^{\top}$ and $\Bf^{f,\mathrm{off}}_{m}=\sum_{\tau=1}^{n_{\mathrm{off}}}\Tilde{\ell}^f_{m,\tau}\Tilde{\xb}^f_{m,\tau}$. These statistics are then used to initialize $\Sg^{f}_{m,t}$ and $\Bf^{f}_{m,t}$ for the subsequent cooperative online learning process. As shown in the following theoretical analysis, this warm-start mechanism leverages offline local inference experiences to improve the regret performance in online learning, thereby accelerating the identification of the optimal DNN partition.


\subsection{Theoretical Performance Guarantee}
In this subsection, we provide theoretical guarantees for the proposed FedLinUCB-DW algorithm by deriving an upper bound on the cumulative regret. We further analyze the computational and communication complexity of FedLinUCB-DW.
\begin{theorem} Under Assumptions~\ref{assumption:noise} and~\ref{assumption:hetero}, if we set the back-end confidence  radius $\beta^b_{t}=\mathcal{O}(\sqrt{\lambda_b}S_b+(\sqrt{1+M\alpha_b}+M\sqrt{2\alpha_b})(R_b\sqrt{d_b\log((1+TL_b^2/\lambda_b)/\delta)}+\sqrt{\lambda_b}S_b))$ and front-end confidence radius $\beta^f_{k,t}=\mathcal{O}(\sqrt{\lambda_f}S_f+(\sqrt{1+M_k\alpha_f}+M_k\sqrt{2\alpha_f})(R_f\sqrt{d_f\log((1+\bar{N}_k(T)L_f^2/\lambda_f)/\delta)}+\sqrt{\lambda_f}S_f+\epsilon L_f \sqrt{\bar{N}_k(T)}))$, with probability at least $1-\delta$, when $\epsilon=1/(M\sqrt{T})$, the cumulative regret $\mathrm{Reg}(T)$ satisfies
\begin{equation}
    \begin{aligned}
       & \mathrm{Reg}(T)\\
       & \leq \widetilde{\mathcal{O}}\Bigg(
        \beta^b_T
        \sqrt{
        d_b(1+M\alpha_b)T
        \log\left(1+\frac{TL_b^2}{\lambda_b d_b}\right)}
        \\
        &+\sum_{k=1}^K
        \beta^f_{k,T}
        \sqrt{
        d_f(1+M_k\alpha_f)N_k(T)
        \log\left(1+\frac{N_k(T)L_f^2}{d_f(\lambda_f+\rho_k)}\right)}
        \Bigg),
    \end{aligned}
\end{equation}
where $M_k=|\mathcal{G}_k|$ denotes the number of devices in group $\mathcal{G}_k$, and $N_k(T)$ denotes the total number of online participations of devices in group $\mathcal{G}_k$ up to round $T$. Moreover, $\bar{N}_k(T)=M_kn_{\mathrm{off}}+N_k(T)$ denotes the total number of offline and online front-end samples associated with group $\mathcal{G}_k$ up to round $T$, and $\rho_k$ denotes the minimum eigenvalue of the group-level offline covariance matrix $\mathbf{\Sigma}^{f,\mathrm{off}}_{\mathcal{G}_k}:=\sum_{n\in\mathcal{G}_k}\mathbf{\Sigma}^{f,\mathrm{off}}_n$.
\end{theorem}

\textbf{Case study.} We next present a simple case study to illustrate the benefits of device cooperation and warm-starting with offline experience in accelerating the identification of optimal partition points. Suppose that each device type contains the same number of devices, i.e., $M_k=\frac{M}{K}$ for all $k\in[K]$. In addition, we assume that each device $m\in\mathcal{M}$ participates in the same number of rounds, denoted by $T_1$. Thus, we have $T=MT_1$ and $N_k(T)=\frac{MT_1}{K}$. Furthermore, the minimum eigenvalue of $\Sg^{f,\mathrm{off}}_{\mathcal{G}_k}$ is set to $\rho$ for each group $\mathcal{G}_k$. Under this setting, the cumulative regret averaged over the number of devices, i.e., $\mathrm{Reg}(T)/M$, satisfies
\begin{equation}\label{eq:regret_bound_2}
    \begin{aligned}
       & \frac{\mathrm{Reg}(T)}{M}\\
     &   \leq \widetilde{\mathcal{O}}\left(
        d_b\sqrt{\frac{T_1}{M}}
        +
        d_f\sqrt{
        T_1\cdot\frac{K}{M}
        \cdot
        \log\left(1+\frac{MT_1L_f^2}{K d_f(\lambda_f+\rho)}\right)}
        \right).
    \end{aligned}
\end{equation}

From Eq.~\eqref{eq:regret_bound_2}, we observe that the upper bound on the average cumulative regret becomes tighter as the number of participating devices $M$ increases and the number of device types $K$ decreases. This indicates that stronger collaboration among more devices, together with higher device homogeneity, i.e., more devices belonging to the same type, enables each device to identify its optimal partition point more efficiently. Furthermore, a larger $\rho$ indicates that the group-level offline covariance matrix is better conditioned, meaning that the offline experiences provide richer coverage over different feature directions. This reduces the uncertainty in front-end parameter estimation and leads to a tighter regret bound, thereby accelerating the identification of the optimal partition point during the online learning phase. Due to space limitations, the detailed proof is deferred to the appendix.

\subsection{Computation and Communication Complexity}

We first analyze the computational complexity of each device. Following the analysis in~\cite{zhang2021autodidactic}, the computational complexity  is $\max\{\mathcal{O}(d_f^3+d_b^3),\mathcal{O}((P+2)(d_f^2+d_b^2))\}$. We measure the communication complexity by the number of communication events between devices and the server~\cite{he2022simple}. The communication complexity of the overall system is thus $\mathcal{O}\left(\sum_{k=1}^{K}\frac{M_kd_f}{\alpha_f}\log\left(1+\frac{N_k(T)L_f^2}{d_f(\lambda_f+\rho_k)}\right)+\frac{Md_b}{\alpha_b}\log\left(1+\frac{TL_b^2}{d_b\lambda_b}\right)\right)$. Due to space limitations, we defer the detailed proof of the communication complexity to the appendix.

Since $d_f$ and $d_b$ are typically small in our formulation, the computational overhead mainly comes from low-dimensional matrix operations and is negligible compared with regular deep inference tasks. Moreover, each communication event primarily involves transmitting the compact sufficient statistics maintained by LinUCB, rather than frequently exchanging raw inference data. Therefore, FedLinUCB-DW incurs low computation and communication overhead, making it well suited for deployment on resource-constrained edge devices.

%% file: content/experiment.tex
In this section, we conduct both the simulation and hardware prototype experiments to validate the superiority of the proposed method over other competitive baselines. Below, we first introduce the comparison baselines and then introduce the simulation and hardware prototype experiments, respectively.
\subsection{Baseline methods}
In this subsection, we introduce the following baseline approaches used in experiments.
\begin{enumerate}
  \item \textbf{LinUCB for ANS:} This baseline deploys ANS~\cite{zhang2021autodidactic} by restricting each UE to independently execute the LinUCB algorithm~\cite{abbasi2011improved} for partition point selection, devoid of inter-device information sharing.
    \item \textbf{FedLinUCB for ANS:} This baseline uses FedLinUCB algorithm~\cite{he2022simple} to enable cooperation among UEs by sharing sufficient statistics used for estimating the overall system parameters, without considering inter-device computational heterogeneity.
    \item \textbf{Warm-Start LinUCB for ANS:} Each UE deploys ANS based on the Warm-Start LinUCB algorithm~\cite{DBLP:journals/kais/OetomoPBR23}, which initializes the local sufficient statistics using offline early-exit inference experience, without cooperating with other UEs.
    \item \textbf{Random:} Each UE selects  DNN partition points randomly to conduct collaborative edge inference.
      \item \textbf{Purely local inference:} All inference computations are conducted on UE locally without the assistance of the edge server. 
\end{enumerate}

\subsection{Simulation experiment}
In this subsection,  we develop a simulation environment with heterogeneous computing devices to demonstrate the superior performance of the proposed FedLinUCB-DW method compared with baseline approaches.

\textbf{Heterogeneous edge devices.} Our experimental setup comprises $N=25$ UEs, categorized into three performance tiers to simulate diverse edge computational capabilities. The High-Performance Group (30\%) represents high-end edge GPUs (NVIDIA Jetson Orin Nano, $20.0$ TOPS); the Mid-Performance Group (40\%) corresponds to mid-range accelerators (NVIDIA Jetson Xavier NX, $10.5$ TOPS)~\cite{jeon2021run}; and the Low-Performance Group (30\%) captures resource-constrained devices (Raspberry Pi 5, $75.0$ GFLOPS), utilizing NEON-optimized CPU execution~\cite{DBLP:journals/corr/abs-2108-09457}.


\textbf{Server and network configuration.} The edge-side is supported by a powerful back-end server (e.g., NVIDIA H100/A100 class) with a peak AI compute capacity of $312$ TFLOPS. Wireless communication is configured  with the available bandwidth $B=10$  Mbps~\cite{huang2023adversarial}.

\textbf{Backbone model and action space.} We utilize the VGG-16, ResNet-50 and ViT-16 architectures as the backbone models, which are with the corresponding action space $\mathcal{P}$ consisting of $23, 20, 15$ distinct partition points, respectively. Each partition point $p \in \mathcal{P}$ is characterized by a three-dimensional contextual feature vector $ [\text{MAC}^p_{front}, \text{Data}^p_{trans}, \text{MAC}^p_{back}]^T \in \mathbb{R}^3$, where $\text{MAC}^p_{front}$ and $\text{MAC}^p_{back}$ denote the MAC units of the front-end and back-end layers, respectively. $\text{Data}^p_{trans}$ denotes the size of the intermediate result to be transmitted at the partition point. For CNN-based models, we use the online analytical tool NetscopeAnalyzer  to
 calculate the computational workload of each layer (in terms
 of the number of MAC units) and
 the intermediate data size between the layers~\cite{netscope_analyzer}. For ViT-16 model, we employ the fvcore library to extract the cumulative MAC operations for each candidate partition.

\textbf{Learning protocol.} We consider an online learning protocol over $T=2500$ rounds. In each round $t \in [T]$, one UE $m_t \in \mathcal{M}$ becomes active and performs collaborative edge inference with the edge server. The parameter
vector of $m_t$ is defined as $\boldsymbol{\theta}_{m_t} := [1/f_{m_t}, 1/B, 1/f_{S}]^\top$, representing the reciprocal of the computational and communication capacities. $f_{m_t}$ and $f_{S}$ denote the computational speed of device $m_t$ and the edge server, respectively. For a partition point $p_t$ selected by UE $m_t$, the realized latency is decomposed into front-end and back-end components. Specifically, the front-end latency $\ell^f_t$ and back-end latency $\ell^b_t$ are modeled as:
\begin{align}
\ell^f_t &= \frac{MAC^{p_t}_{front}}{f_{m_t}} + \eta^f_t, \\
\ell^b_t &= \frac{\text{Data}^{p_t}_{trans}}{B} + \frac{MAC^{p_t}_{back}}{f_{S}} + \eta^b_t,
\end{align}
where $\eta^f_t, \eta^b_t \sim \mathcal{N}(0, \sigma^2)$ are independent Gaussian noise terms representing measurement errors and network jitter in practical deployments. 

To emulate the offline early-exit inference experience adopted in the proposed FedLinUCB-DW and Warm-Start LinUCB methods, each UE randomly selects partition points to perform local inference for five rounds prior to the online learning phase. The resulting latency feedback is then collected to construct local statistics, which enable more accurate and efficient estimation of front-end parameters during the subsequent online phase.


\textbf{Evaluation Metrics.} To evaluate the performance of CANS, the following metrics are used: 
\begin{enumerate}
 \item  \textbf{Cumulative regret.} For each round's active UE, the instantaneous regret is defined as the difference between the latency of the selected partition point and the  latency of the optimal partition point. The cumulative regret is the sum of  instantaneous regrets over completed rounds.
    \item  \textbf{Average latency.} For each  round, we first record the end-to-end inference latency of the UE active in that round. The final reported metric is the round-averaged latency, obtained by averaging these per-round latency values over all completed rounds.
    \item  \textbf{Latency estimation error.} For each round's active UE, the estimated latency is computed using learned parameters and the feature vector of the selected partition point, while the true latency is derived from true parameters. We report the average discrepancy between the estimated and true latency across  completed  rounds.
\end{enumerate}

\begin{figure*}[htbp]
    \captionsetup[subfloat]{font=footnotesize}	
    \centering
    \subfloat[Cumulative Regret for VGG]{
        \includegraphics[width = 0.32\textwidth]{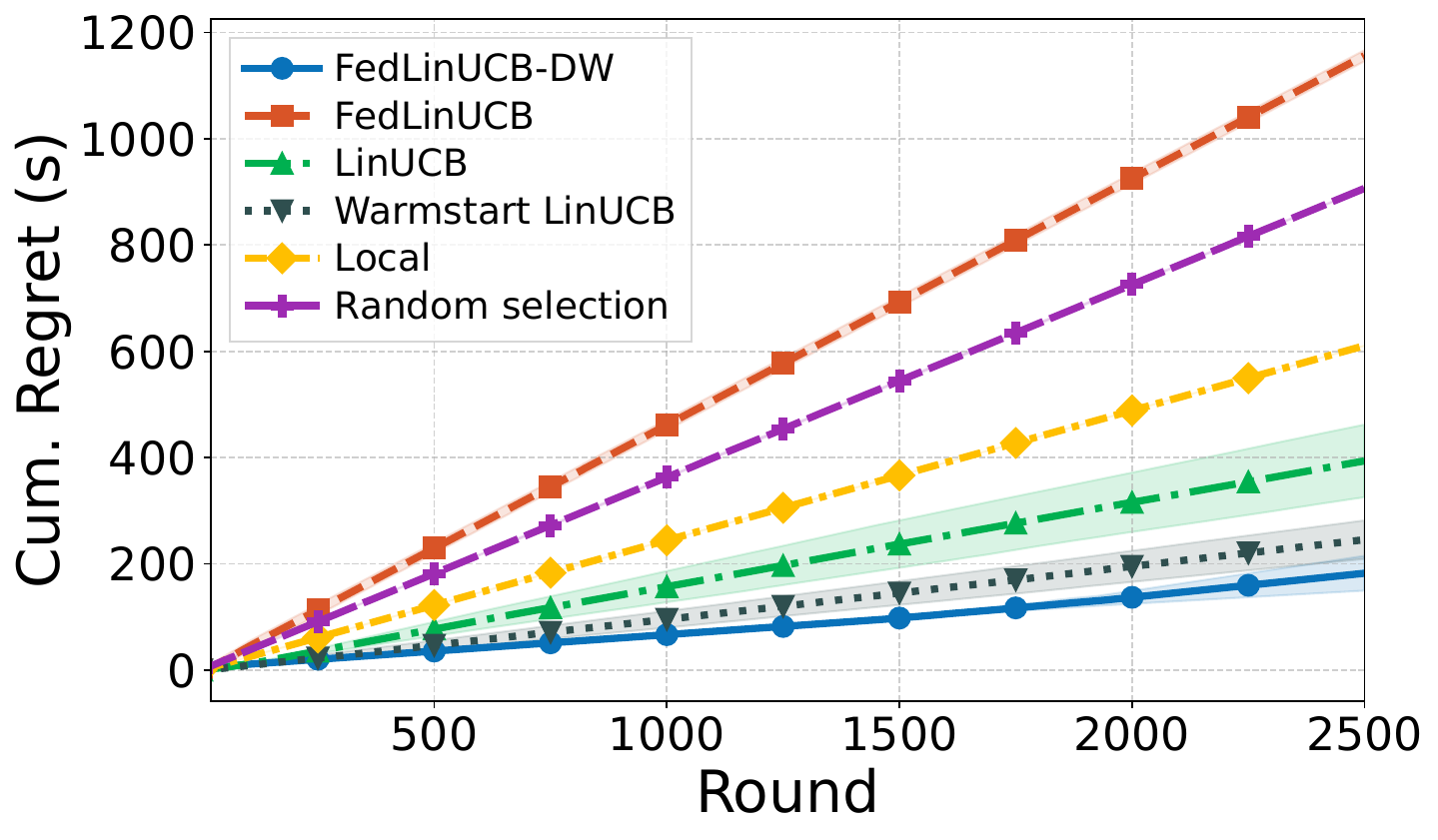}
        \label{fig:regret_vgg} 
    }
    \subfloat[Cumulative Regret for ResNet]{
        \includegraphics[width = 0.32\textwidth]{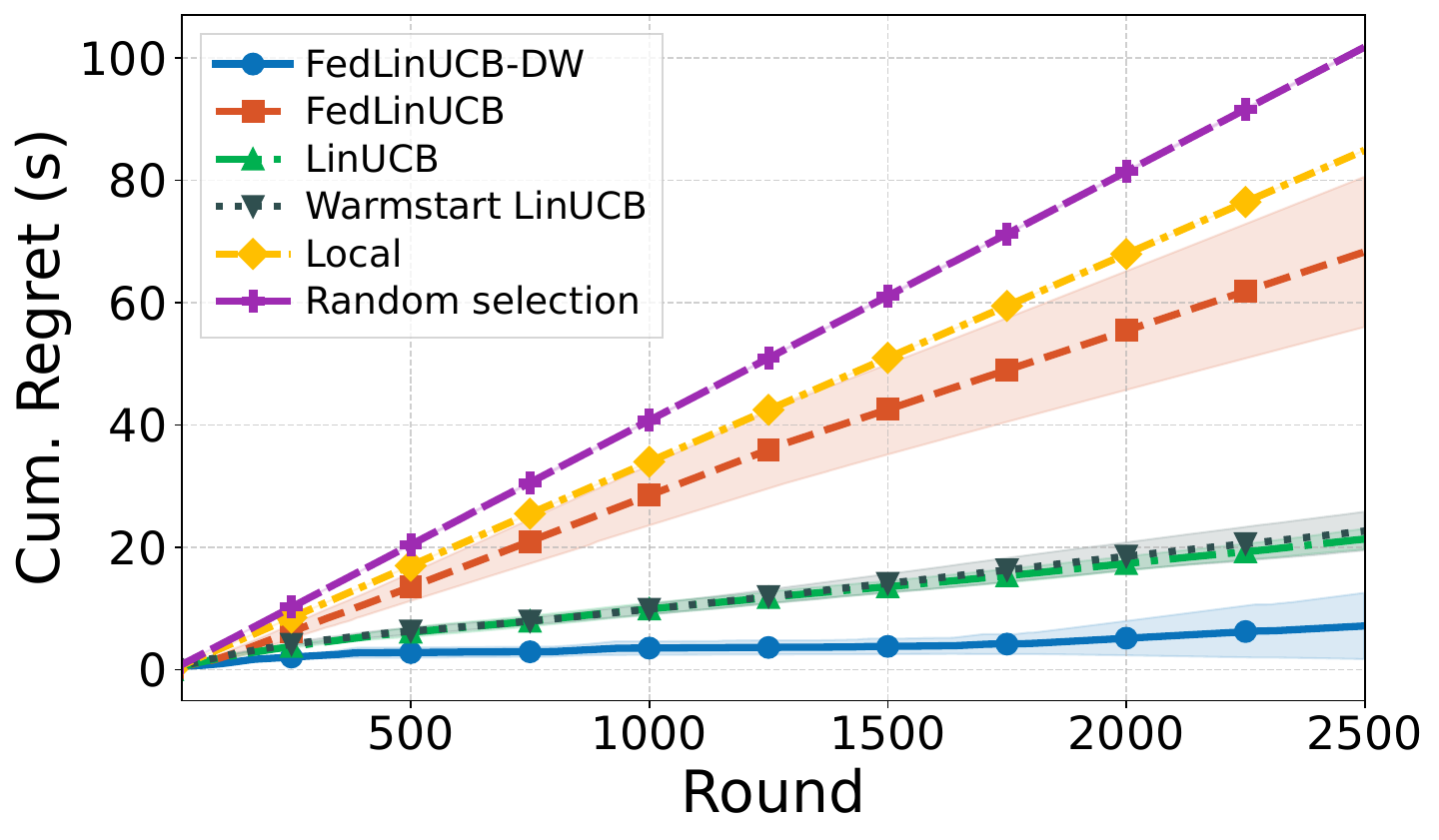}
        \label{fig:regret_resnet} 
    }
    \subfloat[Cumulative Regret for ViT]{
        \includegraphics[width = 0.32\textwidth]{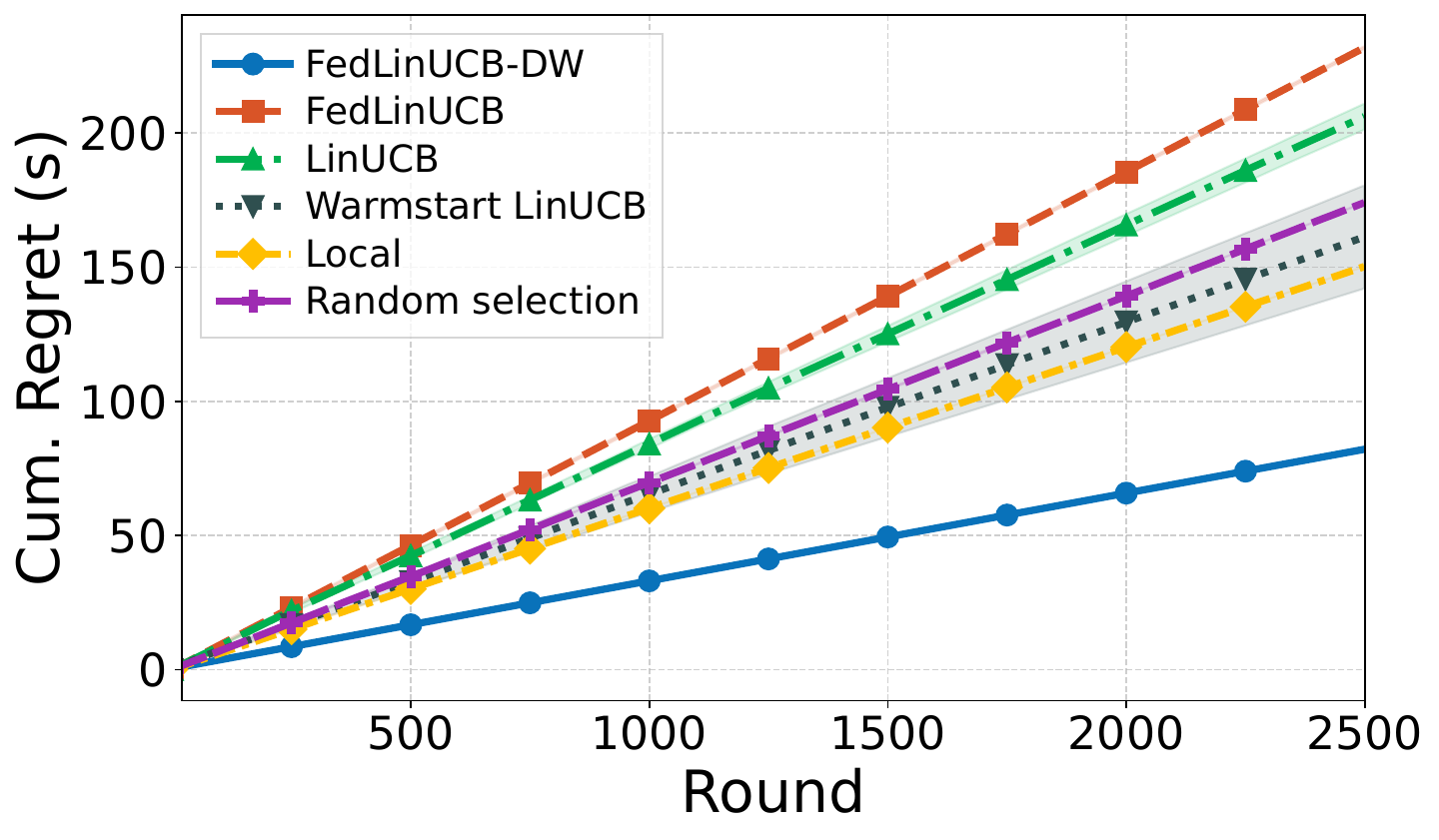}
        \label{fig:regret_vit} 
    }
    \caption{Cumulative Regret for three backbone models.}
    \label{fig:regret} 
\end{figure*}

\begin{figure*}[htbp]
\captionsetup[subfloat]{font=footnotesize}	
\centering
\subfloat[Average latency  for VGG]{\includegraphics[width = 0.33\textwidth]{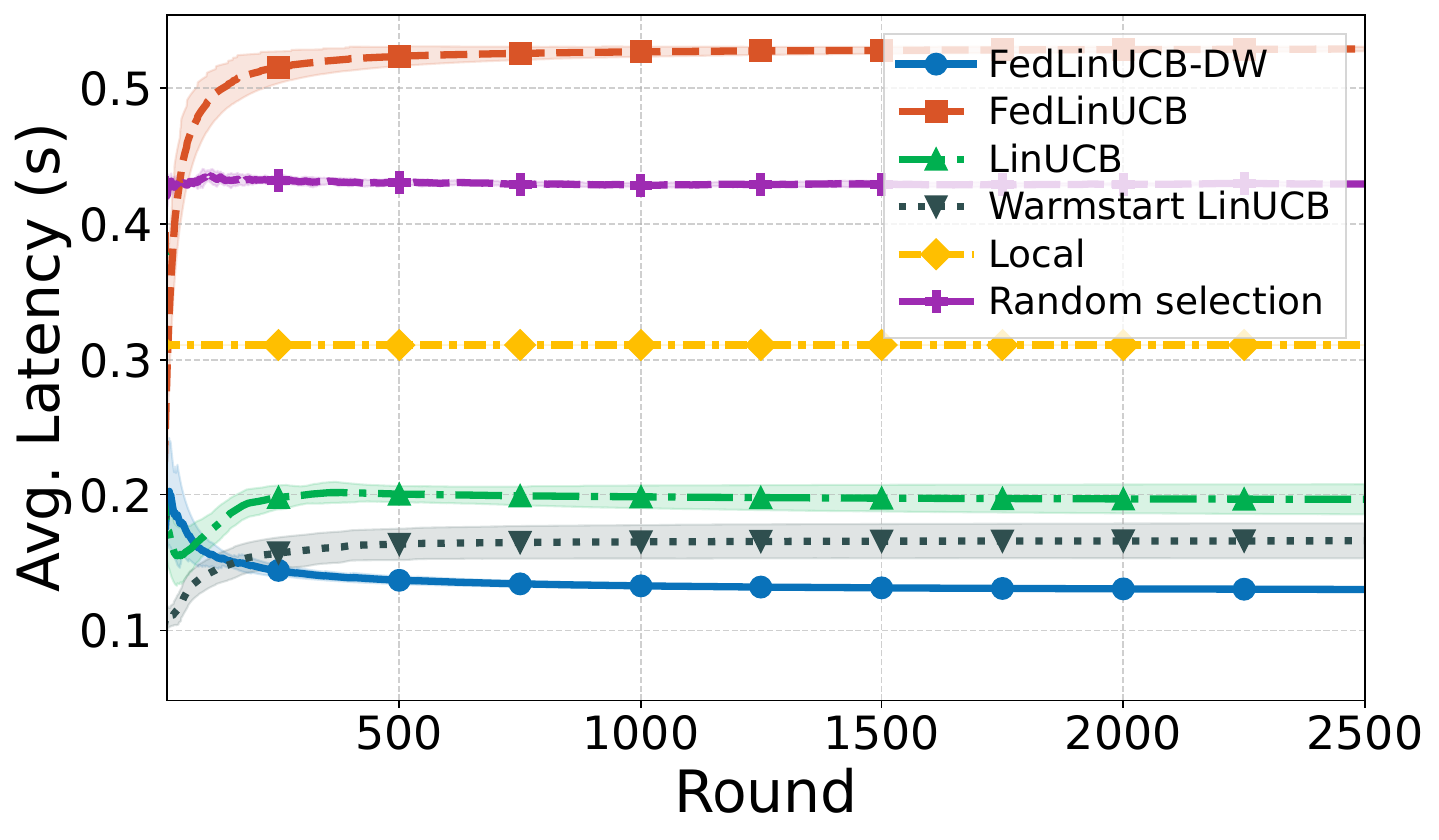}}
   	\subfloat[Average latency  for ResNet]{\includegraphics[width = 0.33\textwidth]{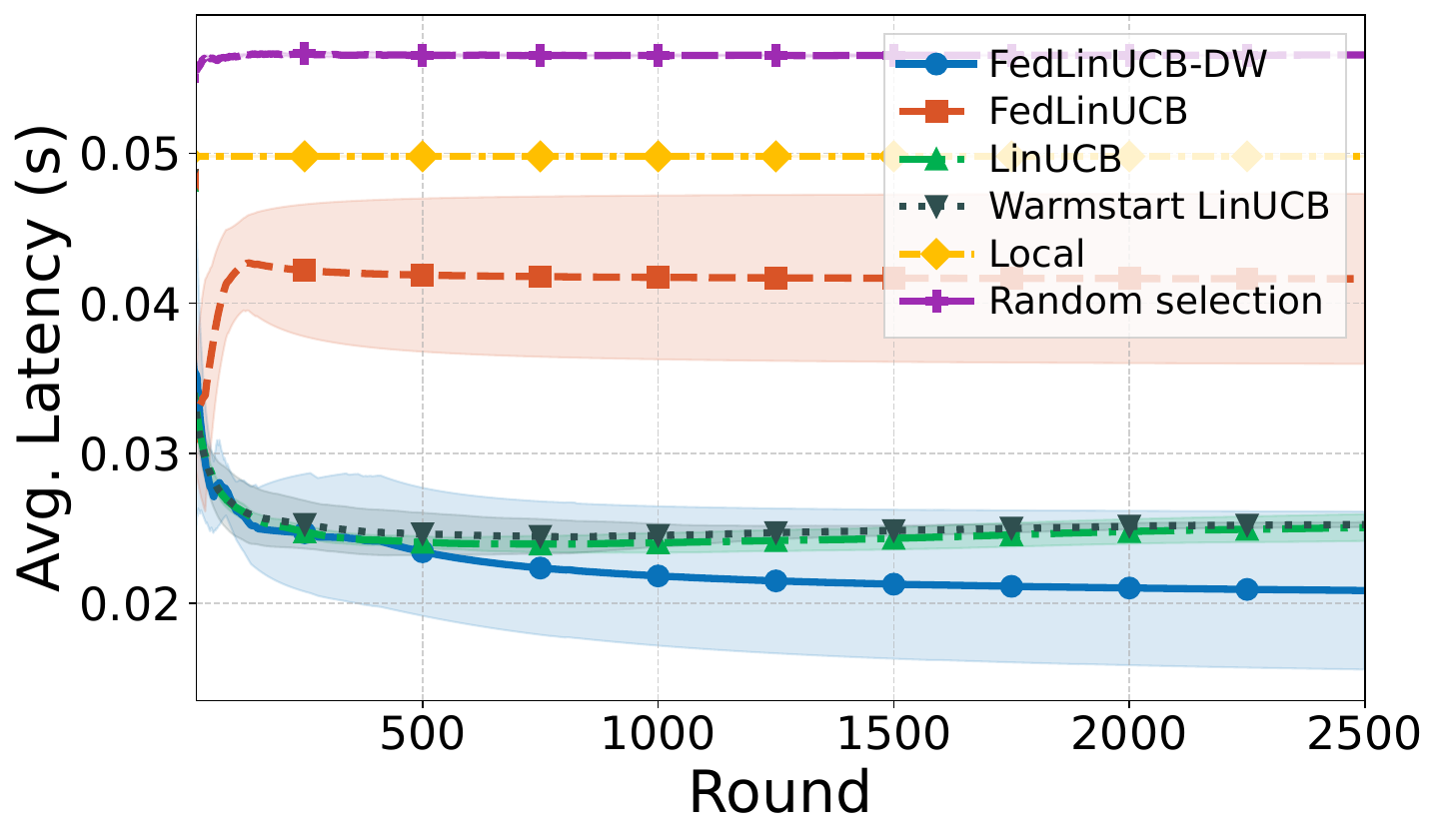}}
	\subfloat[Average latency  for ViT]{\includegraphics[width = 0.33\textwidth]{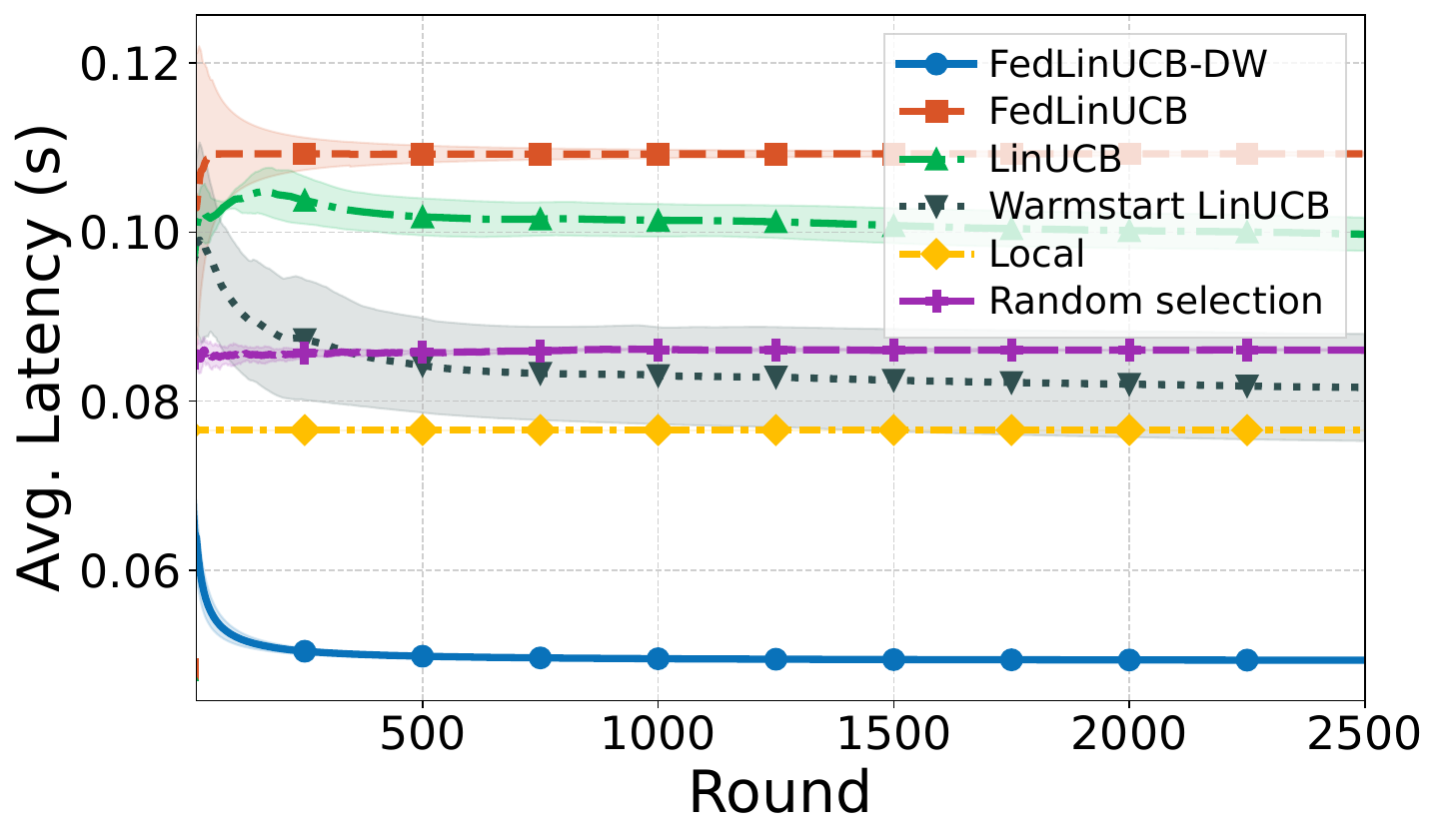}}
\caption{Average latency for three backbone models.}
\label{fig:latency}
\end{figure*}

\begin{figure*}[htbp]
\captionsetup[subfloat]{font=footnotesize}	
\centering
\subfloat[Estimation error  for VGG]{\includegraphics[width = 0.33\textwidth]{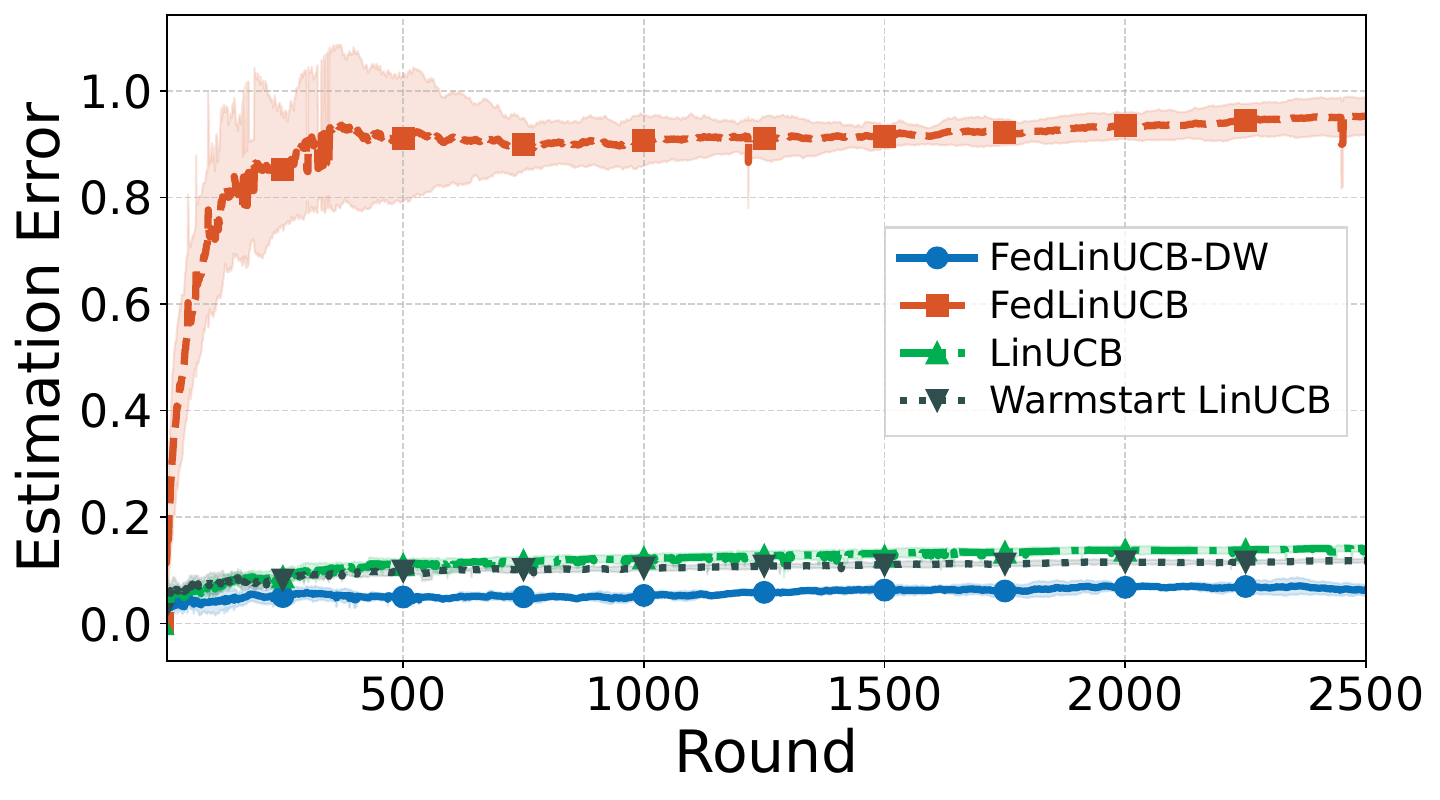}}
   	\subfloat[Estimation error  for ResNet]{\includegraphics[width = 0.33\textwidth]{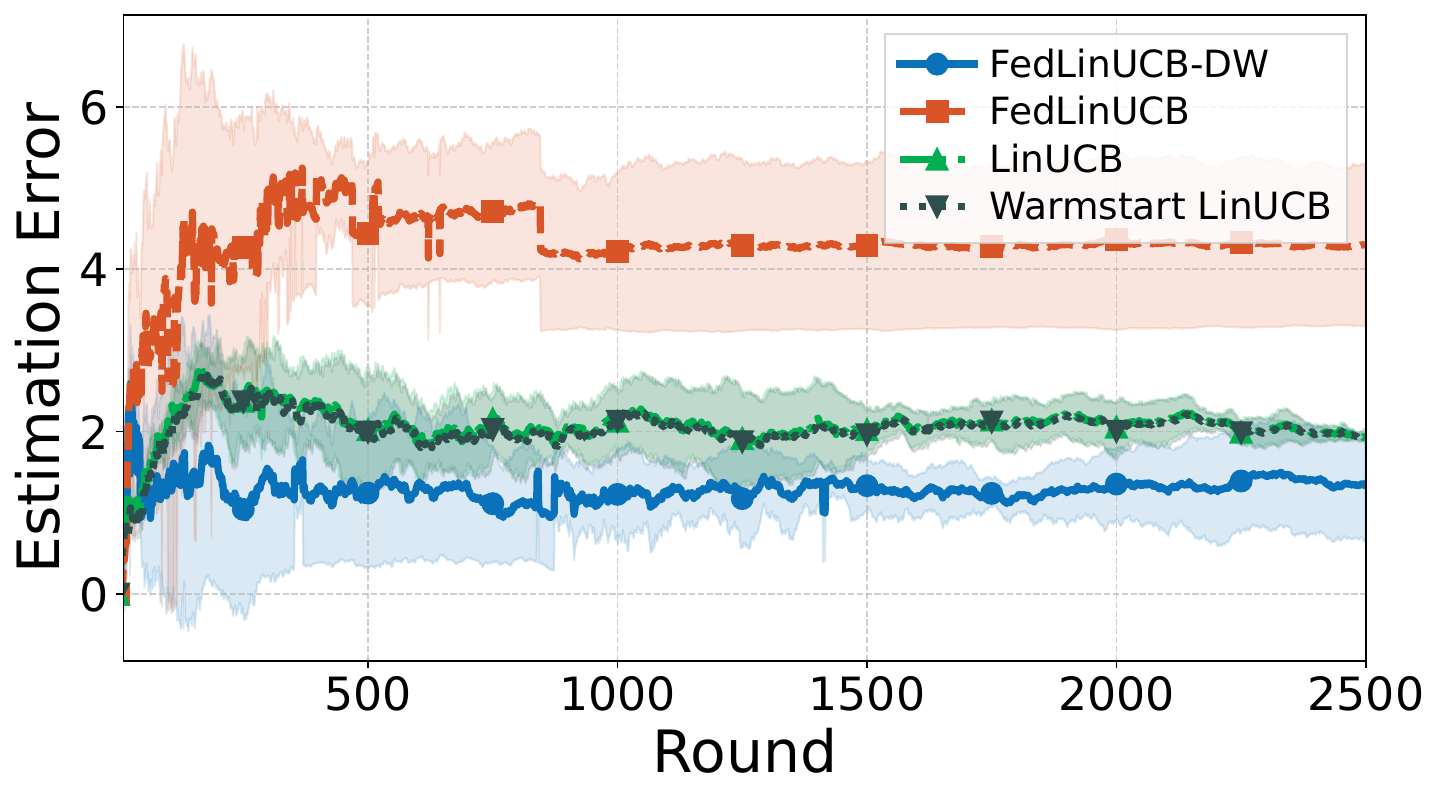} \label{fig:estimation_resnet}}
	\subfloat[Estimation error  for ViT]{\includegraphics[width = 0.33\textwidth]{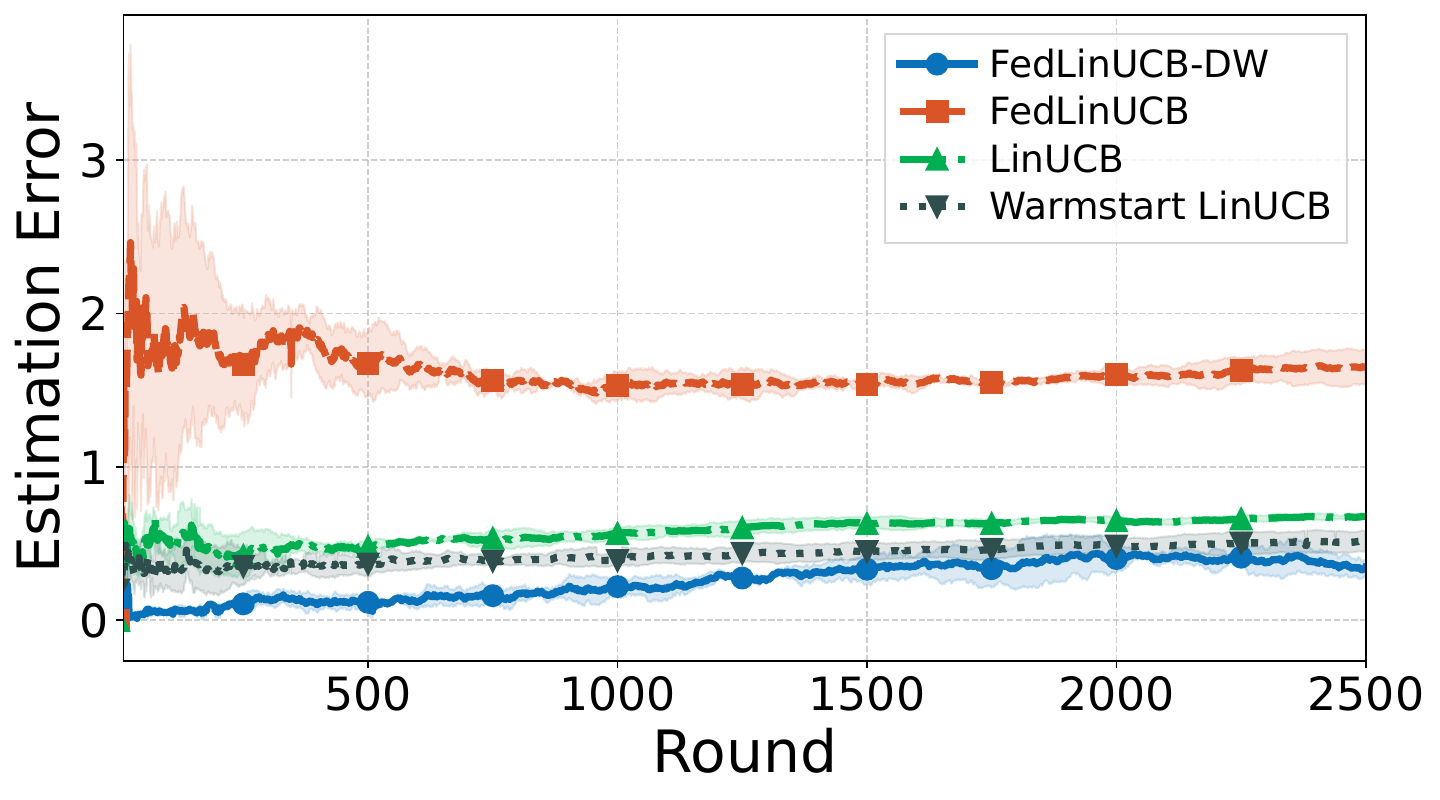}}
\caption{Estimation error for three backbone models.}
\label{fig:estimation}
\end{figure*}

\textbf{Experimental results.} We first compare the performance of our proposed FedLinUCB-DW with baseline methods in terms of the cumulative regret  (shown in Fig.~\ref{fig:regret}), the average latency (shown in Fig.~\ref{fig:latency}), and the latency estimation error (shown in Fig.~\ref{fig:estimation}) for three different backbone models.  All reported simulation results are averaged over three independent runs with different random seeds, and are presented as mean values with standard deviation error bars.


Regarding the cumulative regret results in Fig.~\ref{fig:regret}, we emphasize that this metric is defined as the sum of instantaneous regrets incurred by all active agents throughout the completed  rounds. Specifically, the instantaneous regret at round $t$ is defined as the difference between the latency of the partition point $p_t$ selected by the active UE $m_t$, and the latency of the optimal partition point $p_{m_t}^*$. The optimal partition point is determined  by leveraging the ground-truth system parameters at each round. Specifically, we evaluate all candidate partitions and select the one that yields the minimum theoretical latency as the benchmark. For all three different models, we notice that our proposed FedLinUCB-DW algorithm achieves the smallest cumulative regret and exhibits a sublinear regret in the number of rounds. Below we provide the reason why the baseline methods underperform our proposed FedLinUCB-DW algorithm. \textbf{Pure local inference}: Although pure local inference eliminates the need to upload intermediate inference results to the edge server via wireless transmission, it fails to exploit the advantages of a more powerful edge server. In Figure~\ref{fig:regret_resnet}, when employing ResNet-50 as the backbone model, this approach performs worse than most other algorithms, with the exception of FedLinUCB and random selection. \textbf{FedLinUCB:} We observe that FedLinUCB, which aggregates local statistics of all devices to estimate the global parameter $\boldsymbol{\theta}$, underperforms compared to other baseline methods when using VGG-16 and ViT-16 in Figure~\ref{fig:regret_vgg} and Figure~\ref{fig:regret_vit}. This is primarily because the algorithm neglects the computational heterogeneity across mobile devices. By indiscriminately combining latency feedback from heterogeneous devices, it may introduce misleading information into the estimation process, ultimately resulting in degraded performance. \textbf{LinUCB:} This method achieves lower regret than both Random selection and pure local inference in VGG-16 and ResNet-50, owing to its ability to learn the optimal partition point through parameter estimation based on latency feedback. It also outperforms FedLinUCB under the considered computational heterogeneity setting, as each UE relies solely on its own local statistics to estimate device-specific parameters. However, LinUCB performs worse than our proposed FedLinUCB-DW, as it neither facilitates cooperative online learning among devices of the same type nor leverages prior offline inference experiences. \textbf{Warm-Start LinUCB:} This method achieves lower regret than LinUCB in VGG-16 and ViT-16, and demonstrates comparable performance in ResNet-50, as it leverages offline inference experiences to initialize local statistics for improved parameter estimation. This observation highlights the effectiveness of offline warm-start in enhancing online learning. Nevertheless, it is outperformed by our proposed FedLinUCB-DW, as it relies solely on local information and does not incorporate cooperative learning through the sharing of latency feedback among  UEs of the same device type.

As shown in Fig. \ref{fig:latency}, we observe that our proposed FedLinUCB-DW achieves the lowest average latency (round-average end-to-end inference latency) among all baseline methods. The latency decreases rapidly during the initial stage, which can be attributed to its effective initialization and its ability to aggregate a larger number of samples from UEs of the same type. This enables more efficient and accurate learning of system parameters. Besides, FedLinUCB exhibits the highest average latency among other learning-based methods. Notably, the latency even exhibits an adverse upward trend before converging to a suboptimal partition point, resulting in significantly higher end-to-end latency. This phenomenon can be attributed to a similar underlying issue: FedLinUCB aggregates local statistics for estimating global parameters across all UEs with heterogeneous computational capabilities. Such indiscriminate aggregation introduces biased and potentially misleading information into the parameter estimation process. Consequently, the learned model fails to accurately capture device-specific characteristics, leading to suboptimal decision-making and degraded overall system performance.
   \begin{figure*}[htbp]
\captionsetup[subfloat]{font=footnotesize}	
\centering
\subfloat[Average latency for VGG]{\includegraphics[width = 0.33\textwidth]{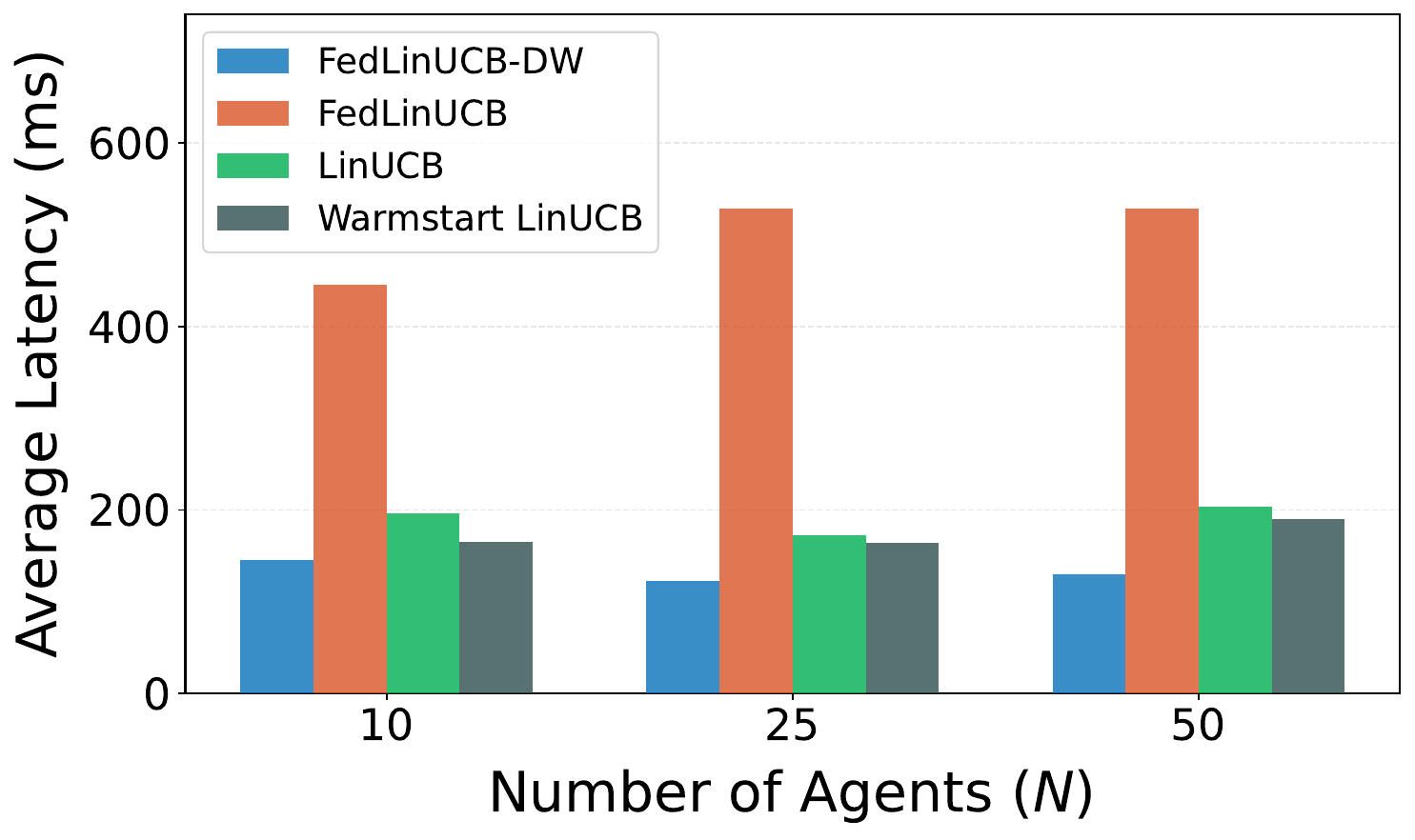}}
   	\subfloat[Average latency for ResNet]{\includegraphics[width = 0.33\textwidth]{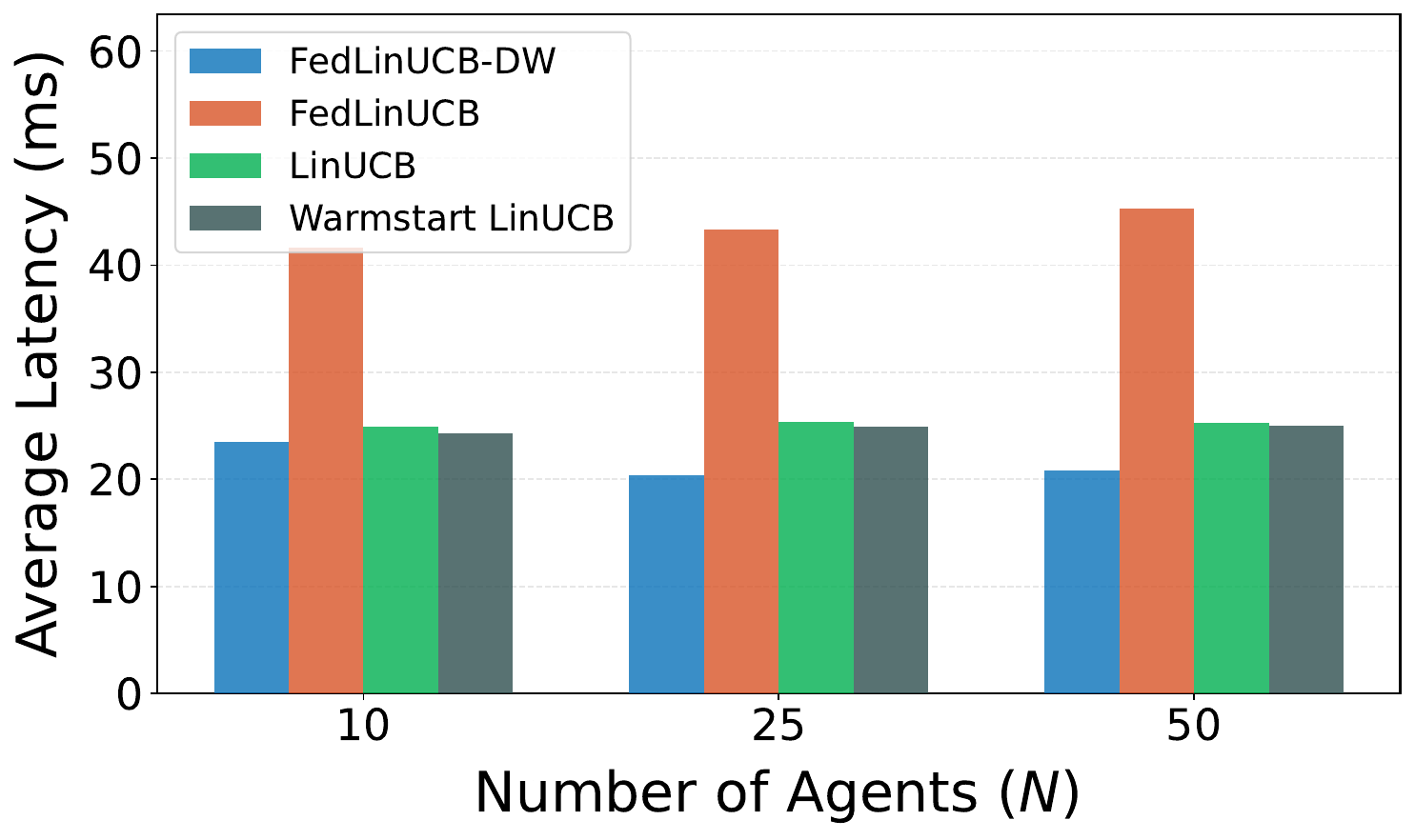}}
	\subfloat[Average latency for ViT]{\includegraphics[width = 0.33\textwidth]{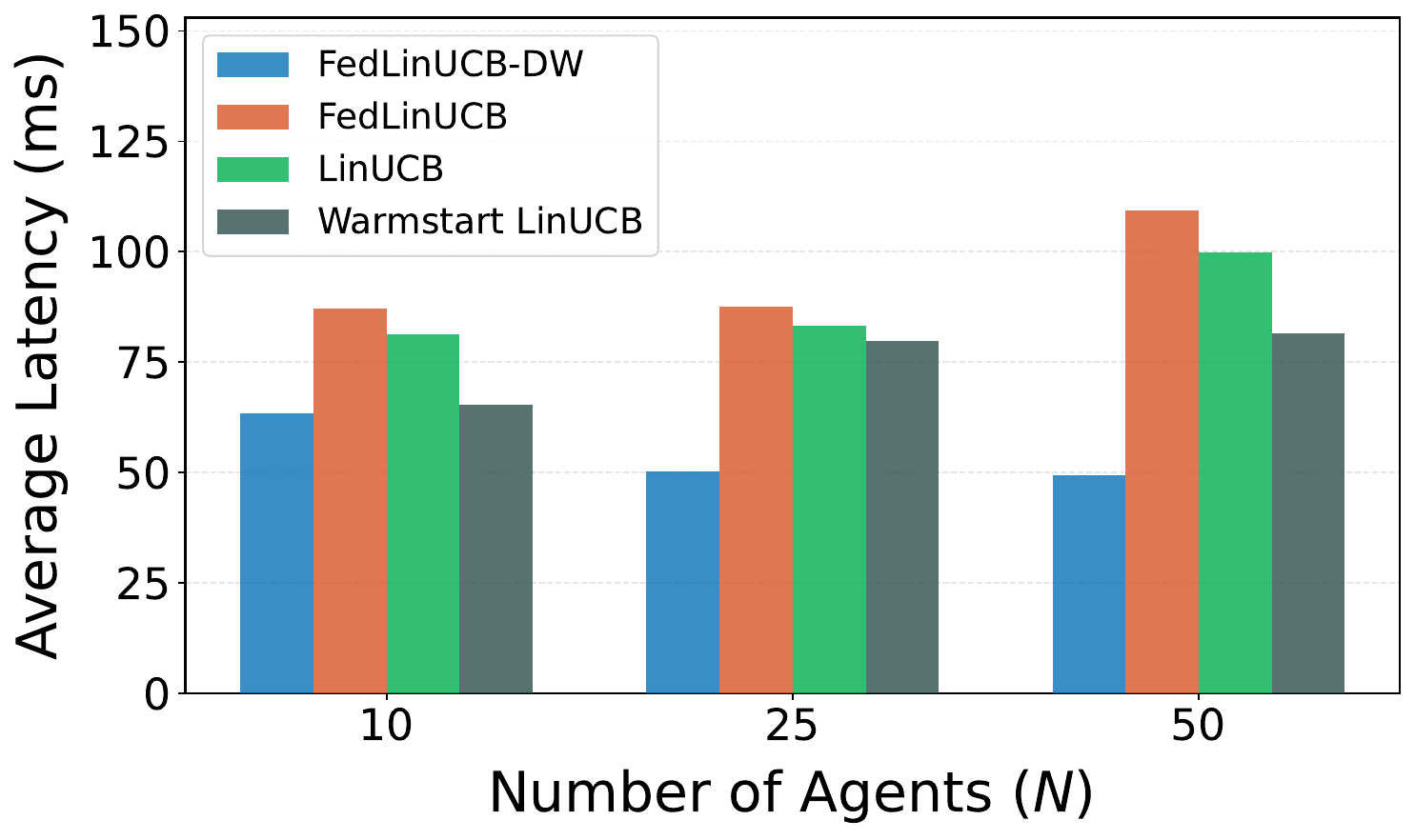}}
\caption{Average latency of different numbers of agents for three backbone models.}
\label{fig:agent}
\end{figure*}
\begin{figure}[!t]
    \centering
    \includegraphics[width=3.5in]{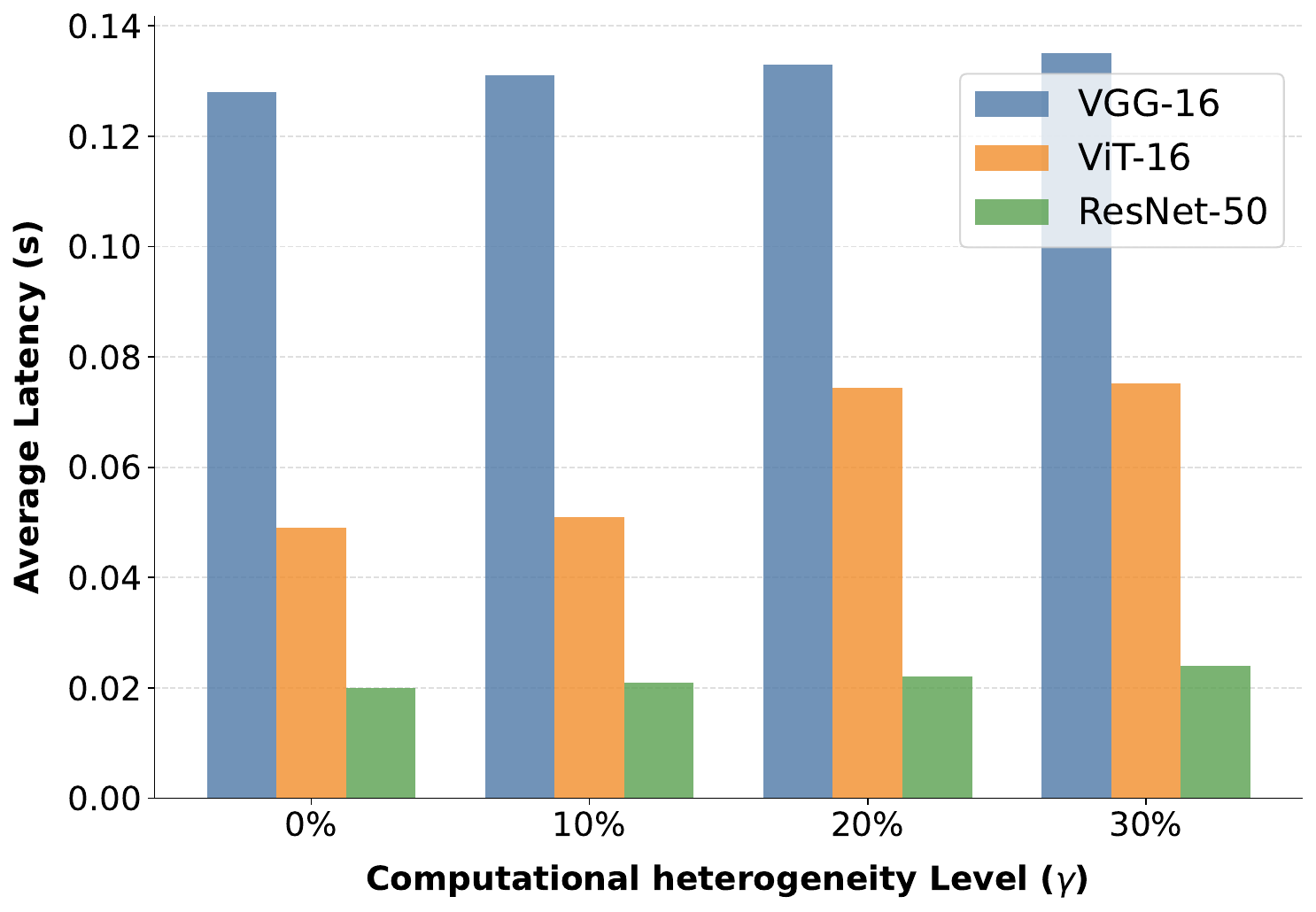}
    \caption{Impact of computational  heterogeneity  within each group on the average latency across different backbone models using FedLinUCB-DW.}
    \label{fig:noise_latency}
\end{figure}
Regarding the latency estimation error illustrated in Fig. \ref{fig:estimation}, we observe that  FedLinUCB-DW still achieves the best performance, indicating that it provides more accurate estimation of the unknown system parameters. In addition, we note that in Figure~\ref{fig:estimation_resnet}, when ResNet-50 is adopted as the backbone model, all four online learning methods exhibit significant fluctuations and larger variance in estimation error. This suggests that parameter estimation becomes more challenging when ResNet-50 is used. Similarly, FedLinUCB continues to demonstrate the largest estimation error, further confirming that aggregating local statistics from heterogeneous devices can indeed mislead the system parameter estimation process.

Fig. \ref{fig:agent} illustrates the impact of agent scaling on the average end-to-end latency. In this evaluation, the number of participating UEs was varied ($N = 10, 25, 50$) across all backbone models to assess the scalability of each approach. The results demonstrate that FedLinUCB-DW consistently achieves the lowest average latency across all configurations, indicating its robustness to varying network scales. In contrast, the performance of FedLinUCB degrades markedly as the number of agents reaches $N=50$. This observation suggests that a larger population of hardware-diverse devices amplifies computational heterogeneity, which in turn exacerbates the negative effects of aggregating shared information. Specifically, aggregating local statistics across heterogeneous devices can introduce biased or misleading signals into the estimation of the overall system parameters, thereby impairing the effectiveness of cooperative online learning.

We further investigate the performance of FedLinUCB-DW under computational heterogeneity within each group. Specifically, we model stochastic hardware fluctuations (e.g., thermal throttling) that cause performance variations among UEs of the same type. As mentioned earlier, $N=25$ UEs are categorized into three groups $\mathcal{G}=\{G_1, G_2, G_3\}$, the instantaneous compute capacity $C_i$ for UE  $i$ is formulated as $C_i = g_i \times \alpha$, where $g_i \in \mathcal{G}$ and $\alpha \sim \mathcal{U}(1 - \gamma, 1 + \gamma)$. Here, $\gamma \in \{0\%, 10\%, 20\%, 30\%\}$ denotes the noise level, controlling the deviation within the same tier. As illustrated in Fig.~\ref{fig:noise_latency}, increasing $\gamma$ leads to a general rise in average latency across all models. Notably, VGG-16 and ViT-16 exhibit higher sensitivity to these fluctuations, whereas ResNet-50 maintains a more resilient latency profile. Throughout these volatile scenarios, FedLinUCB-DW consistently mitigates performance degradation, demonstrating superior adaptability.


\subsection{Hardware prototype experiment}

In this subsection, we conduct hardware prototype experiments to evaluate the practicality of CANS in real-world edge-inference settings. The prototype experiments focus on validating the deployment feasibility of CANS and its ability to reduce end-to-end inference latency on real edge devices, while the large-scale benefits of device grouping are evaluated in the preceding simulation study.

\textbf{Testbed setup.} We build a hardware testbed to validate the performance of our proposed CANS system with FedLinUCB-DW algorithm. We employ two distinct NVIDIA Jetson modules as the mobile computing nodes to evaluate performance across different hardware generations. Specifically, we use the NVIDIA Jetson Xavier NX, which features a 384-core NVIDIA Volta GPU and 8 GB of 128-bit LPDDR4x unified memory. Additionally, we utilize the NVIDIA Jetson Orin Nano, equipped with a 1024-core NVIDIA Ampere GPU and 8 GB of 128-bit LPDDR5 memory.  The workstation is equipped with an Intel Xeon Platinum 8276L CPU @ 2.20 GHz, 376 GB RAM, and four NVIDIA A100-SXM4-40GB GPUs. The mobile devices and edge server are wirelessly connected by  Wi-Fi, and we use Linux traffic control to set the same wireless transmission conditions.

 Fig.~\ref{fig:hardware_testbed} shows the hardware prototype testbed used in our experiments. The workstation serves as the edge server, while the Jetson Xavier NX and Jetson Orin Nano act as two heterogeneous mobile edge devices. During each inference task, the active Jetson device performs the front-end computation locally, transmits the intermediate output to the workstation, and receives the final inference result together with the back-end latency feedback. This setup emulates a practical multiuser collaborative edge-inference scenario, where multiple resource-constrained devices independently interact with a shared edge server.

     \begin{figure}[htbp]
    \centering
    \includegraphics[width=0.52\columnwidth]{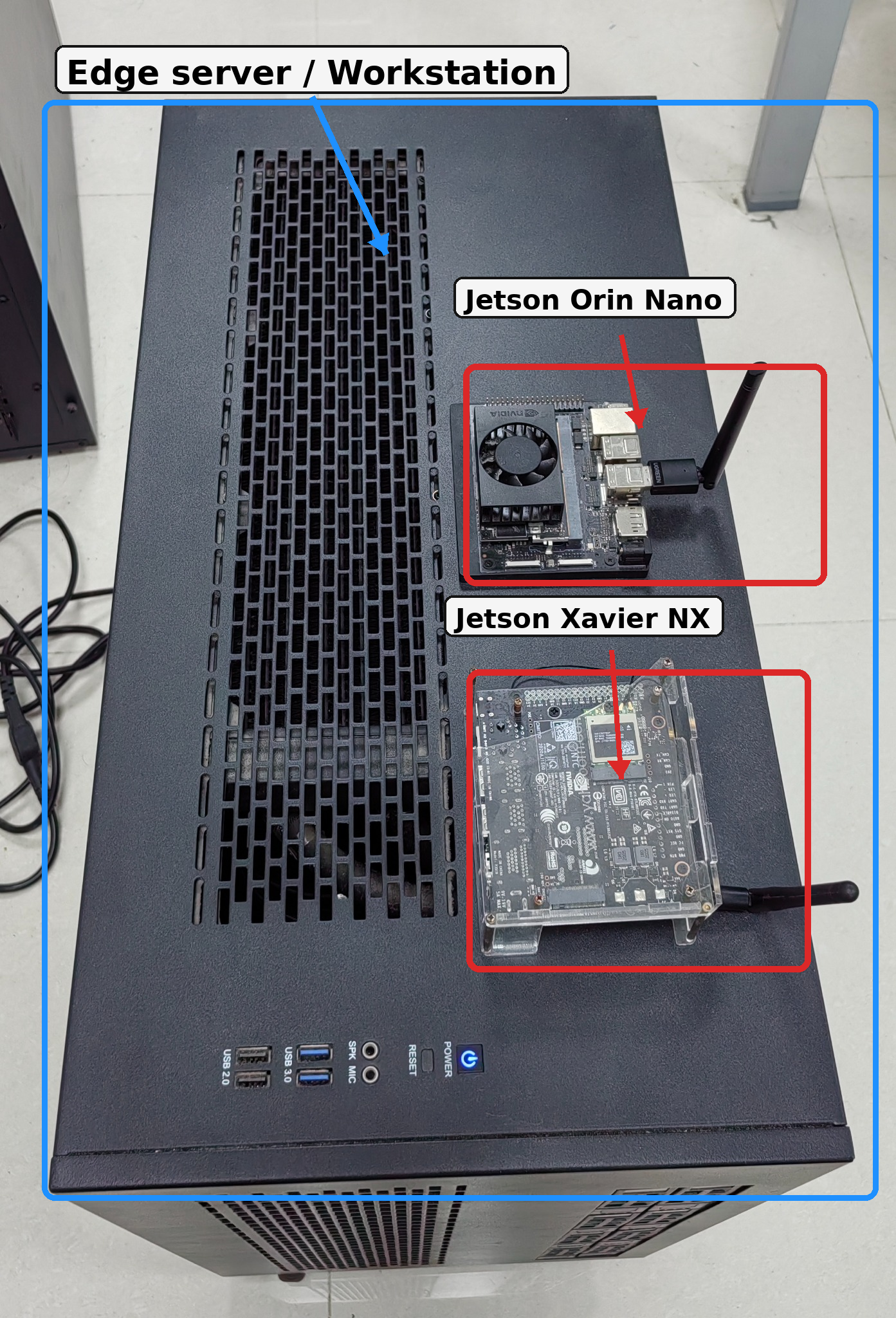}
    \caption{Overall hardware prototype testbed of CANS. }
    \label{fig:hardware_testbed}
    \vspace{-0.15in}
\end{figure}


\textbf{Backbone model and action space.}
Similar to the simulation experiments, we adopt VGG-16, ResNet-50, and ViT-16 as backbone models, with 23, 20, and 15 partition points, respectively.
The key difference is that, in this hardware prototype experiment, we employ a more fine-grained contextual feature representation for each partition point.
Specifically, for CNN-based models, we consider both the number of operations and their corresponding MAC units for convolutional layers, linear layers, and activation functions at different partition points.
For ViT-16, we follow a similar approach by modeling the number of operations and MAC costs for attention layers, linear layers, and activation functions across different partition points. This operation-specific categorization is superior because it explicitly captures the hardware-level distinction between compute-bound operations and memory-bound operations. By encoding these distinct characteristics into the feature vector, the framework can accurately predict the actual execution latency on different devices, rather than relying on misleadingly uniform total MAC counts.

\begin{figure*}[htbp]
\captionsetup[subfloat]{font=footnotesize}	
\centering
\subfloat[Average latency for VGG]{\includegraphics[width = 0.33\textwidth]{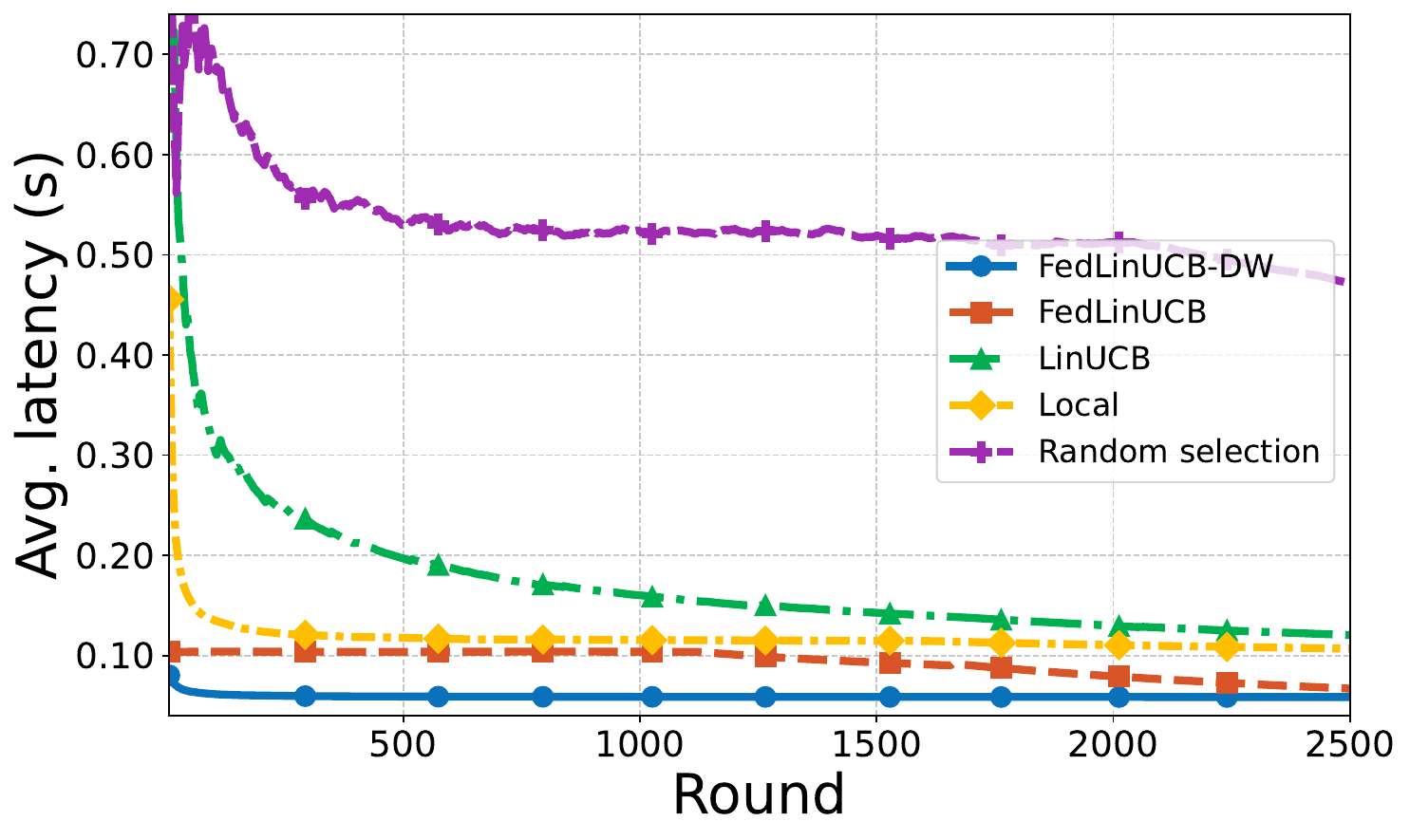}}
   	\subfloat[Average latency for ResNet]{\includegraphics[width = 0.33\textwidth]{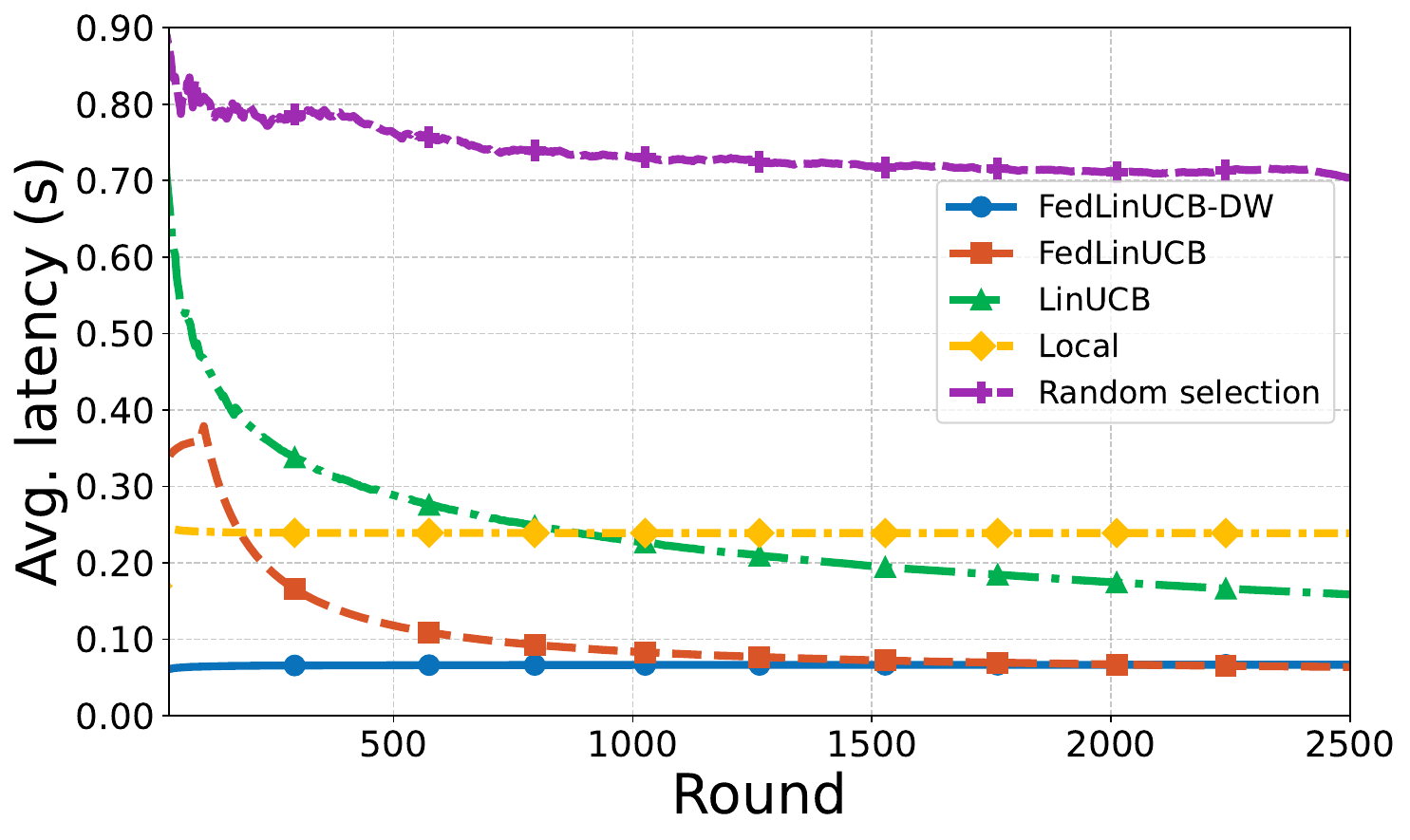}}
	\subfloat[Average latency for ViT]{\includegraphics[width = 0.33\textwidth]{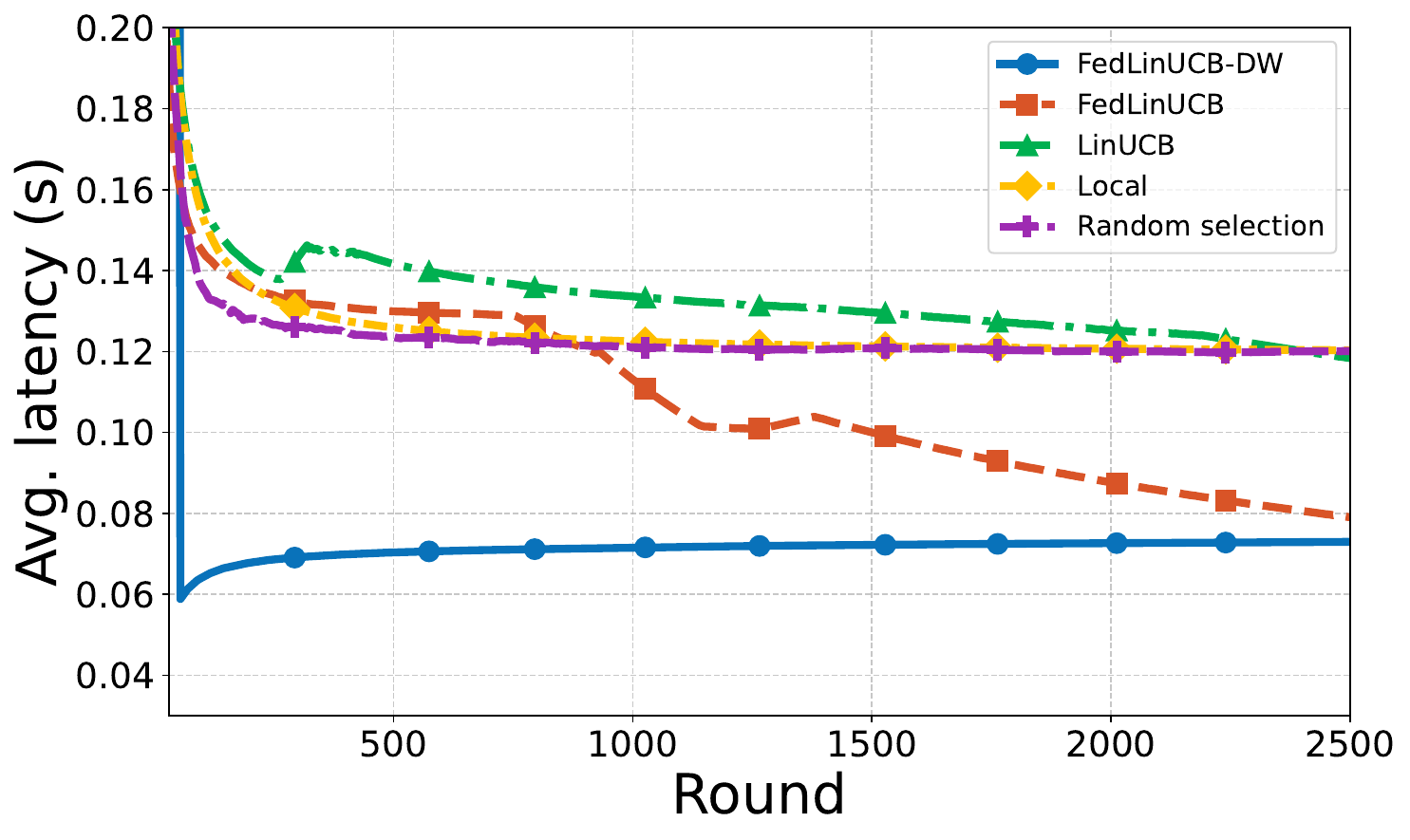}}
\caption{Average latency for different baselines on Xavier NX.}
\label{fig:latency_bs_nx}
\end{figure*}

    \begin{figure*}[htbp]
\captionsetup[subfloat]{font=footnotesize}	
\centering
\subfloat[Average latency for VGG]{\includegraphics[width = 0.33\textwidth]{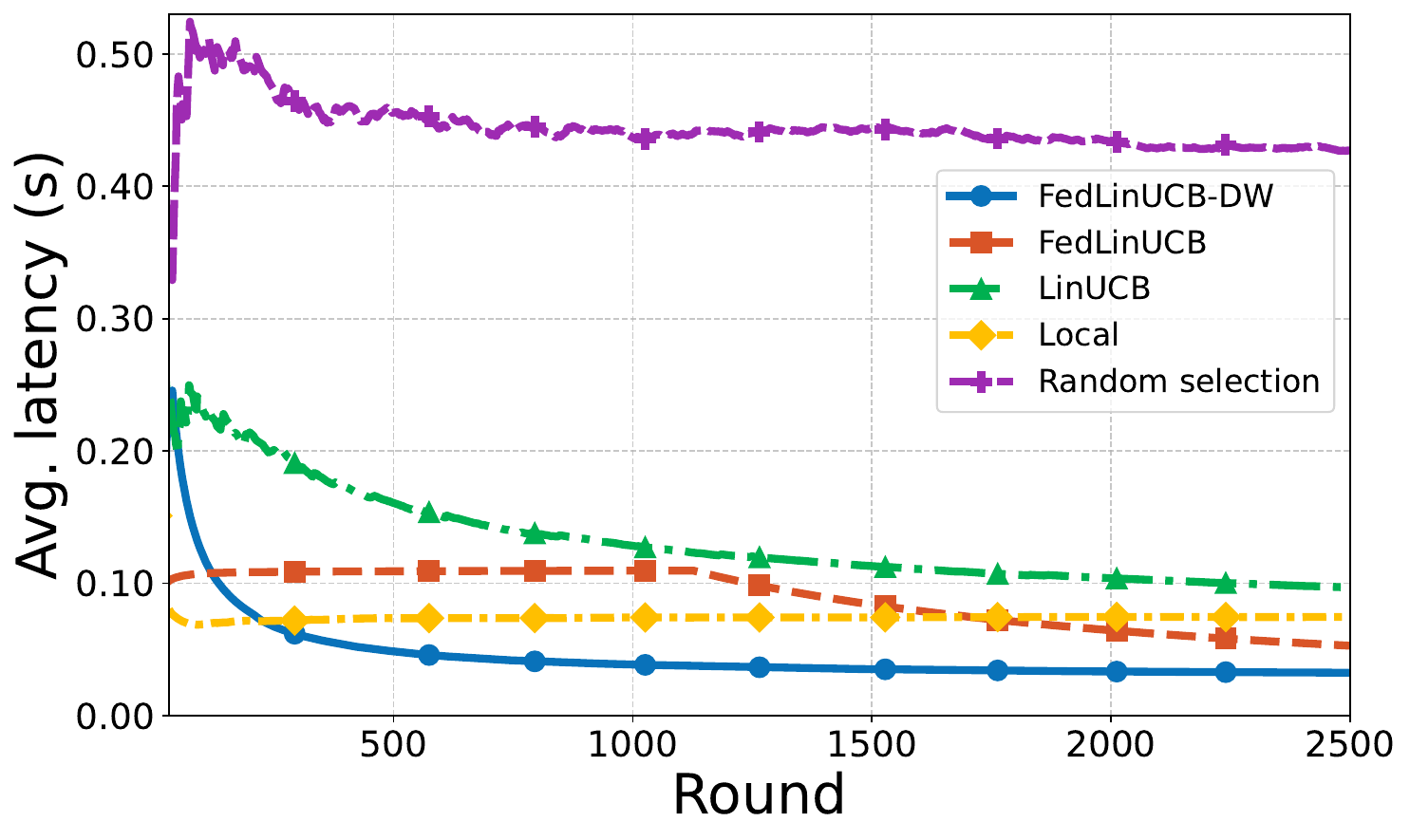}}
   	\subfloat[Average latency for ResNet]{\includegraphics[width = 0.33\textwidth]{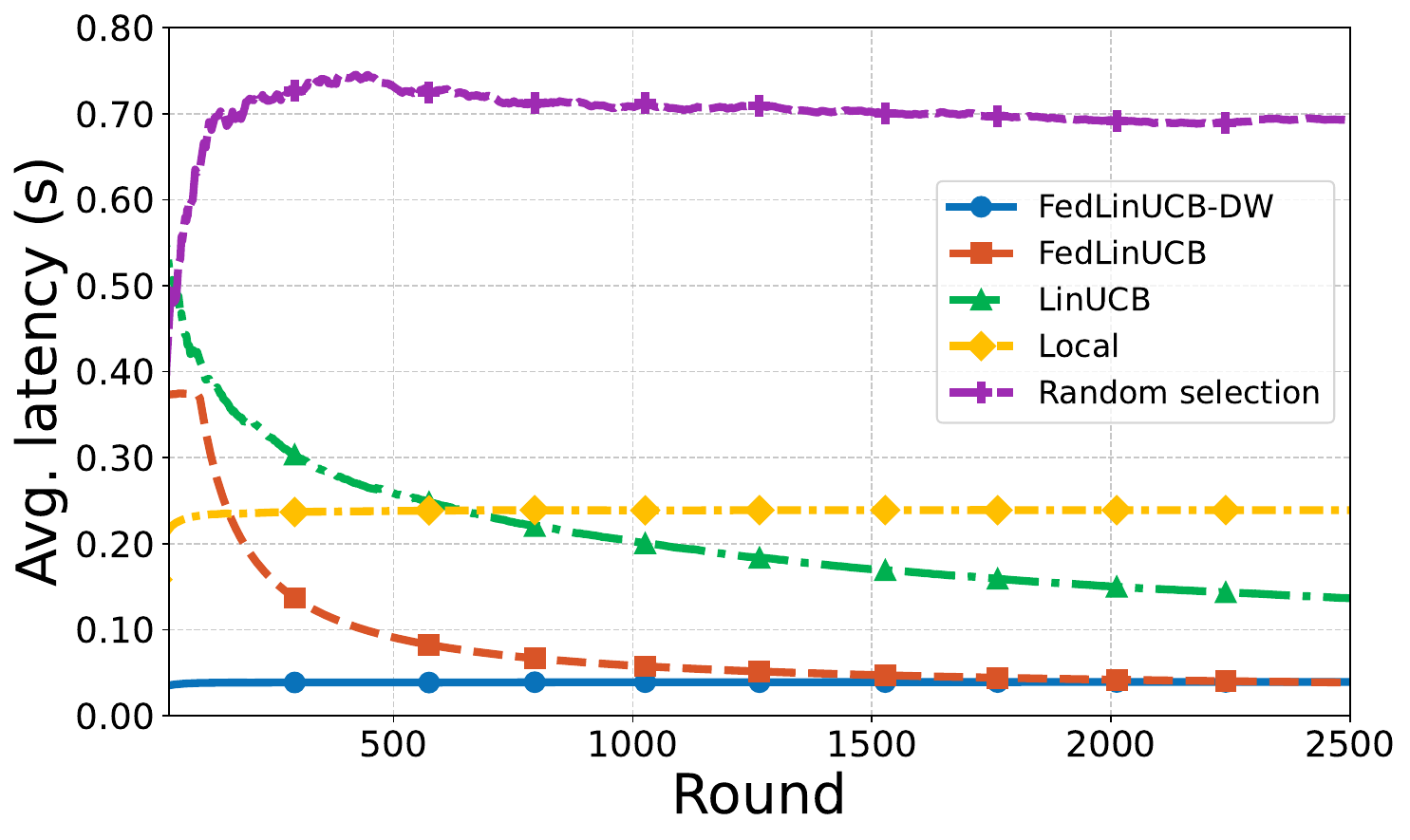}}
	\subfloat[Average latency for ViT]{\includegraphics[width = 0.33\textwidth]{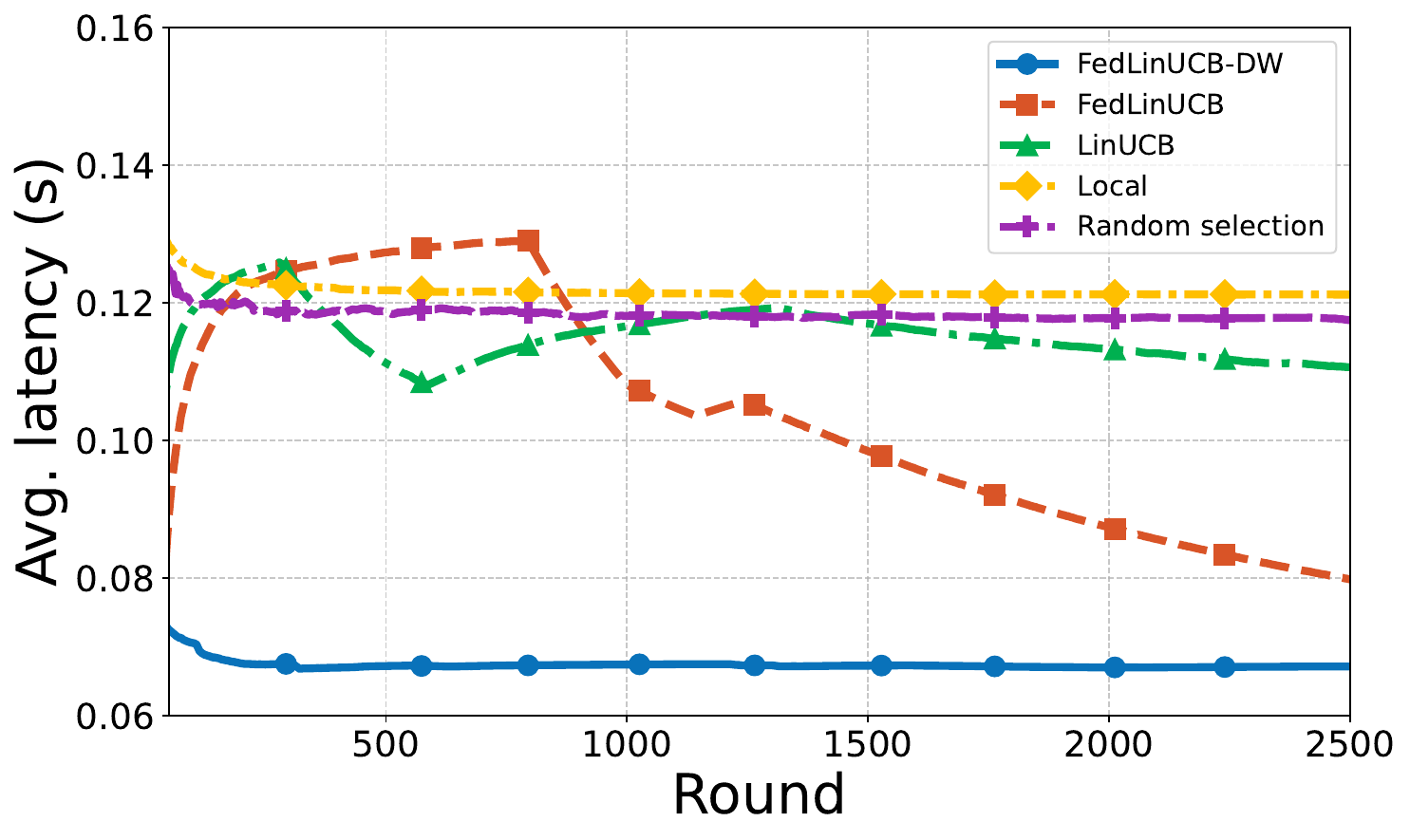}}
\caption{Average latency for different baselines on Jetson Orin Nano.}
\label{fig:latency_bs_nano}
\end{figure*}

\textbf{Learning protocol.} Each device (Xavier NX and  Jetson Orin Nano) is connected to the edge server through point-to-point socket connections over a local area network (LAN). The system adopts a fully asynchronous communication protocol, enabling each device to perform collaborative inference and exchange sufficient statistics with the server  independently, thereby eliminating inter-device dependencies.

Following the protocol of our simulation experiments, each device performs local early-exit inference at randomly sampled partition points. This process gathers the initial front-end statistics required for a warm-start initialization, enabling FedLinUCB-DW to effectively leverage offline experience during the subsequent online stage.

    \begin{figure*}[htbp]
\captionsetup[subfloat]{font=footnotesize}	
\centering
\subfloat[Average latency for VGG]{\includegraphics[width = 0.33\textwidth]{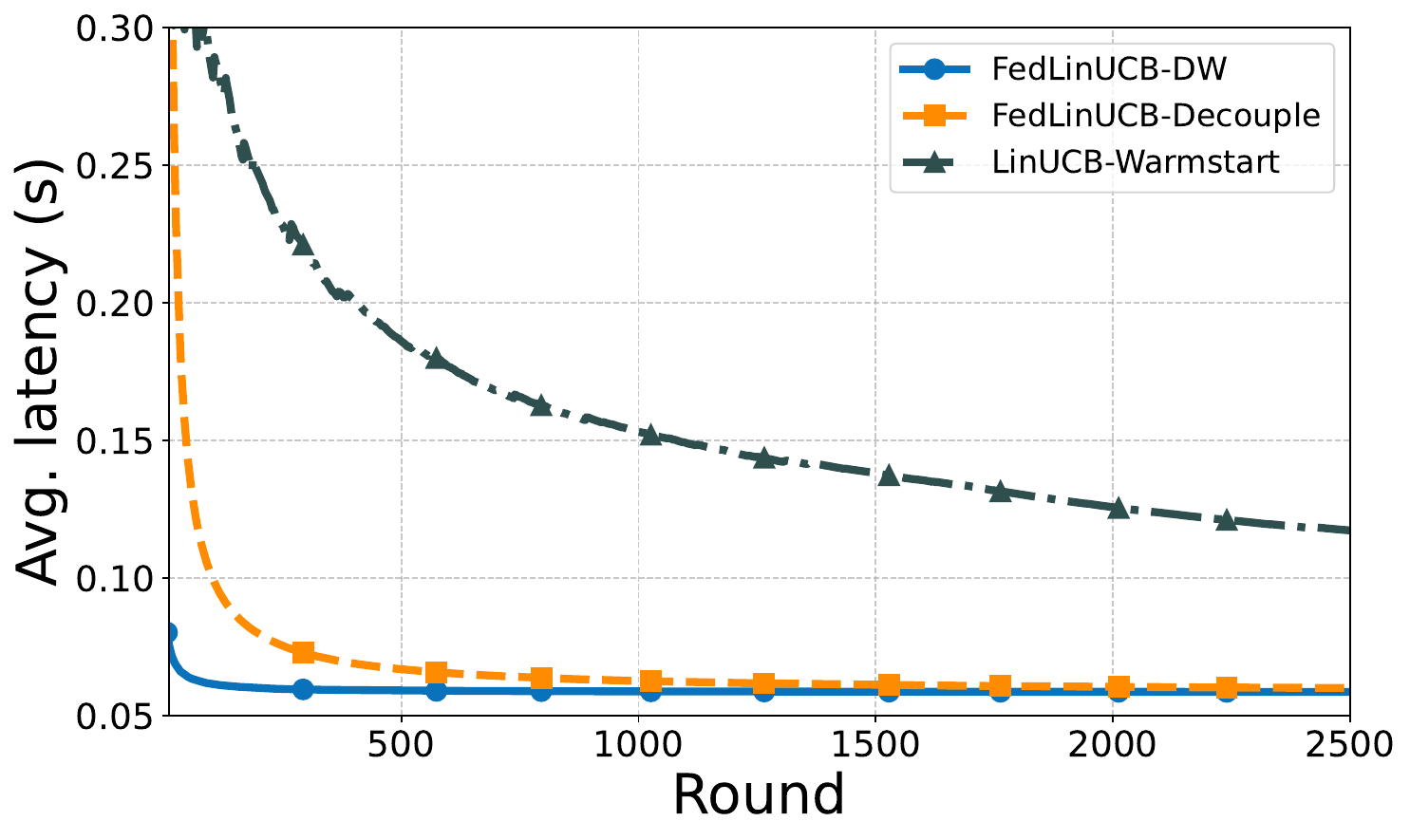}}
   	\subfloat[Average latency for ResNet]{\includegraphics[width = 0.33\textwidth]{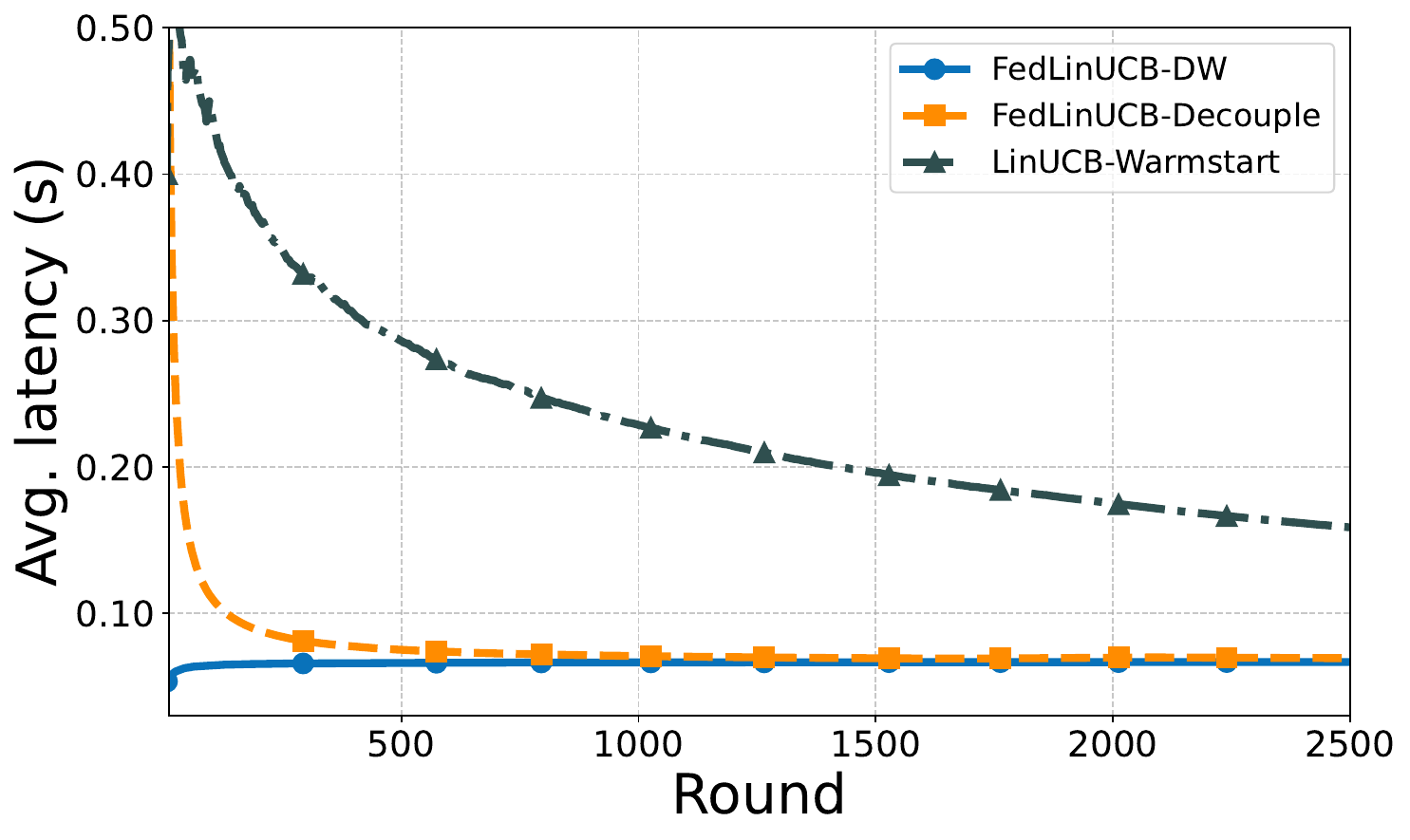}}
	\subfloat[Average latency for ViT]{\includegraphics[width = 0.33\textwidth]{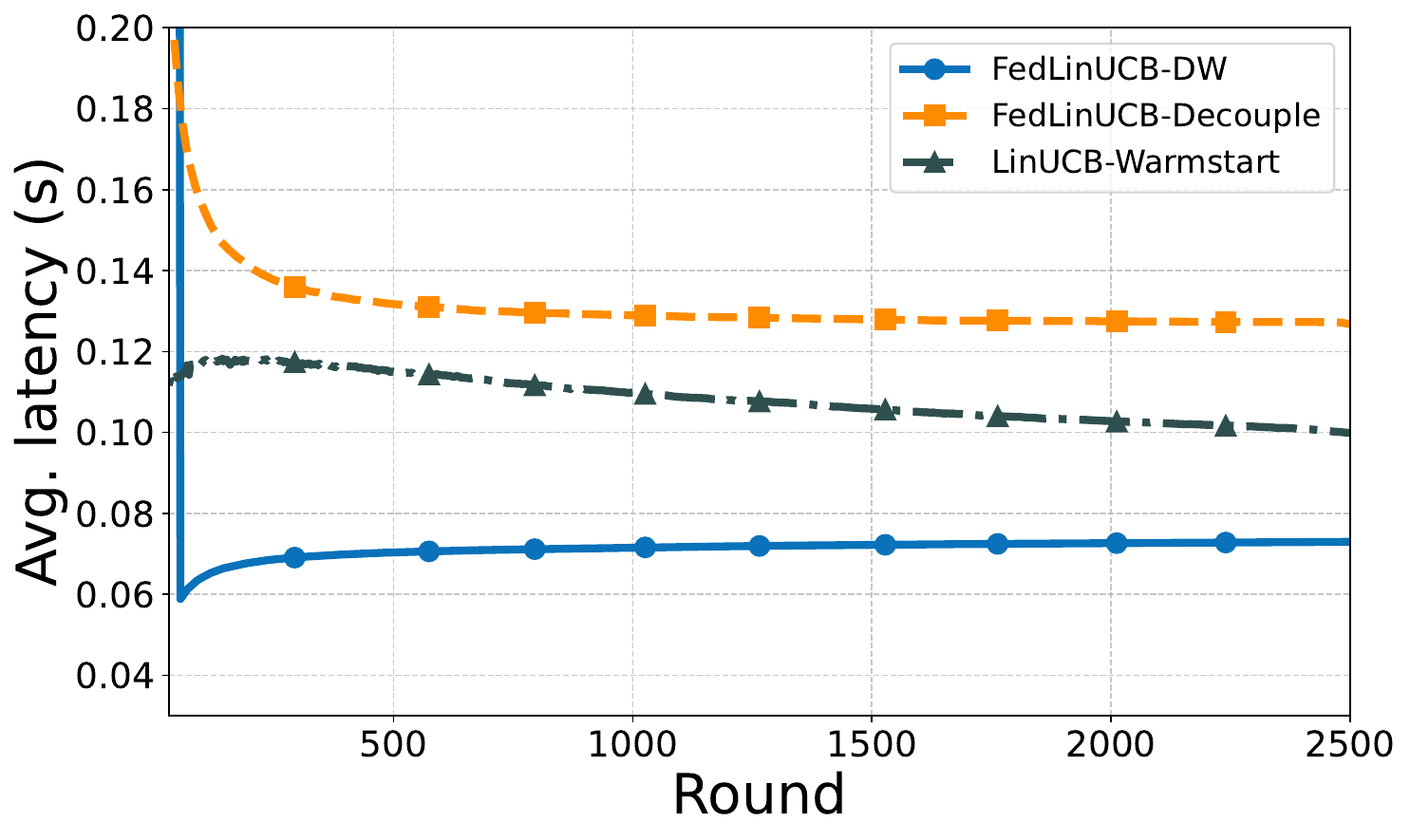}}
\caption{Average latency of ablation study on Xavier NX.}
\label{fig:latency_ablation_nx}
\end{figure*}

\begin{figure*}[htbp]
\captionsetup[subfloat]{font=footnotesize}	
\centering
\subfloat[Average latency for VGG]{\includegraphics[width = 0.33\textwidth]{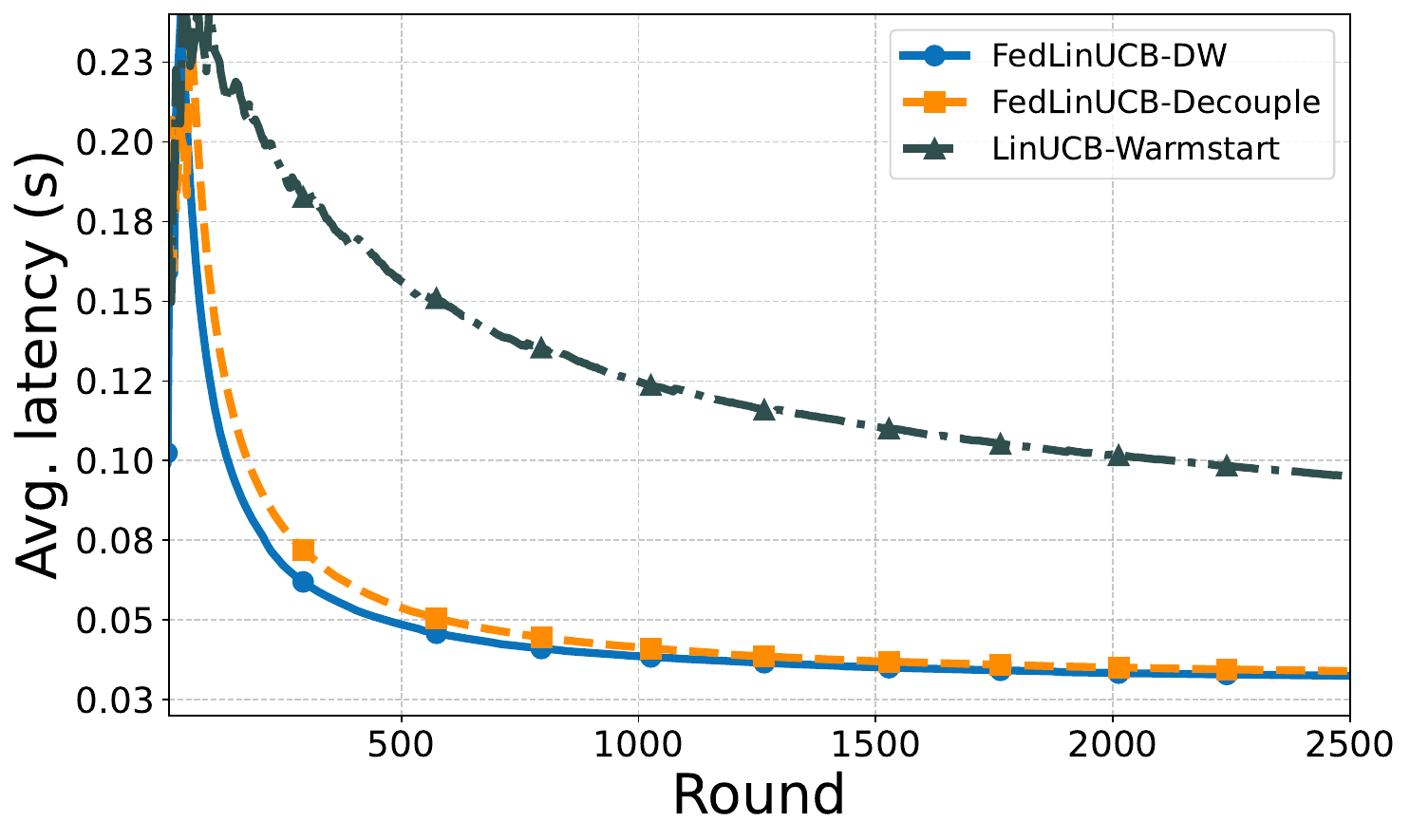}}
   	\subfloat[Average latency for ResNet]{\includegraphics[width = 0.33\textwidth]{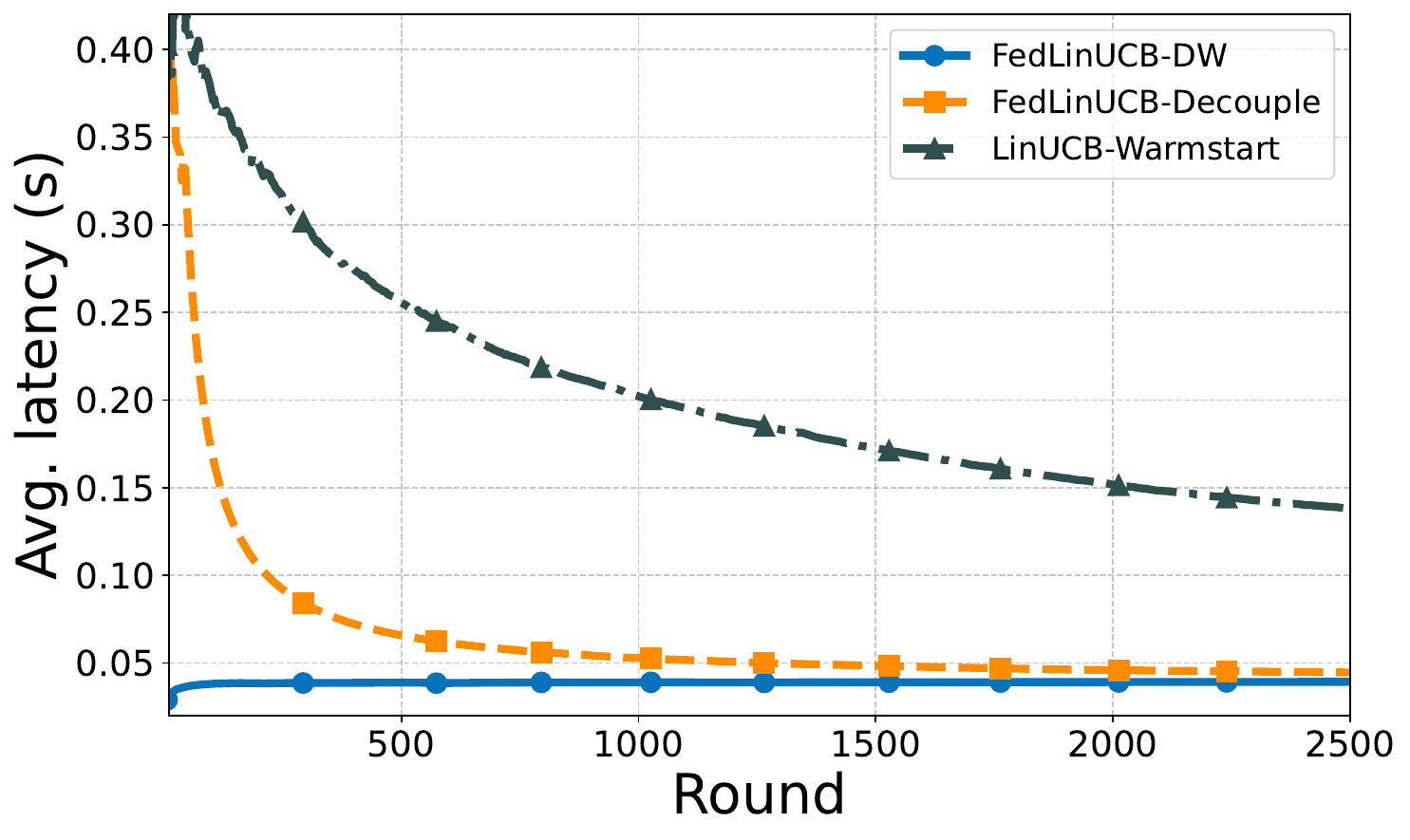}}
	\subfloat[Average latency for ViT]{\includegraphics[width = 0.33\textwidth]{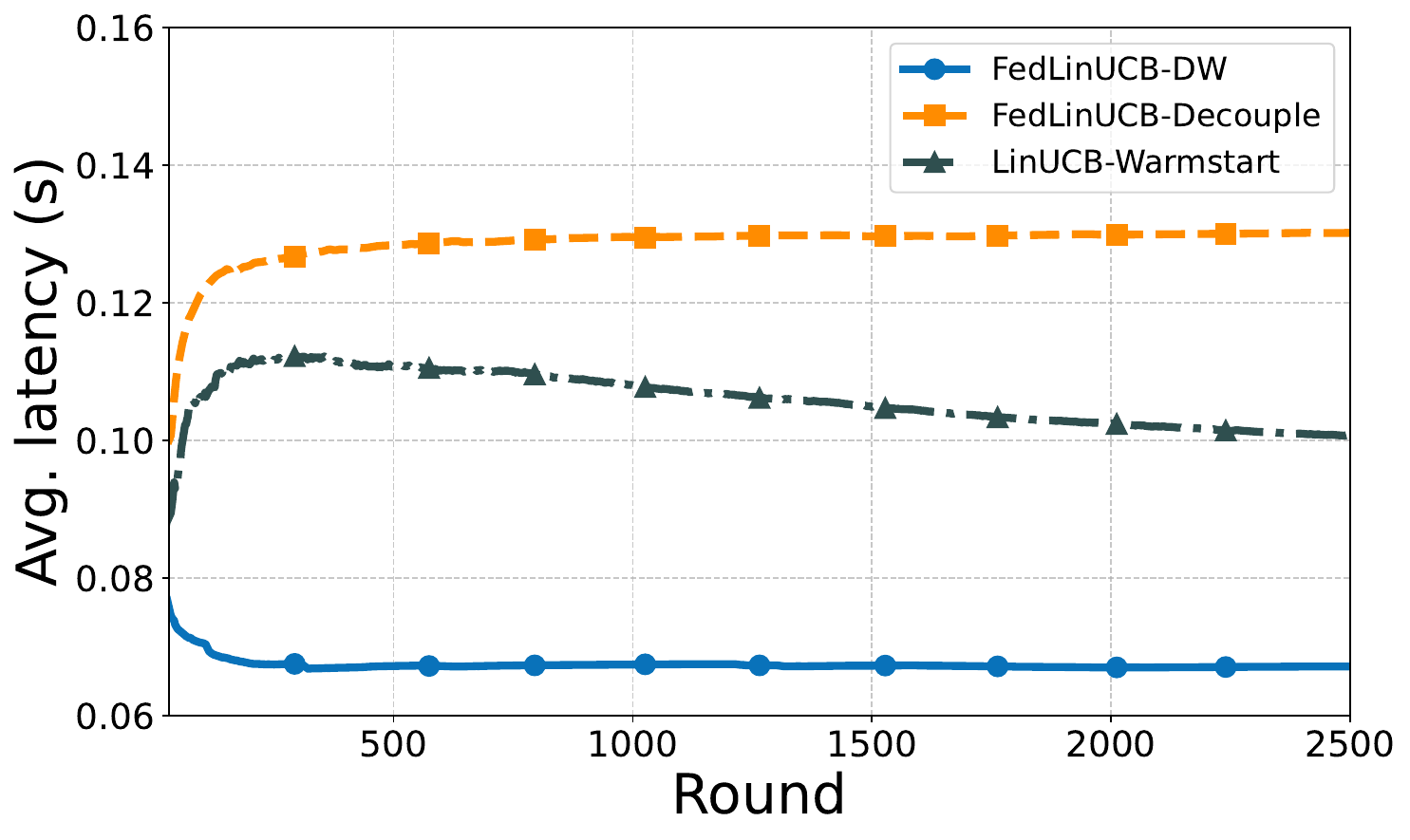}}
\caption{Average latency of ablation study on Jetson Orin Nano.}
\label{fig:latency_ablation_nano}
\end{figure*}

\textbf{Experimental results.} We first compare the average latency of the proposed FedLinUCB-DW against various baseline methods. The experimental results for the Jetson Xavier NX and Jetson Orin Nano are reported in Fig. \ref{fig:latency_bs_nx} and Fig. \ref{fig:latency_bs_nano}, respectively. We observe that, across three different backbone models on both devices, FedLinUCB-DW consistently achieves the lowest average end-to-end inference latency. For VGG-16 and ViT-16, the latency of FedLinUCB-DW decreases rapidly, as it jointly leverages decoupled shared feedback and offline experience to accelerate parameter estimation. For ResNet-50, FedLinUCB-DW attains the lowest latency from the initial rounds, owing to the warm-start initialization, which provides an accurate estimate of the system parameters. Besides, we observe that, unlike the simulation results, the average latency of FedLinUCB in the hardware prototype experiments decreases rapidly over rounds and eventually converges to a level comparable to FedLinUCB-DW, outperforming other baselines. This is because the two devices used in our prototype, the Xavier NX and Jetson Orin Nano, possess similar computational capabilities. The low degree of hardware heterogeneity ensures that sharing statistics for the overall system parameters in FedLinUCB does not introduce significant bias. Moreover, although LinUCB shows a decrease in average latency as the number of rounds increases, its final performance still lags behind FedLinUCB-DW and FedLinUCB. This is because LinUCB only allows each device to estimate parameters using its own local feedback, leading to low learning efficiency and making it difficult to quickly identify the optimal partition point.

We then conduct an ablation study of the proposed FedLinUCB-DW on the Jetson Xavier NX and Jetson Orin Nano, with results reported in Figure~\ref{fig:latency_ablation_nx} and Figure~\ref{fig:latency_ablation_nano}, respectively. As previously discussed, FedLinUCB-DW consists of two key components: (1) decoupling the cooperative learning process of front-end and back-end parameters, and (2) leveraging offline inference experience for warm-start initialization. To isolate their individual effects and evaluate the effectiveness of these components, we compare FedLinUCB-DW with two variants: \textbf{FedLinUCB-Decouple}, which performs decoupled cooperative online learning without offline warm-start, and \textbf{LinUCB-Warmstart}, which utilizes offline experience for initialization but does not involve any inter-device cooperation. From the experimental results, we first observe that FedLinUCB-DW outperforms the other two variants, indicating that the two components work jointly to achieve performance gains. Moreover, for VGG-16 and ResNet-50, FedLinUCB-Decouple exhibits performance comparable to FedLinUCB-DW and significantly outperforms LinUCB-Warmstart, suggesting that offline experience has a marginal impact on these specific tasks, whereas the majority of the performance improvements are derived from cooperative feedback sharing. In contrast, for the task of ViT-16, the warm-start component plays a more significant role. This is because the local statistics constructed from offline experience provide a strong initial estimate of the system parameters, enabling rapid identification of the optimal partition point and thereby reducing inference latency.

%% file: content/supplement.tex
\onecolumn

\clearpage
\begin{center}
    {\LARGE \bfseries Supplementary Material}
\end{center}
\appendices
    
In this supplementary material, we provide detailed proofs of the regret upper bound and communication complexity of the proposed FedLinUCB-DW algorithm. The overall analysis can be decomposed into two parts, corresponding to the front-end and back-end cooperative online learning processes. We note that the analysis of the back-end component follows arguments similar to those in~\cite{he2022simple}, which considers homogeneous agents without warm-starting from offline experience. Therefore, in the following, we mainly focus on the proofs for the front-end component.

\section*{Notations}

In this section, we summarize the main notations used in the following  analysis.
Table~\ref{tab:basic_notation} lists the basic system, latency, and model parameters.
Table~\ref{tab:frontend_statistics_notation} lists the offline and online front-end sufficient statistics used in the analysis.

During the online stage,
$\boldsymbol{\Sigma}^{f,\mathrm{up}}_{m,t}$ and $\mathbf b^{f,\mathrm{up}}_{m,t}$ denote online information of device $m$ that has already been uploaded to the server, while
$\boldsymbol{\Sigma}^{f,\mathrm{loc}}_{m,t}$ and $\mathbf b^{f,\mathrm{loc}}_{m,t}$ denote online information that remains in the local buffer of device $m$. The all-information quantities
$\boldsymbol{\Sigma}^{f,\mathrm{all}}_{\mathcal G_k,t}$ and
$\mathbf b^{f,\mathrm{all}}_{\mathcal G_k,t}$
are used only for the proof.
They contain the group-level offline information, all uploaded online information, and all local unuploaded online information within group $\mathcal G_k$.
We emphasize that the actual algorithm cannot access the all-information quantities under asynchronous communication.

\begin{table*}[htbp]
\centering
\caption{Basic notations.}
\label{tab:basic_notation}
\normalsize
\renewcommand{\arraystretch}{1.18}
\begin{tabularx}{\textwidth}{>{\centering\arraybackslash}p{0.30\textwidth}X}
\toprule
Notation & Meaning \\
\midrule

$M$ & Number of devices. \\

$K$ & Number of front-end groups. \\

$\mathcal G_k$ & The $k$-th front-end group. \\

$M_k=|\mathcal G_k|$ & Size of group $\mathcal G_k$. \\

$g(m)$ & Group index of device $m$. \\

$m_t$ & Active device at round $t$. \\

$k_t=g(m_t)$ & Active group index at round $t$. \\

$\mathcal P=[P]$ & Set of partition points. \\

$p_t$ & Selected partition point at round $t$. \\

$p_t^\star$ & Optimal partition point at round $t$. \\

$\mathbf x_p$ & Context of partition point $p$. \\

$\mathbf x_p^f$ & Front-end context of $p$. \\

$\mathbf x_p^b$ & Back-end context of $p$. \\

$d_f,d_b$ & Front-end and back-end dimensions. \\

$\boldsymbol{\theta}^f_m$ & Front-end parameter of device $m$. \\

$\boldsymbol{\theta}^b$ & Shared back-end parameter. \\

$\boldsymbol{\theta}_m$ & Full parameter of device $m$. \\

$\ell_t^f$ & Front-end latency at round $t$. \\

$\ell_t^b$ & Back-end latency at round $t$. \\

$\ell_t$ & Total latency at round $t$. \\

$\ell_{m,p}$ & Expected latency of device $m$ on $p$. \\

$\eta_t^f,\eta_t^b$ & Front-end and back-end noises. \\

$R_f,R_b$ & Sub-Gaussian noise parameters. \\

$L_f,L_b$ & Context norm bounds. \\

$S_f,S_b$ & Parameter norm bounds. \\

$\epsilon$ & Intra-group heterogeneity level. \\

$\lambda_f,\lambda_b$ & Ridge parameters. \\

$\alpha_f,\alpha_b$ & Communication trigger thresholds. \\

\bottomrule
\end{tabularx}
\end{table*}

\begin{table*}[htbp]
\centering
\caption{Offline and Online Front-end statistics.}
\label{tab:frontend_statistics_notation}
\normalsize
\renewcommand{\arraystretch}{1.18}
\begin{tabularx}{\textwidth}{>{\centering\arraybackslash}p{0.34\textwidth}X}
\toprule
Notation & Meaning \\
\midrule

$n_{\mathrm{off}}$ & Offline samples per device. \\

$\widetilde{\mathbf x}^f_{m,j}$ & Offline front-end context. \\

$\widetilde{\ell}^f_{m,j}$ & Offline front-end latency. \\

$\widetilde{\eta}^f_{m,j}$ & Offline front-end noise. \\

$\boldsymbol{\Sigma}^{f,\mathrm{off}}_m$ & Offline covariance of device $m$. \\

$\mathbf b^{f,\mathrm{off}}_m$ & Offline response vector of device $m$. \\

$\boldsymbol{\Sigma}^{f,\mathrm{off}}_{\mathcal G_k}$ & Group offline covariance. \\

$\mathbf b^{f,\mathrm{off}}_{\mathcal G_k}$ & Group offline response vector. \\

$\boldsymbol{\Sigma}^{f,0}_{\mathcal G_k}$ & Warm-start covariance  of group $\mathcal{G}_k$. \\

$\rho_k$ & Offline information strength  of group $\mathcal{G}_k$. \\

$\boldsymbol{\Sigma}^{f,\mathrm{up}}_{m,t}$ & Uploaded front-end covariance. \\

$\mathbf b^{f,\mathrm{up}}_{m,t}$ & Uploaded front-end response vector. \\

$\boldsymbol{\Sigma}^{f,\mathrm{loc}}_{m,t}$ & Local front-end covariance. \\

$\mathbf b^{f,\mathrm{loc}}_{m,t}$ & Local front-end response vector. \\

$\boldsymbol{\Sigma}^{f,\mathrm{ser}}_{\mathcal G_k,t}$ & Server front-end covariance  of group $\mathcal{G}_k$. \\

$\mathbf b^{f,\mathrm{ser}}_{\mathcal G_k,t}$ & Server front-end response vector  of group $\mathcal{G}_k$. \\

$\boldsymbol{\Sigma}^{f,\mathrm{all}}_{\mathcal G_k,t}$ & All-information front-end covariance  of group $\mathcal{G}_k$. \\

$\mathbf b^{f,\mathrm{all}}_{\mathcal G_k,t}$ & All-information front-end response vector  of group $\mathcal{G}_k$. \\

$\boldsymbol{\Sigma}^{f}_{m,t}$ & Front-end covariance held by device $m$. \\

$\mathbf b^f_{m,t}$ & Front-end response held by device $m$. \\

$\widehat{\boldsymbol{\theta}}^f_m(t)$ & Front-end estimator of device $m$. \\

$\widehat{\boldsymbol{\theta}}^{f,\mathrm{all}}_{\mathcal G_k}(t)$ & All-information front-end estimator of group $\mathcal{G}_k$. \\

\bottomrule
\end{tabularx}
\end{table*}


\section*{Proofs of the Regret Upper Bound}

\begin{lemma}[All-Information Confidence Bound]
For a group $\mathcal{G}_k$, under the boundedness, sub-Gaussian noise, and intra-group heterogeneity assumptions, with probability at least $1-\delta$, for all $t\in[T]$ and all target devices $m\in\mathcal{G}_k$, we have
\begin{align}
    \left\lVert 
    \widehat{\boldsymbol{\theta}}^{f,\mathrm{all}}_{\mathcal{G}_k}(t)-\boldsymbol{\theta}^f_m
     \right\rVert_{\boldsymbol{\Sigma}^{f,\mathrm{all}}_{\mathcal{G}_k,t}}
    \le
    \beta^{f,\mathrm{all}}_{k,T},
\end{align}
where
\begin{align}
\beta^{f,\mathrm{all}}_{k,T}
=
\sqrt{\lambda_f}S_f
+
R_f\sqrt{
d_f\log\left(1+\frac{\bar{N}_k(T)L_f^2}{\lambda_f d_f}\right)
+2\log\frac{M_kT}{\delta}
}
+
\epsilon L_f\sqrt{\bar{N}_k(T)}.
\end{align}

Equivalently, define
\begin{align}
\gamma^{f,\mathrm{all}}_{k,T}
=
R_f\sqrt{
d_f\log\left(1+\frac{\bar{N}_k(T)L_f^2}{\lambda_f d_f}\right)
+2\log\frac{M_kT}{\delta}
}
+
\epsilon L_f\sqrt{\bar{N}_k(T)}.    
\label{eq:gamma-all}
\end{align}

so that
\begin{align}
\beta^{f,\mathrm{all}}_{k,T}=\sqrt{\lambda_f}S_f+\gamma^{f,\mathrm{all}}_{k,T}.
\end{align}
\end{lemma}

\begin{proof} For a group $\mathcal G_k$, we define the all-information front-end covariance and response vector by
\begin{align}
    \boldsymbol{\Sigma}^{f,\mathrm{all}}_{\mathcal G_k,t}
    &=
    \lambda_f\mathbf I_{d_f}
    +
    \boldsymbol{\Sigma}^{f,\mathrm{off}}_{\mathcal G_k}
    +
    \sum_{\substack{s\le t\\m_s\in\mathcal G_k}}
    \mathbf x^f_{p_s}
    \left(\mathbf x^f_{p_s}\right)^\top,
    \label{eq:all-cov}\\
    \mathbf b^{f,\mathrm{all}}_{\mathcal G_k,t}
    &=
    \mathbf b^{f,\mathrm{off}}_{\mathcal G_k}
    +
    \sum_{\substack{s\le t\\m_s\in\mathcal G_k}}
    \mathbf x^f_{p_s}\ell^f_s.
    \label{eq:all-b}
\end{align}

The all-information estimator is
\begin{align}
    \widehat{\boldsymbol{\theta}}^{f,\mathrm{all}}_{\mathcal G_k}(t)
    =
    \left(
    \boldsymbol{\Sigma}^{f,\mathrm{all}}_{\mathcal G_k,t}
    \right)^{-1}
    \mathbf b^{f,\mathrm{all}}_{\mathcal G_k,t}.
\end{align}

Let
\begin{align}
N_k(t)
&=
\sum_{s=1}^{t}
\mathbf 1\{m_s\in\mathcal G_k\},\\
\bar{N}_k(t)&=M_kn_{\mathrm{off}}+N_k(t).
\end{align}

We define the all-information noise vector as
\begin{align}
    \mathbf u^{f,\mathrm{all}}_{\mathcal G_k,t}
    =
    \sum_{m\in\mathcal G_k}
    \sum_{j=1}^{n_{\mathrm{off}}}
    \widetilde{\mathbf x}^f_{m,j}\widetilde{\eta}^f_{m,j}
    +
    \sum_{\substack{s\le t\\m_s\in\mathcal G_k}}
    \mathbf x^f_{p_s}\eta^f_s.
\end{align}

Fix a target device $m\in\mathcal G_k$. For any sample $\xb$ generated by a device $n\in\mathcal G_k$, we have
\begin{align}
    \ell^f
    &=
    \left\langle
    \boldsymbol{\theta}^f_n,
    \mathbf x^f
    \right\rangle
    +
    \eta^f
    =
    \left\langle
    \boldsymbol{\theta}^f_m,
    \mathbf x^f
    \right\rangle
    +
    \left\langle
    \boldsymbol{\theta}^f_n-\boldsymbol{\theta}^f_m,
    \mathbf x^f
    \right\rangle
    +
    \eta^f.
\end{align}

Using this decomposition in \eqref{eq:all-b}, we obtain
\begin{align}
\mathbf b^{f,\mathrm{all}}_{\mathcal G_k,t}
=&
\left(
\boldsymbol{\Sigma}^{f,\mathrm{all}}_{\mathcal G_k,t}
-
\lambda_f\mathbf I_{d_f}
\right)
\boldsymbol{\theta}^f_m
+
\mathbf u^{f,\mathrm{all}}_{\mathcal G_k,t}+
\sum_{n\in\mathcal G_k}
\sum_{j=1}^{n_{\mathrm{off}}}
\widetilde{\mathbf x}^f_{n,j}
\left(\widetilde{\mathbf x}^f_{n,j}\right)^\top
\left(
\boldsymbol{\theta}^f_n-
\boldsymbol{\theta}^f_m
\right)
+
\sum_{\substack{s\le t\\m_s\in\mathcal G_k}}
\mathbf x^f_{p_s}
\left(\mathbf x^f_{p_s}\right)^\top
\left(
\boldsymbol{\theta}^f_{m_s}-
\boldsymbol{\theta}^f_m
\right).
\label{eq:all-b-decomp-no-h}
\end{align}

Multiplying both sides by $\left(\boldsymbol{\Sigma}^{f,\mathrm{all}}_{\mathcal G_k,t}\right)^{-1}$ and subtracting $\boldsymbol{\theta}^f_m$ gives
\begin{align}
\widehat{\boldsymbol{\theta}}^{f,\mathrm{all}}_{\mathcal G_k}(t)
-
\boldsymbol{\theta}^f_m
&=
-
\lambda_f
\left(
\boldsymbol{\Sigma}^{f,\mathrm{all}}_{\mathcal G_k,t}
\right)^{-1}
\boldsymbol{\theta}^f_m
+
\left(
\boldsymbol{\Sigma}^{f,\mathrm{all}}_{\mathcal G_k,t}
\right)^{-1}
\mathbf u^{f,\mathrm{all}}_{\mathcal G_k,t}
+
\left(
\boldsymbol{\Sigma}^{f,\mathrm{all}}_{\mathcal G_k,t}
\right)^{-1}
\Bigg[
\sum_{n\in\mathcal G_k}
\sum_{j=1}^{n_{\mathrm{off}}}
\widetilde{\mathbf x}^f_{n,j}
\left(\widetilde{\mathbf x}^f_{n,j}\right)^\top
\left(
\boldsymbol{\theta}^f_n-
\boldsymbol{\theta}^f_m
\right)
\nonumber\\
&\quad+
\sum_{\substack{s\le t\\m_s\in\mathcal G_k}}
\mathbf x^f_{p_s}
\left(\mathbf x^f_{p_s}\right)^\top
\left(
\boldsymbol{\theta}^f_{m_s}-
\boldsymbol{\theta}^f_m
\right)
\Bigg].
\end{align}

Taking the $\boldsymbol{\Sigma}^{f,\mathrm{all}}_{\mathcal G_k,t}$-norm yields three terms.
The regularization term satisfies
\begin{align}
&\left\|
\lambda_f
\left(
\boldsymbol{\Sigma}^{f,\mathrm{all}}_{\mathcal G_k,t}
\right)^{-1}
\boldsymbol{\theta}^f_m
\right\|_{\boldsymbol{\Sigma}^{f,\mathrm{all}}_{\mathcal G_k,t}}
\nonumber=
\lambda_f
\sqrt{
\left(\boldsymbol{\theta}^f_m\right)^\top
\left(
\boldsymbol{\Sigma}^{f,\mathrm{all}}_{\mathcal G_k,t}
\right)^{-1}
\boldsymbol{\theta}^f_m
}
\le
\sqrt{\lambda_f}S_f.
\end{align}

By the self-normalized concentration inequality and a union bound over $m\in\mathcal G_k$ and $t\le T$,
\begin{align}
\left\|
\mathbf u^{f,\mathrm{all}}_{\mathcal G_k,t}
\right\|_{
\left(
\boldsymbol{\Sigma}^{f,\mathrm{all}}_{\mathcal G_k,t}
\right)^{-1}}
\le
R_f
\sqrt{
    d_f
    \log\left(
        1+
        \frac{\bar{N}_k(T)L_f^2}{\lambda_fd_f}
    \right)
    +
    2\log\frac{M_kT}{\delta}
}.
\end{align}

It remains to control the explicit heterogeneity terms. By dual norm, for any vector $\mathbf z$ satisfying
\begin{align}
    \left\|\mathbf z\right\|_{\boldsymbol{\Sigma}^{f,\mathrm{all}}_{\mathcal G_k,t}}
    \le 1,
\end{align}

we have
\begin{align}
&\mathbf z^\top
\Bigg[
\sum_{n\in\mathcal G_k}
\sum_{j=1}^{n_{\mathrm{off}}}
\widetilde{\mathbf x}^f_{n,j}
\left(\widetilde{\mathbf x}^f_{n,j}\right)^\top
\left(
\boldsymbol{\theta}^f_n-
\boldsymbol{\theta}^f_m
\right)
+
\sum_{\substack{s\le t\\m_s\in\mathcal G_k}}
\mathbf x^f_{p_s}
\left(\mathbf x^f_{p_s}\right)^\top
\left(
\boldsymbol{\theta}^f_{m_s}-
\boldsymbol{\theta}^f_m
\right)
\Bigg]\\
&\le
\left(
\sum_{r\in\mathcal I_k(t)}
\left(\mathbf z^\top\mathbf x^f_r\right)^2
\right)^{1/2}
\left(
\sum_{r\in\mathcal I_k(t)}
\left[
\left(\mathbf x^f_r\right)^\top
\left(
\boldsymbol{\theta}^f_{a(r)}-
\boldsymbol{\theta}^f_m
\right)
\right]^2
\right)^{1/2},
\end{align}
where $\mathcal I_k(t)$ is the set of all offline front-end samples in group $\mathcal G_k$ and all online front-end samples from group $\mathcal G_k$ up to round $t$, and $a(r)$ is the device that generated sample $r$.

The first factor is at most one because
\begin{align}
\sum_{r\in\mathcal I_k(t)}
\left(\mathbf z^\top\mathbf x^f_r\right)^2
\le
\mathbf z^\top
\boldsymbol{\Sigma}^{f,\mathrm{all}}_{\mathcal G_k,t}
\mathbf z
\le 1.
\end{align}

For the second factor, since $a(r),m\in\mathcal G_k$,
\begin{align}
\left\|
\boldsymbol{\theta}^f_{a(r)}-
\boldsymbol{\theta}^f_m
\right\|_2
\le
\epsilon,
\end{align}
and $\left\|\mathbf x^f_r\right\|_2\le L_f$, we have
\begin{align}
\left|
\left(\mathbf x^f_r\right)^\top
\left(
\boldsymbol{\theta}^f_{a(r)}-
\boldsymbol{\theta}^f_m
\right)
\right|
\le
\epsilon L_f.
\end{align}

Since $|\mathcal I_k(t)|\le \bar{N}_k(T)$, the explicit heterogeneity terms are bounded by
\begin{align}
\epsilon L_f\sqrt{\bar{N}_k(T)}.
\end{align}

Combining the three bounds proves the lemma.
\end{proof}

\subsection{Per-Agent Uploaded and Local Information}

Below we define the per-agent information. For each device $m$, let $\tau^f_m(t)$ denote the most recent front-end upload time of device $m$ before or at round $t$. Online samples with $s\le \tau^f_m(t)$ have been uploaded, while samples with $s>\tau^f_m(t)$ remain in the local buffer.

For device $m$, the uploaded front-end covariance, response, and noise are
\begin{align}
\boldsymbol{\Sigma}^{f,\mathrm{up}}_{m,t}
&=
\sum_{\substack{s\le t\\m_s=m\\s\le \tau^f_m(t)}}
\mathbf{x}^f_{p_s}(\mathbf{x}^f_{p_s})^\top,\\
\mathbf{b}^{f,\mathrm{up}}_{m,t}
&=
\sum_{\substack{s\le t\\m_s=m\\s\le \tau^f_m(t)}}
\mathbf{x}^f_{p_s}\ell_s^f,\\
\mathbf{u}^{f,\mathrm{up}}_{m,t}
&=
\sum_{\substack{s\le t\\m_s=m\\s\le \tau^f_m(t)}}
\mathbf{x}^f_{p_s}\eta_s^f.
\end{align}

The local unuploaded covariance, response, and noise are
\begin{align}
\boldsymbol{\Sigma}^{f,\mathrm{loc}}_{m,t}
&=
\sum_{\substack{s\le t\\m_s=m\\s> \tau^f_m(t)}}
\mathbf{x}^f_{p_s}(\mathbf{x}^f_{p_s})^\top,\\
\mathbf{b}^{f,\mathrm{loc}}_{m,t}
&=
\sum_{\substack{s\le t\\m_s=m\\s> \tau^f_m(t)}}
\mathbf{x}^f_{p_s}\ell_s^f,\\
\mathbf{u}^{f,\mathrm{loc}}_{m,t}
&=
\sum_{\substack{s\le t\\m_s=m\\s> \tau^f_m(t)}}
\mathbf{x}^f_{p_s}\eta_s^f.
\end{align}


Based on these definitions, if $m\in\mathcal{G}_k$,  the server-side group covariance is
\begin{align}
    \boldsymbol{\Sigma}^{f,\mathrm{ser}}_{\mathcal{G}_k,t}
    =
    \lambda_f\mathbf{I}_{d_f}
    +
    \boldsymbol{\Sigma}^{f,\mathrm{off}}_{\mathcal{G}_k}
    +
    \sum_{m\in\mathcal{G}_k}\boldsymbol{\Sigma}^{f,\mathrm{up}}_{m,t},
\end{align}
and the all-information covariance satisfies
\begin{align}
    \boldsymbol{\Sigma}^{f,\mathrm{all}}_{\mathcal{G}_k,t}
    =
    \boldsymbol{\Sigma}^{f,\mathrm{ser}}_{\mathcal{G}_k,t}
    +
    \sum_{m\in\mathcal{G}_k}\boldsymbol{\Sigma}^{f,\mathrm{loc}}_{m,t}.
\end{align}

\begin{lemma}[Local Front-End Concentration]
For a group $\mathcal G_k$, with probability at least $1-\delta$, for all $t\le T$ and all $m,n\in\mathcal G_k$, we have
\begin{align}
&\Bigg\|
\mathbf u^{f,\mathrm{loc}}_{n,t}
+
\sum_{\substack{s\le t\\m_s=n\\s>\tau^f_n(t)}}
\mathbf x^f_{p_s}
\left(\mathbf x^f_{p_s}\right)^\top
\left(
\boldsymbol{\theta}^f_n-
\boldsymbol{\theta}^f_m
\right)
\Bigg\|_{
\left(
\alpha_f\lambda_f\mathbf I_{d_f}
+
\boldsymbol{\Sigma}^{f,\mathrm{loc}}_{n,t}
\right)^{-1}}
\le
\gamma^{f,\mathrm{loc}}_{k,T},
\end{align}
where
\begin{align}
\gamma^{f,\mathrm{loc}}_{k,T}
=&
R_f
\sqrt{
    d_f
    \log\left(
        1+
        \frac{N_k(T)L_f^2}{\alpha_f\lambda_fd_f}
    \right)
    +
    2\log\frac{M_kT^2}{\delta}
}
+
\epsilon L_f\sqrt{N_k(T)}.
\label{eq:gamma-loc}
\end{align}
\end{lemma}

\begin{proof}
For a fixed local interval, the self-normalized concentration inequality gives
\begin{align}
&\left\|
\mathbf u^{f,\mathrm{loc}}_{n,t}
\right\|_{
\left(
\alpha_f\lambda_f\mathbf I_{d_f}
+
\boldsymbol{\Sigma}^{f,\mathrm{loc}}_{n,t}
\right)^{-1}}
\le
R_f
\sqrt{
    d_f
    \log\left(
        1+
        \frac{N_k(T)L_f^2}{\alpha_f\lambda_fd_f}
    \right)
    +
    2\log\frac{M_kT^2}{\delta}
},
\end{align}
where the union bound is taken over devices and all possible local intervals.
For the explicit heterogeneity term, use the same dual-norm and Cauchy-Schwarz argument as in the all-information case, restricted to the local index set
\begin{align}
    \{s\le t:m_s=n,\ s>\tau^f_n(t)\}.
\end{align}
The first Cauchy-Schwarz factor is at most one in the norm induced by
\begin{align}
    \alpha_f\lambda_f\mathbf I_{d_f}
    +
    \boldsymbol{\Sigma}^{f,\mathrm{loc}}_{n,t},
\end{align}
and the second factor is at most $\epsilon L_f\sqrt{N_k(T)}$. Combining the stochastic and deterministic terms gives \eqref{eq:gamma-loc}.
\end{proof}

\begin{lemma}[Front-End Covariance Comparison]\label{lemma:covcomparion}
For a front-end group $\mathcal G_k$, under the front-end determinant-triggered communication rule, after the possible communication at round $t$, for each device $m\in\mathcal G_k$ it holds that
\begin{align}
\boldsymbol{\Sigma}^{f,\mathrm{ser}}_{\mathcal G_k,t}
=
\lambda_f\mathbf I_{d_f}
+
\boldsymbol{\Sigma}^{f,\mathrm{off}}_{\mathcal G_k}
+
\sum_{n\in\mathcal G_k}
\boldsymbol{\Sigma}^{f,\mathrm{up}}_{n,t}
\succeq
\frac{1}{\alpha_f}
\boldsymbol{\Sigma}^{f,\mathrm{loc}}_{m,t}.
\label{eq:frontend_cov_comp_1}
\end{align}
Consequently,
\begin{align}
\boldsymbol{\Sigma}^{f,\mathrm{ser}}_{\mathcal G_k,t}
\succeq
\frac{1}{1+M_k\alpha_f}
\boldsymbol{\Sigma}^{f,\mathrm{all}}_{\mathcal G_k,t},
\label{eq:frontend_cov_comp_server_all}
\end{align}
and
\begin{align}
\boldsymbol{\Sigma}^{f,\mathrm{ser}}_{\mathcal G_k,t}
\succeq
\frac{1}{2\alpha_f}
\left(
\alpha_f\lambda_f\mathbf I_{d_f}
+
\boldsymbol{\Sigma}^{f,\mathrm{loc}}_{m,t}
\right).
\label{eq:frontend_cov_comp_server_loc_reg}
\end{align}
Moreover, for any $1\le t_1<t_2\le T$, if device $m\in\mathcal G_k$ is the only active device in group $\mathcal G_k$ from round $t_1$ to round $t_2-1$, and device $m$ communicates with the group server only at round $t_1$, then for all $t_1+1\le t\le t_2$,
\begin{align}
\boldsymbol{\Sigma}^{f}_{m,t}
\succeq
\frac{1}{1+M_k\alpha_f}
\boldsymbol{\Sigma}^{f,\mathrm{all}}_{\mathcal G_k,t}.
\label{eq:frontend_cov_comp_2}
\end{align}
\end{lemma}

The proof follows the determinant-triggered covariance comparison argument in FedLinUCB~\cite{he2022simple}. The additional offline covariance \(\Sigma^{f,\mathrm{off}}_{G_k}\) is positive semidefinite and is included in the base matrix, so it only strengthens the covariance lower bounds and does not affect the determinant-growth argument.

\begin{lemma}[Local Front-End Confidence]
Fix a group $\mathcal G_k$. With probability at least $1-\delta$, for all rounds $t\in[T]$ and every device $m\in\mathcal G_k$,
\begin{align}
\left\|
\widehat{\boldsymbol{\theta}}^f_m(t)
-
\boldsymbol{\theta}^f_m
\right\|_{
\boldsymbol{\Sigma}^{f}_{m,t}}
\le
\beta^f_{k,T},
\end{align}
where
\begin{align}
\beta^f_{k,T}
=
\sqrt{\lambda_f}S_f
+
\sqrt{1+M_k\alpha_f}\gamma^{f,\mathrm{all}}_{k,T}
+
M_k\sqrt{2\alpha_f}\gamma^{f,\mathrm{loc}}_{k,T}.
\label{eq:beta-front}
\end{align}
\end{lemma}

\begin{proof}
It is enough to prove the result for synchronization rounds. If device $m$ does not synchronize at a round, then its front-end statistics and estimator remain unchanged, and the bound is inherited from the previous synchronization round.

Suppose device $m\in\mathcal G_k$ synchronizes with the front-end group server at the end of round $t$. Then the statistics used after synchronization satisfy
\begin{align}
\boldsymbol{\Sigma}^f_{m,t+1}
&=
\boldsymbol{\Sigma}^{f,\mathrm{ser}}_{\mathcal G_k,t},\quad \quad \mathbf b^f_{m,t+1}=
\mathbf b^{f,\mathrm{ser}}_{\mathcal G_k,t}.
\end{align}
The server-side response vector can be decomposed relative to $\boldsymbol{\theta}^f_m$ as
\begin{align}
\mathbf b^{f,\mathrm{ser}}_{\mathcal G_k,t}
=&
\left(
\boldsymbol{\Sigma}^{f,\mathrm{ser}}_{\mathcal G_k,t}
-
\lambda_f\mathbf I_{d_f}
\right)
\boldsymbol{\theta}^f_m
+
\mathbf u^{f,\mathrm{ser}}_{\mathcal G_k,t}
+
\sum_{n\in\mathcal G_k}
\sum_{j=1}^{n_{\mathrm{off}}}
\widetilde{\mathbf x}^f_{n,j}
\left(\widetilde{\mathbf x}^f_{n,j}\right)^\top
\left(
\boldsymbol{\theta}^f_n-\boldsymbol{\theta}^f_m
\right)
+
\sum_{n\in\mathcal G_k}
\sum_{\substack{s\le t\\m_s=n\\s\le\tau^f_n(t)}}
\mathbf x^f_{p_s}
\left(\mathbf x^f_{p_s}\right)^\top
\left(
\boldsymbol{\theta}^f_n-\boldsymbol{\theta}^f_m
\right),
\label{eq:server-b-decomp}
\end{align}
where
\begin{align}
\mathbf u^{f,\mathrm{ser}}_{\mathcal G_k,t}
=
\sum_{n\in\mathcal G_k}
\sum_{j=1}^{n_{\mathrm{off}}}
\widetilde{\mathbf x}^f_{n,j}\widetilde{\eta}^f_{n,j}
+
\sum_{n\in\mathcal G_k}
\sum_{\substack{s\le t\\m_s=n\\s\le\tau^f_n(t)}}
\mathbf x^f_{p_s}\eta^f_s.
\end{align}

Therefore, the estimation error is
\begin{align}
&\widehat{\boldsymbol{\theta}}^f_m(t+1)
-
\boldsymbol{\theta}^f_m
\\
&=
-\lambda_f
\left(\boldsymbol{\Sigma}^{f}_{m,t+1}\right)^{-1}
\boldsymbol{\theta}^f_m
+
\left(\boldsymbol{\Sigma}^{f}_{m,t+1}\right)^{-1}
\Bigg[
\mathbf u^{f,\mathrm{ser}}_{\mathcal G_k,t}
+
\sum_{n\in\mathcal G_k}
\sum_{j=1}^{n_{\mathrm{off}}}
\widetilde{\mathbf x}^f_{n,j}
\left(\widetilde{\mathbf x}^f_{n,j}\right)^\top
\left(
\boldsymbol{\theta}^f_n-\boldsymbol{\theta}^f_m
\right)
+
\sum_{n\in\mathcal G_k}
\sum_{\substack{s\le t\\m_s=n\\s\le\tau^f_n(t)}}
\mathbf x^f_{p_s}
\left(\mathbf x^f_{p_s}\right)^\top
\left(
\boldsymbol{\theta}^f_n-\boldsymbol{\theta}^f_m
\right)
\Bigg].
\end{align}
The regularization term is bounded by $\sqrt{\lambda_f}S_f$ as in Lemma 1.

Then we focus on another term:
\begin{align}
&\mathbf u^{f,\mathrm{ser}}_{\mathcal G_k,t}
+
\sum_{n\in\mathcal G_k}
\sum_{j=1}^{n_{\mathrm{off}}}
\widetilde{\mathbf x}^f_{n,j}
\left(\widetilde{\mathbf x}^f_{n,j}\right)^\top
\left(
\boldsymbol{\theta}^f_n-\boldsymbol{\theta}^f_m
\right)
+
\sum_{n\in\mathcal G_k}
\sum_{\substack{s\le t\\m_s=n\\s\le\tau^f_n(t)}}
\mathbf x^f_{p_s}
\left(\mathbf x^f_{p_s}\right)^\top
\left(
\boldsymbol{\theta}^f_n-\boldsymbol{\theta}^f_m
\right)
\\
&=
\underbrace{\Bigg[
\mathbf u^{f,\mathrm{all}}_{\mathcal G_k,t}
+
\sum_{n\in\mathcal G_k}
\sum_{j=1}^{n_{\mathrm{off}}}
\widetilde{\mathbf x}^f_{n,j}
\left(\widetilde{\mathbf x}^f_{n,j}\right)^\top
\left(
\boldsymbol{\theta}^f_n-\boldsymbol{\theta}^f_m
\right)
+
\sum_{\substack{s\le t\\m_s\in\mathcal G_k}}
\mathbf x^f_{p_s}
\left(\mathbf x^f_{p_s}\right)^\top
\left(
\boldsymbol{\theta}^f_{m_s}-\boldsymbol{\theta}^f_m
\right)
\Bigg]}_{\text{All information term}}
\\
&\quad-
\underbrace{\sum_{n\in\mathcal G_k}
\Bigg[
\mathbf u^{f,\mathrm{loc}}_{n,t}
+
\sum_{\substack{s\le t\\m_s=n\\s>\tau^f_n(t)}}
\mathbf x^f_{p_s}
\left(\mathbf x^f_{p_s}\right)^\top
\left(
\boldsymbol{\theta}^f_n-\boldsymbol{\theta}^f_m
\right)
\Bigg]}_{\text{Local information term}}.
\label{eq:ser-all-loc-no-he}
\end{align}

Using \eqref{eq:frontend_cov_comp_server_all},  the all-information term in \eqref{eq:ser-all-loc-no-he} is bounded by $\sqrt{1+M_k\alpha_f}\gamma^{f,\mathrm{all}}_{k,T}$.

For each local information term, using \eqref{eq:frontend_cov_comp_server_loc_reg}, and Lemma 2, this gives
\begin{align}
&\Bigg\|
\mathbf u^{f,\mathrm{loc}}_{n,t}
+
\sum_{\substack{s\le t\\m_s=n\\s>\tau^f_n(t)}}
\mathbf x^f_{p_s}
\left(\mathbf x^f_{p_s}\right)^\top
\left(
\boldsymbol{\theta}^f_n-\boldsymbol{\theta}^f_m
\right)
\Bigg\|_{
\left(\boldsymbol{\Sigma}^{f}_{m,t+1}\right)^{-1}}
\le
\sqrt{2\alpha_f}\gamma^{f,\mathrm{loc}}_{k,T}.
\end{align}
Summing the local terms over $n\in\mathcal G_k$ and adding the regularization term proves \eqref{eq:beta-front} for synchronization rounds. The all-round statement follows because the estimator is unchanged between two synchronizations.
\end{proof}

\begin{lemma}[Per-round regret]
Under the algorithm setting, with probability $1-\delta$, for each round $t\in[T]$, the regret satisfies
\begin{align}
 r_t
 =&
 \ell_{m_t,p_t}
 -
 \ell_{m_t,p_t^\star}
\le
 2\beta^f_{k_t,T}
 \left\|
 \mathbf x^f_{p_t}
 \right\|_{
 \left(\boldsymbol{\Sigma}^{f}_{m_t,t}\right)^{-1}}
 +
 2\beta^b_T
 \left\|
 \mathbf x^b_{p_t}
 \right\|_{
 \left(\boldsymbol{\Sigma}^{b}_{m_t,t}\right)^{-1}}.
\end{align}
\end{lemma}

\begin{proof}
For device $m$ and partition point $p$, define
\begin{align}
\ell_{m,p}
=
\left\langle
\boldsymbol{\theta}^f_m,
\mathbf x^f_p
\right\rangle
+
\left\langle
\boldsymbol{\theta}^b,
\mathbf x^b_p
\right\rangle.
\end{align}

At round $t$,
\begin{align}
    p_t^\star
    =
    \arg\min_{p\in\mathcal P}
    \ell_{m_t,p}.
\end{align}

The algorithm selects
\begin{align}
    p_t
    =
    \arg\min_{p\in\mathcal P}
    \left[
    \widehat{\ell}_{m_t,p}(t)
    -
    \mathrm{CB}_{m_t,p}(t)
    \right],
\end{align}
where
\begin{align}
\widehat{\ell}_{m_t,p}(t)
&=
\left\langle
\widehat{\boldsymbol{\theta}}^f_{m_t}(t),
\mathbf x^f_p
\right\rangle
+
\left\langle
\widehat{\boldsymbol{\theta}}^b_{m_t}(t),
\mathbf x^b_p
\right\rangle,\\
\mathrm{CB}_{m_t,p}(t)
&=
\beta^f_{k_t,T}
\left\|
\mathbf x^f_p
\right\|_{
\left(\boldsymbol{\Sigma}^{f}_{m_t,t}\right)^{-1}}
+
\beta^b_T
\left\|
\mathbf x^b_p
\right\|_{
\left(\boldsymbol{\Sigma}^{b}_{m_t,t}\right)^{-1}}.
\end{align}

On the confidence event,
\begin{align}
\ell_{m_t,p}
\le
\widehat{\ell}_{m_t,p}(t)
+
\mathrm{CB}_{m_t,p}(t),
\qquad
\ell_{m_t,p}
\ge
\widehat{\ell}_{m_t,p}(t)
-
\mathrm{CB}_{m_t,p}(t).
\end{align}

Therefore,
\begin{align}
\ell_{m_t,p_t}
&\le
\widehat{\ell}_{m_t,p_t}(t)
+
\mathrm{CB}_{m_t,p_t}(t)
\\
&=
\widehat{\ell}_{m_t,p_t}(t)
-
\mathrm{CB}_{m_t,p_t}(t)
+
2\mathrm{CB}_{m_t,p_t}(t)
\\
&\le
\widehat{\ell}_{m_t,p_t^\star}(t)
-
\mathrm{CB}_{m_t,p_t^\star}(t)
+
2\mathrm{CB}_{m_t,p_t}(t)
\\
&\le
\ell_{m_t,p_t^\star}
+
2\mathrm{CB}_{m_t,p_t}(t).
\end{align}

Thus the instantaneous regret satisfies
\begin{align}
 r_t
 =&
 \ell_{m_t,p_t}
 -
 \ell_{m_t,p_t^\star}
\le
 2\beta^f_{k_t,T}
 \left\|
 \mathbf x^f_{p_t}
 \right\|_{
 \left(\boldsymbol{\Sigma}^{f}_{m_t,t}\right)^{-1}}
 +
 2\beta^b_T
 \left\|
 \mathbf x^b_{p_t}
 \right\|_{
 \left(\boldsymbol{\Sigma}^{b}_{m_t,t}\right)^{-1}}.
 \label{eq:per-round-regret}
\end{align}
\end{proof}

\begin{theorem}[Total Regret Bound]
With probability at least $1-\delta$, the cumulative regret
\begin{align}
\mathrm{Reg}(T)=\sum_{t=1}^{T}r_t
\end{align}
satisfies
\begin{align}
\mathrm{Reg}(T)
\le
R^f(T)+R^b(T),
\end{align}
where
\begin{align}
R^f(T)
\le
2
\sum_{k=1}^{K}
\beta^f_{k,T}
\left[
C^f_k(T)
+
\sqrt{
2(1+M_k\alpha_f)
N_k(T)
\Gamma^f_k(T)
}
\right],
\label{eq:front-regret-final}
\end{align}
and the back-end term follows the standard asynchronous FedLinUCB bound.
Moreover,
\begin{align}
\Gamma^f_k(T)
\le
 d_f
\log\left(
1+
\frac{N_k(T)L_f^2}{d_f(\lambda_f+\rho_k)}
\right),
\label{eq:gamma-off-gain}
\end{align}
where
\begin{align}
\rho_k
=
\lambda_{\min}
\left(
\boldsymbol{\Sigma}^{f,\mathrm{off}}_{\mathcal G_k}
\right).
\end{align}
\end{theorem}

\begin{proof}
Summing \eqref{eq:per-round-regret} over $t=1,\ldots,T$ gives
\begin{align}
\mathrm{Reg}(T)
\le&
2
\sum_{t=1}^{T}
\beta^f_{k_t,T}
\left\|
\mathbf x^f_{p_t}
\right\|_{
\left(\boldsymbol{\Sigma}^{f}_{m_t,t}\right)^{-1}}
+
2\beta^b_T
\sum_{t=1}^{T}
\left\|
\mathbf x^b_{p_t}
\right\|_{
\left(\boldsymbol{\Sigma}^{b}_{m_t,t}\right)^{-1}}.
\end{align}
Define the first term as $R^f(T)$ and the second as $R^b(T)$.
For the front-end term, group the rounds by $\mathcal G_k$:
\begin{align}
R^f(T)
=
2
\sum_{k=1}^{K}
\beta^f_{k,T}
\sum_{\substack{t\le T\\m_t\in\mathcal G_k}}
\left\|
\mathbf x^f_{p_t}
\right\|_{
\left(\boldsymbol{\Sigma}^{f}_{m_t,t}\right)^{-1}}.
\end{align}
By the covariance comparison lemma and the same asynchronous communication-epoch reordering argument used in FedLinUCB~\cite{he2022simple},
\begin{align}
\sum_{\substack{t\le T\\m_t\in\mathcal G_k}}
\left\|
\mathbf x^f_{p_t}
\right\|_{
\left(\boldsymbol{\Sigma}^{f}_{m_t,t}\right)^{-1}}
\le
C^f_k(T)
+
\sqrt{
2(1+M_k\alpha_f)
N_k(T)
\Gamma^f_k(T)
}.
\end{align}
The boundary term satisfies
\begin{align}
C^f_k(T)
=
O\left(M_kd_f\Gamma^f_k(T)\right).
\end{align}
It remains to bound $\Gamma^f_k(T)$. By definition,
\begin{align}
\Gamma^f_k(T)
=
\log
\frac{
\det\left(
\boldsymbol{\Sigma}^{f,\mathrm{all}}_{\mathcal G_k,T}
\right)
}{
\det\left(
\lambda_f\mathbf I_{d_f}
+
\boldsymbol{\Sigma}^{f,\mathrm{off}}_{\mathcal G_k}
\right)
}.
\end{align}
Since
\begin{align}
\boldsymbol{\Sigma}^{f,\mathrm{all}}_{\mathcal G_k,T}
=
\lambda_f\mathbf I_{d_f}
+
\boldsymbol{\Sigma}^{f,\mathrm{off}}_{\mathcal G_k}
+
\sum_{\substack{t\le T\\m_t\in\mathcal G_k}}
\mathbf x^f_{p_t}
\left(\mathbf x^f_{p_t}\right)^\top,
\end{align}
and
\begin{align}
\lambda_{\min}
\left(
\lambda_f\mathbf I_{d_f}
+
\boldsymbol{\Sigma}^{f,\mathrm{off}}_{\mathcal G_k}
\right)
\ge
\lambda_f+
\rho_k,
\end{align}
we get
\begin{align}
\Gamma^f_k(T)
\le
 d_f
\log\left(
1+
\frac{
\sum_{\substack{t\le T\\m_t\in\mathcal G_k}}
\left\|\mathbf x^f_{p_t}\right\|_2^2
}{
d_f(\lambda_f+
\rho_k)
}
\right).
\end{align}
Using $\left\|\mathbf x^f_{p_t}\right\|_2\le L_f$ and the definition of $N_k(T)$ gives \eqref{eq:gamma-off-gain}.

The back-end parameter $\boldsymbol{\theta}^b$ is shared by all devices, and the back-end statistics are aggregated globally using the standard asynchronous determinant-triggered FedLinUCB protocol. Therefore,
\begin{align}
R^b(T)
\le
2\beta^b_T
\left[
C^b(T)
+
\sqrt{
2(1+M\alpha_b)T\Gamma^b(T)
}
\right],
\end{align}
where
\begin{align}
\Gamma^b(T)
\le
 d_b
\log\left(
1+
\frac{TL_b^2}{\lambda_bd_b}
\right),
\qquad
C^b(T)
=
O\left(Md_b\Gamma^b(T)\right).
\end{align}
Combining the front-end and back-end bounds completes the proof.
\end{proof}

\section*{Proofs of the Communication Complexity}

Let $\mathcal C^f_k(T)$ denote the number of front-end communications triggered by devices in group $\mathcal G_k$ up to horizon $T$. Each time a device triggers a front-end communication, its local covariance increases the determinant of the covariance held by the device by a factor larger than $1+\alpha_f$. Therefore, by the same determinant-growth argument as in FedLinUCB~\cite{he2022simple}, we obtain
\begin{align}
\mathcal C^f_k(T)
\le
2\log2 \cdot d_f(M_k+1/\alpha_f)
\log\left(
1+\frac{N_k(T)L_f^2}{d_f(\lambda_f+\rho_k)}
\right).
\end{align}
Thus the total number of front-end communications is
\begin{align}
\mathcal C^f(T)
=
\sum_{k=1}^{K}\mathcal C^f_k(T)
\le
\sum_{k=1}^{K}
\left[
2\log2 \cdot d_f(M_k+1/\alpha_f)
\log\left(
1+\frac{N_k(T)L_f^2}{d_f(\lambda_f+\rho_k)}
\right)
\right].
\end{align}

For the back-end component, the parameter is shared by all devices and the server aggregation is global. Therefore, the standard FedLinUCB communication bound gives
\begin{align}
\mathcal C^b(T)
\le
2\log 2\cdot d_b(M+1/\alpha_b)
\log\left(
1+\frac{TL_b^2}{\lambda_bd_b}
\right).
\end{align}
Consequently, the total number of communication events satisfies
\begin{align}
\mathcal C(T)
=
\Tilde{\mathcal{O}}\left(
\sum_{k=1}^{K}
\frac{M_kd_f}{\alpha_f}
\log\left(
1+\frac{N_k(T)L_f^2}{d_f(\lambda_f+\rho_k)}
\right)
+
\frac{Md_b}{\alpha_b}
\log\left(
1+\frac{TL_b^2}{\lambda_bd_b}
\right)
\right).
\end{align}